\documentclass{article}
\usepackage{arxiv}

\usepackage[T1]{fontenc}
\usepackage[utf8]{inputenc}
\usepackage{fancyhdr}

\usepackage[square,comma,numbers,compress]{natbib}
\setcitestyle{nosort}

\usepackage{tikz}
\usetikzlibrary{arrows.meta,positioning,calc,fit,backgrounds,shadows.blur,shapes.geometric}
\usetikzlibrary{arrows.meta,positioning,calc,matrix,decorations.pathreplacing}

\usepackage{microtype}
\usepackage{lmodern}
\usepackage{amsmath,amssymb,mathtools}
\usepackage{amsthm}
\usepackage{makecell}
\usepackage{thmtools}
\usepackage{thm-restate}
\usepackage{bm}
\usepackage{booktabs}
\usepackage{multirow}
\usepackage{siunitx,etoolbox}

\usepackage{booktabs}
\usepackage{graphicx}
\usepackage{float}
\usepackage{xcolor}
\usepackage[normalem]{ulem}
\usepackage{hyperref}
\hypersetup{colorlinks=true,linkcolor=black,citecolor=black,urlcolor=blue}

\newtheorem{proposition}{Proposition}

\newtheorem{lemma}{Lemma}
\newtheorem{corollary}{Corollary}

\newtheorem{definition}{Definition}

\newtheorem{notation}{Notation}

\newtheorem{remark}{Remark}
\newtheorem{example}{Example}
\newtheorem{algo}{Algorithm}

\DeclareMathOperator{\Proj}{Proj}

\title{Probabilistic classification from possibilistic data: computing Kullback-Leibler  projection with a possibility distribution}
\date{}

\author{
  Ismaïl Baaj \\
  LEMMA, Paris-Panthéon-Assas University \\
  Paris, France \\
  \texttt{ismail.baaj@assas-universite.fr}
  \And
  Pierre Marquis \\
  Univ. Artois, CNRS, CRIL \\
  Institut Universitaire de France \\
  Lens, France \\
 \texttt{marquis@cril.fr}
}

\hypersetup{
pdftitle={Probabilistic classification from possibilistic data: computing Kullback-Leibler projection with a possibility distribution},
pdfsubject={cs.AI},
pdfauthor={Ismaïl Baaj, Pierre Marquis},
}

\begin{document}

\maketitle

\pagestyle{fancy}
\fancyhead{}
\fancyhead[RO]{\textbf{Probabilistic classification from possibilistic data}}

\begin{abstract}
We consider learning with possibilistic supervision for multi-class classification. For each training instance, the supervision is a normalized possibility distribution that expresses graded plausibility over the classes. From this possibility distribution, we construct a non-empty closed convex set of admissible probability distributions by combining two requirements: probabilistic compatibility with the possibility and necessity measures induced by the possibility distribution,  and linear shape constraints that must be satisfied to preserve the qualitative structure of the possibility distribution. Thus, classes with the same possibility degree receive equal probabilities, and if a class has a strictly larger possibility degree than another class, then it receives a strictly larger probability.\\ Given a strictly positive probability vector output by a model for an instance, we compute its Kullback–Leibler projection onto the admissible set. This projection yields the closest admissible probability distribution in Kullback–Leibler sense. We can then train the model by minimizing the divergence between the prediction and its projection, which quantifies the smallest adjustment needed to satisfy the induced dominance and shape constraints. The projection is computed with Dykstra’s algorithm using Bregman projections associated with the negative entropy, and we provide explicit formulas for the projections onto each  constraint set. 
Experiments conducted on synthetic data and on a real-world natural language inference task, based on the ChaosNLI dataset, show that the proposed projection algorithm is efficient enough for practical use, and that the resulting projection-based learning objective can improve predictive performance.

\end{abstract}

\section{Introduction}

Possibility Theory is an uncertainty theory, which provides computable methods for the representation of incomplete and/or imprecise information. Initially introduced by Zadeh \cite{zadeh1978fuzzy} and considerably developed by Dubois and Prade \cite{dubois2023reasoning}, Possibility Theory models uncertainty by two dual measures, possibility and necessity, which are useful to distinguish what is possible without being certain at all and what is certain to some extent.

While Probability Theory is the standard uncertainty framework in Machine Learning, considering a single probability distribution can be too restrictive when uncertainty mainly reflects partial ignorance rather than randomness \cite{shafer1976,denoeux2020}. In such cases, Possibility Theory \cite{dubois2023reasoning} which offers a simple representation of epistemic uncertainty (uncertainty due to a lack or limited amount of information) can be a valuable alternative.

In this article, we show how Possibility Theory can be leveraged to define a probabilistic classifier (more precisely, a neural network with a softmax output layer) that is trained from uncertain data, when uncertainty is modeled by a possibilistic distribution.
Formally, given a finite set of classes $\mathcal{Y}$, let $\pi^{\mathrm{full}}:\mathcal{Y}\to[0,1]$
be a normalized possibility distribution, i.e., $\pi^{\mathrm{full}}$ is such that $\exists y \in \mathcal{Y}$, $\pi^{\mathrm{full}}(y)=1$. We restrict attention to its support
$Y:=\{y\in\mathcal{Y}:\ \pi^{\mathrm{full}}(y)>0\}$ (with $|Y|=n$) and write
$\pi:Y\to(0,1]$ for the restriction of $\pi^{\mathrm{full}}$ to $Y$, so that
$\max_{y\in Y}\pi(y)=1$ and $\pi^{\mathrm{full}}(y)=0 \text{ for } y\in\mathcal{Y}\setminus Y.$ By relabelling the elements of $Y$ we may assume $Y=\{1,\dots,n\}$, and we identify $\pi$ with its coordinate vector $(\pi_k)_{k=1}^n$. Likewise, for any probability distribution $p$ on $Y$, we identify it with its coordinate vector $p=(p_1,\dots,p_n)$, where $p_k:=p(k)$.

In order to learn a probabilistic classifier from possibilistic data, 
we first show how to associate with $\pi$ a non-empty closed
convex set $\mathcal F^{\mathrm{box}}$ of admissible probability distributions $p$, see
(\ref{eq:fbox}). This set is defined by two types of constraints. First, we require that, for every event $A\subseteq Y$, the probability measure $P$
associated with $p\in\mathcal F^{\mathrm{box}}$ satisfies $N(A)\le P(A)\le \Pi(A)$, where $(N,\Pi)$ are the necessity and possibility measures induced by $\pi$. These inequalities (for all $A\subseteq Y$) define the set of probability measures $P$ compatible with $\Pi$.
Second, we add linear shape constraints that ensure consistency with the ordering
expressed by $\pi$: whenever $\pi_k\ge \pi_{k'}$ for $k,k'\in\{1,\dots,n\}$, admissible probabilities are
constrained so that $p_k\ge p_{k'}$, i.e.,  $\pi_k\ge\pi_{k'}\Longleftrightarrow p_k\ge p_{k'}$ holds.

The next step is to consider a dataset in which, for each input instance, uncertainty about class membership is represented by a possibility distribution $\pi^{\mathrm{full}}$ on $\mathcal{Y}$, whose restriction to $Y$ is denoted by $\pi$ as above. Such possibilistic annotations arise naturally when supervision is incomplete, imprecise, or heterogeneous (e.g., aggregation of multiple assessments), since $\pi^{\mathrm{full}}$ encodes graded plausibility without requiring precise probabilities. 

Finally, we train a probabilistic classifier (e.g., a neural network with a softmax output layer) on this dataset, so that, for each input instance $x$, the classifier outputs a probability distribution $q_\theta(x)$ on $\mathcal{Y}$. We write $q_{\theta\mid Y}(x)$ for the restriction of $q_\theta(x)$ to $Y$, followed by making it strictly positive (if necessary) and then normalizing, so that $q_{\theta\mid Y}(x)$ is a strictly positive probability vector on $Y$. In general, $q_{\theta\mid Y}(x)$ is not guaranteed to belong to $\mathcal F^{\mathrm{box}}(\pi)$, i.e., to satisfy the constraints induced by $\pi$. We therefore use $\pi$ as a possibilistic soft target and define  $p^\star_\theta(x,\pi)$ as the Kullback--Leibler projection \cite{csiszar1975divergence} of $q_{\theta\mid Y}(x)$ onto $\mathcal{F}^{\mathrm{box}}(\pi)$:
\begin{equation}\label{eq:DKLproblem}
p^\star_\theta(x,\pi)
:=
\arg\min_{p\in\mathcal{F}^{\mathrm{box}}(\pi)}
D_{\mathrm{KL}}\!\bigl(p\|q_{\theta\mid Y}(x)\bigr).
\end{equation}
\noindent where the Kullback-Leibler divergence is
$D_{\mathrm{KL}}(p\|q)=\sum_{k=1}^n p_k\log \frac{p_k}{q_k}$
 for $p,q$ on $Y$ (with the usual convention $0\log(0/t)=0$ for $t>0$). To simplify the notations, whenever a training instance $(x,\pi)$ and parameter vector $\theta$ are fixed, we write
$q:=q_{\theta\mid Y}(x)$, $\mathcal F^{\mathrm{box}}:=\mathcal F^{\mathrm{box}}(\pi)$, and $p^\star:=p^\star_\theta(x,\pi)$.

Kullback-Leibler projections onto convex subsets of probability vectors go back to Csiszár's work~\cite{csiszar1975divergence}. In this article we show that  $p^\star$ of~\eqref{eq:DKLproblem} can be
computed iteratively by Dykstra's algorithm
\cite{dykstra1983algorithm,dykstra1985iterative} with Bregman projections
\cite{bregman1967}, see Algorithm~\ref{ex:algo:dykstra}. The fundamental result
of Bauschke and Lewis \cite[Theorem~3.2]{bauschke2000dykstras} guarantees the
convergence of this algorithm. We also provide explicit formulas for the Bregman projections onto each constraint set.\\ 

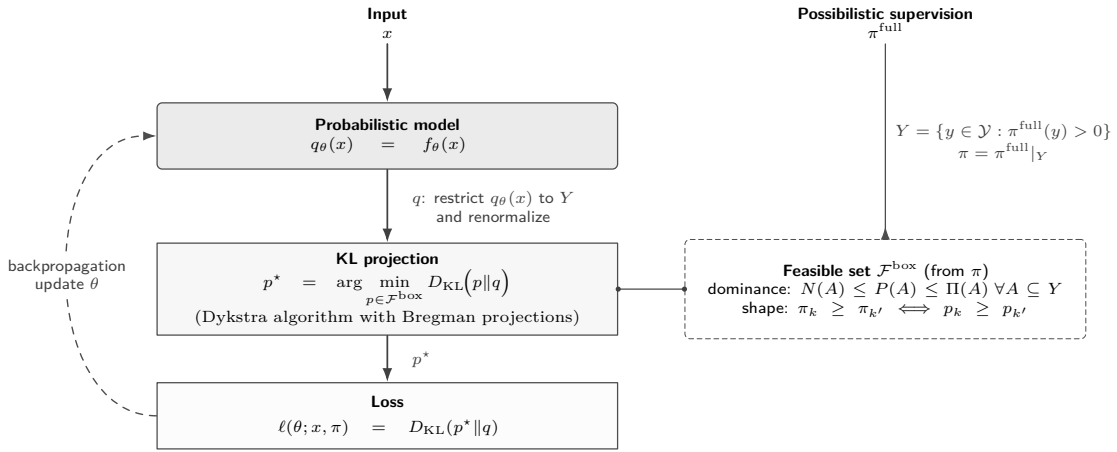
\begin{figure}[H]
\centering
\resizebox{0.9\linewidth}{!}{
\begin{tikzpicture}[
font=\scriptsize\sffamily,
>=Latex,
line cap=round,
line join=round,
node distance=6mm and 12mm,
 /tikz/Wmain/.store in=\Wmain,   /tikz/Wmain=6.6cm,
 /tikz/Wset/.store  in=\Wset,    /tikz/Wset=5.7cm,
 /tikz/Wsmall/.store in=\Wsmall, /tikz/Wsmall=3.1cm,
/tikz/GapX/.store in=\GapX, /tikz/GapX=10mm,
/tikz/GapY/.store in=\GapY, /tikz/GapY=7mm,
 /tikz/GapYbig/.store in=\GapYbig, /tikz/GapYbig=11mm,
 box/.style={draw=black!70,line width=0.55pt,inner xsep=5pt,inner ysep=4pt,minimum height=10mm,align=center},
 box2/.style={draw=black!70,inner xsep=5pt,inner ysep=4pt,minimum height=15mm,align=center},
 datanode/.style={box,rounded corners=2.2pt,fill=black!0},
 wdatanode/.style={box,draw=none,line width=0.pt,fill=black!0},
 modelnode/.style={box,rounded corners=2.2pt,fill=black!8},
 projnode/.style={box,rounded corners=0pt,fill=black!3},
 lossnode/.style={box,rounded corners=0pt,fill=black!1},
 setnode/.style={box2,rounded corners=2.2pt,fill=black!0,dash pattern=on 2.1pt off 1.5pt},
 flow/.style={-{Latex[length=2.1mm,width=1.5mm]},draw=black!75,line width=0.75pt},
 cons/.style={-{Latex[length=2.0mm,width=1.4mm]},draw=black!70,line width=0.55pt},
 loop/.style={-{Latex[length=2.0mm,width=1.4mm]},draw=black!70,line width=0.60pt,dashed},
 tag/.style={font=\scriptsize\sffamily,text=black!75,fill=white,inner sep=1.2pt,rounded corners=1.2pt,align=center}
]
\node[wdatanode,text width=\Wsmall] (x) at (0,0) {\textbf{Input}\\$x$};
\node[modelnode,text width=\Wmain,below=\GapY of x] (model) {\textbf{Probabilistic model}\\$q_\theta(x)=f_\theta(x)$};
\node[projnode,text width=\Wmain,below=\GapYbig of model] (proj) {\textbf{KL projection}\\$p^\star=\arg\min\limits_{p\in\mathcal F^{\mathrm{box}}} D_{\mathrm{KL}}\!\bigl(p\|q\bigr)$\\\textnormal{\scriptsize(Dykstra algorithm with Bregman projections)}};
\node[lossnode,text width=\Wmain,below=\GapY of proj] (loss) {\textbf{Loss}\\[1mm]$\ell(\theta;x,\pi)=D_{\mathrm{KL}}(p^\star\|q)$};
\node[setnode,text width=\Wset,anchor=west] (fbox) at ($(proj.east)+(\GapX,0)$) {\textbf{Feasible set $\mathcal F^{\mathrm{box}}$} (from $\pi$)\\dominance: $N(A)\le P(A)\le \Pi(A)\ \forall A\subseteq Y$\\shape: $\pi_k\ge\pi_{k'}\Longleftrightarrow p_k\ge p_{k'}$};
\node[wdatanode,text width=\Wsmall] (pi) at ($(fbox.center|-x.center)$) {\textbf{Possibilistic supervision}\\$\pi^{\mathrm{full}}$};
\coordinate (elbowPi) at ($(pi.south |- fbox.north)$);
\draw[flow] ($(x.south)-(0,-1.8mm)$) -- (model.north);
\draw[flow] (model.south) -- (proj.north) node[midway,tag,right=3.2mm] {$q$: restrict $q_\theta(x)$ to $Y$\\and renormalize};
\draw[flow] (proj.south) -- (loss.north) node[midway,tag,right=3.2mm] {$p^\star$};
\draw[cons] ($(pi.south)-(0,-1.8mm)$) -- (elbowPi) node[midway,tag,right=1mm] {$Y=\{y\in\mathcal{Y}:\pi^{\mathrm{full}}(y)>0\}$\\$\pi=\pi^{\mathrm{full}}|_Y$} -- ($(fbox.north)+(0,1.9mm)$);
\draw[cons,-] (fbox.west) -- (proj.east);
\fill[black!75] (fbox.west) circle (1.1pt);
\fill[black!75] (proj.east) circle (1.1pt);
\draw[loop] (loss.west) .. controls +(-18mm,0) and +(-18mm,0) .. node[midway,tag] {backpropagation\\update $\theta$} (model.west);
\end{tikzpicture}}
\caption{Training with possibilistic supervision via KL projection onto $\mathcal F^{\mathrm{box}}$.}
\label{fig:learning-pipeline}
\end{figure}

The projection in \eqref{eq:DKLproblem} can then be used for learning purposes, see Figure~\ref{fig:learning-pipeline}. For a given training instance $(x,\pi)$, define the per-instance objective
\begin{equation}\label{eq:DKLpstarq}
\ell(\theta;x,\pi)
=
D_{\mathrm{KL}}(p^\star\|q)
=
\sum_{k=1}^n p^\star_k\log\frac{p^\star_k}{q_k}.
\end{equation}
This divergence quantifies how much the model output $q$ must be corrected, through the projected target $p^\star$, to satisfy the constraints induced by $\pi$. This suggests a learning scheme in which, for each training instance, $p^\star$ is computed from the current prediction $q$ and then used as a soft target to update $\theta$ by minimizing $\ell(\theta;x,\pi)$ (thus, in this learning scheme, we do not backpropagate through the projection step). The loss $\ell(\theta;x,\pi)=D_{\mathrm{KL}}(p^\star\|q)$ is the smallest KL adjustment of $q$ needed to satisfy the dominance and ordering constraints induced by $\pi$. In particular, $\ell(\theta;x,\pi)=0$ whenever $q\in\mathcal F^{\mathrm{box}}$, and for any fixed admissible target $\bar p\in\mathcal F^{\mathrm{box}}$ we have $\ell(\theta;x,\pi)\le D_{\mathrm{KL}}(\bar p\|q)$. Thus, when supervision is specified only through the constraints induced by $\pi$ (rather than by a single probability target satisfying these constraints), the projection-based loss is always no larger than the KL loss to any fixed admissible target.

To evaluate the benefits of the proposed approach, we conduct experiments on synthetic data and on a real natural language inference task \cite{bowman2015large} based on the ChaosNLI dataset \cite{nie2020can}. The obtained results show that the projection algorithm is efficient enough to be used in practice and that the resulting projection-based approach can improve predictive performance.

The rest of the paper is structured as follows. In Section \ref{sec:bg}, we remind the necessary background on Possibility Theory and probability-possibility transformation. In Section \ref{sec:prob-constraints},  given a strictly positive normalized possibility distribution $\pi$, we show how to associate with it a non-empty closed convex set $\mathcal F^{\mathrm{box}}$ of admissible probability distributions. In Section \ref{sec:bregman-dykstra}, we study the Kullback-Leibler projection problem (\ref{eq:DKLproblem}) and show that one can obtain such a projection $p^\star$ using Dykstra algorithm with Bregman projections, see Algorithm \ref{ex:algo:dykstra}. We provide explicit formulas for the projections onto each  constraint set. 
In Section~\ref{sec:numerical}, we present three experiments: first, an empirical evaluation of Algorithm~\ref{ex:algo:dykstra} on synthetic data; second, a synthetic learning task showing that projection-based targets can improve predictive performance over a fixed probability target derived from $\pi$ under the same possibilistic supervision; third, a real natural language inference task based on the ChaosNLI dataset, where the  projection-based approach is evaluated under naturally ambiguous annotations.
Finally, we discuss applications and extensions of the proposed KL-projection framework under possibilistic supervision. Building on Remark~\ref{remark:extension}, we note that the same approach applies to admissible sets other than $\mathcal F^{\mathrm{box}}$: it can be used with any set of probability vectors $F\subseteq\Delta_n$ defined by linear subset inequalities and/or linear shape constraints, provided that $F$ is non-empty, closed, and convex. A key strength of the framework is that different constraint types can be combined by working with their intersection.\\

The proofs of the results of the article are provided in an appendix.

\section{Background}
\label{sec:bg}
In this section, we remind Possibility Theory and present a bijective probability-possibility transformation.
Throughout the article, we use the following notations:
\[
\mathbb{R}^n_{+}
:=
\{x\in\mathbb{R}^n \mid x_i \ge 0,\ i=1,\dots,n\},
\qquad
\mathbb{R}^n_{++}
:=
\{x\in\mathbb{R}^n \mid x_i > 0,\ i=1,\dots,n\}.
\]

\subsection{Possibility Theory}\label{subsection:possibility}

In the following, we give some background on Possibility Theory \cite{dubois2023reasoning,dubois2015possibility}, focusing on concepts needed to define a probability-possibility transformation \cite{dubois1983unfair,dubois1993possibility} that will be presented in the subsequent subsection.

\noindent Let $U$ be a finite set. Any subset  $A \subseteq U$ is called an \textit{event}. In particular, for each $u \in U$, the singleton $\{u\}$ is called an \textit{elementary event}. We denote by $\overline A := U\setminus A$ the complement of $A$.

\begin{definition}
A \emph{possibility measure} on $U$ is a mapping 
$\Pi: 2^U \rightarrow [0,1]$, which assigns a degree $\Pi(A)$ to each event $A \subseteq U$ in order to assess to what extent the event $A$ is possible. It satisfies the following conditions:

\begin{itemize}
    \item $\Pi(\emptyset) = 0$ and $\Pi(U) = 1$,
    \item For any subset $\{ A_1,A_2,\dots,A_m \} \subseteq 2^U$, $\Pi(\bigcup_{i=1}^m A_i) = \max_{i=1,2,\dots,m} \Pi(A_i)$.
\end{itemize}
\end{definition}

\noindent For any event $A$, if $\Pi(A)$ is equal to 1, it means that $A$ is totally possible, while if $\Pi(A)$ is equal to 0, it means that $A$ is impossible. A possibility measure $\Pi$ has the following properties:
\begin{itemize}
    \item $\Pi(A \cup \overline{A}) = \max(\Pi(A),\Pi(\overline{A}))=1$.
    \item 
    For any $A_1,A_2\in 2^U$, if $A_1 \subseteq A_2$, then $\Pi(A_1) \leq \Pi(A_2)$. It follows that for any $A_1,A_2\in 2^U$, we have $\Pi(A_1 \cap A_2) \leq \min(\Pi(A_1),\Pi(A_2))$.
\end{itemize}

\smallskip
\noindent Likewise the notion of possibility measure, a \textit{necessity measure} is defined by:
\begin{definition}
\noindent A \emph{necessity measure} on $U$ is a mapping $N: 2^U \rightarrow [0,1]$, which assigns a degree $N(A)$ to each event $A\subseteq U$ in order to assess to what extent the event $A$ is certain. It satisfies:
\begin{itemize}
    \item $N(\emptyset) = 0$ and $N(U) = 1$,
    \item For any subset $\{ A_1,A_2,\dots,A_m \} \subseteq 2^U$, $N(\bigcap_{i=1}^m A_i) = \min_{i=1,2,\dots,m} N(A_i)$.
\end{itemize}
\end{definition}

\noindent  If $N(A) = 1$, it means that $A$ is certain. If $N(A) = 0$, the event $A$ is not certain at all, but this does not mean that $A$ is impossible. A necessity measure has the following properties:
\begin{itemize}
    \item $N(A \cap \overline{A}) = \min(N(A),N(\overline{A}))=0$.
    \item For any $A_1,A_2\in 2^U$, if $A_1 \subseteq A_2$, then $N(A_1) \leq N(A_2)$. It follows that for any $A_1,A_2\in 2^U$, we have $N(A_1 \cup A_2) \geq \max(N(A_1),N(A_2))$.
\end{itemize}

\smallskip
\noindent The two notions of possibility measure and of necessity measure are dual to each other in the following sense:
\begin{itemize}
    \item If $\Pi$ is a possibility measure, then the corresponding necessity measure $N$ is defined by the following formula:
    \[ N(A):= 1 - \Pi(\overline{A}).\]
    \item Reciprocally, if $N$ is a necessity measure, then the corresponding possibility measure $\Pi$ is defined by the following formula:
    \[ \Pi(A) := 1 - N(\overline A).\] 
\end{itemize}

\noindent A \textit{possibility distribution} on the set $U$ is defined by: 
\begin{definition}
A \emph{possibility distribution} $\pi$ on the set $U$ is a mapping $\pi: U \rightarrow [0,1]$, which assigns to each element $u \in U$ a possibility degree $\pi(u) \in [0,1]$. A possibility distribution is said to be \emph{normalized} if $\exists u \in U$ such that $\pi(u) = 1$.
\end{definition}

\noindent Any possibility measure $\Pi$ gives rise to a normalized possibility distribution $\pi$ defined by the formula:
\[\pi(u) = \Pi(\{u\}),  u\in U.\]

\noindent Therefore, for any subset $A\subseteq U$, we have:
\[ \Pi(A) = \max_{u\in A} \pi(u) \quad \text{ and }  
N(A) =1 - \Pi(\overline A)= \min_{u \notin A} (1 - \pi(u)). \]

\noindent Reciprocally, a normalized possibility distribution $\pi$ gives rise to a possibility measure $\Pi$ 
and a necessity measure $N$ defined by:
\[ \text{for any } A \subseteq U,\quad \Pi(A) = \max_{u\in A} \pi(u) \quad \text{ and } \quad  
N(A) =1 - \Pi(\overline A)= \min_{u \notin A} (1 - \pi(u)).\]

\noindent Possibilistic conditioning is defined in both the qualitative and the quantitative frameworks of Possibility Theory. For a detailed overview, see \cite{dubois2023reasoning}. In the following, the qualitative framework is used.

\subsection{Probability-possibility transformation}\label{sec:ppt}

In the following, we present the bijective probability-possibility transformation  introduced in \cite{dubois1983unfair} and named  ``antipignistic method'' in  \cite{dubois2020possibilistic}.

Transforming a probability distribution $p$ on $Y$ (with its associated probability measure $P$) into a possibility distribution $\pi$ on $Y$  (with its associated possibility measure $\Pi$ and necessity measure $N$) 
consists in \textit{finding a framing interval} $[N(A), \Pi(A)]$ of $P(A)$ for any subset $A \subseteq Y$ \cite{dubois2006possibility,dubois1993possibility}:  the possibility measure $\Pi$ dominates the probability measure $P$. The transformation of the probability distribution $p$ into a possibility distribution $\pi$ should preserve the shape of the distribution: for $y,y' \in Y$, $p(y) \ge p(y') \Longleftrightarrow \pi(y) \ge \pi(y')$.

\subsubsection{Antipignistic method}
\label{subsec:antipignistic}

\noindent If $p$ is a probability distribution on a finite set $Y$, let $P$ denote the probability measure on $Y$ defined by $p$, i.e., $P(A) = \sum_{y\in A} p_y$ where $p_y = P(\{y\})$. 

\noindent The antipignistic method associates  a normalized possibility distribution $\pi$ on $Y$ with $p$, which is such that for all $A \subseteq Y$:
 \[ N(A) \leq P(A) \leq \Pi(A),\]
\noindent where $N(A)$ and $\Pi(A)$ are the necessity and possibility measures defined by $\pi$.

\noindent \textit{Let us suppose that the elements of $Y$ are ordered} so that for $Y=\{y_1, \dots, y_n\}$, we have  $p_1 \geq p_2 \geq \dots \geq p_n\,$ \text{where} $\quad p_i =P(\{y_i\}).$ We call this assumption \textit{the decreasing assumption}.

\noindent The possibility degree $\pi_i = \pi(y_i)$ of $y_i$ where $1 \leq i \leq n$ is defined  by:

\begin{equation}\label{eq:piiform}
\pi_i = i p_i + \sum_{j = i+1}^n p_j =  \sum_{j=1}^n \min(p_j, p_i).    
\end{equation}

\noindent where the equality  $i p_i + \sum_{j = i+1}^n p_j =  \sum_{j=1}^n \min(p_j, p_i)$ holds because of the assumption $p_1 \ge p_2 \ge \dots \ge p_n$.

\noindent For all $A \subseteq Y$, the necessity measure of $A$ can be computed as:
\[
N(A)=\sum_{y\in A}\max\Bigl(p_y-\max_{y'\notin A}p_{y'},\,0\Bigr).
\]
\noindent  Note that $Y$ can be exhausted as follows:
\[ A_0 = \emptyset  \subset A_1 \subset A_2 \subset \dots \subset A_n = Y,\,  \text{with }\,
A_i = \{y_1, y_2, \dots, y_i\}.\]
Then, we have:
\[ N(A) = \max_{0\leq k \leq n, A_k \subseteq A} N(A_k). \]
For $k = 0, 1, 2, \dots, n$, the computation of $N(A_k)$ by the preceding abstract formula (with the convention $p_{n+1} = 0 $) becomes:
\[ N(\emptyset) = 0,\, N(A_k) = \sum_{i=1}^k (p_i - p_{k+1}),\, N(Y) = \sum_{i=1}^n p_i = 1.\]

\noindent We then have for all $A \subseteq Y$:  $N(A) \leq P(A) \leq \Pi(A)$ (see \cite{dubois1983unfair} for the proof and the   underlying semantics of this result). 

\noindent Note that from (\ref{eq:piiform}),  the possibility distribution $\pi$ associated with such a probability distribution $p$ verifies:
\begin{equation}
     \pi_1 = 1, \quad 
\pi_i - \pi_{i+1 } = i(p_i  - p_{i+1}) \geq 0.
\end{equation}
and then we have $\pi_1 = 1 \geq  \pi_2 \geq \dots \geq \pi_n$.

\noindent Reciprocally, starting from a normalized possibility distribution $\pi$  that verifies  $\pi_1 = 1 \geq  \pi_2 \geq \dots \geq \pi_n$, the following formula generates from $\pi$ a probability distribution $p$ which verifies $p_1 \geq p_2 \geq \dots \geq p_n$:
\begin{equation}\label{eq:p_i}
 p_i = \sum_{j=i}^n \frac{1}{j}(\pi_j - \pi_{j+1}) \quad \text{with the convention } \quad  \pi_{n+1} = 0.   
\end{equation}

\noindent Clearly, we have $p_1 \geq p_2 \geq \dots \geq p_n$ and we easily check that the normalized possibility distribution associated with $p$ via the formula (\ref{eq:piiform}) is equal to $\pi$.

\noindent To sum up, between the set of probability values  on the set $\{1, 2, \dots, n\}$ which verifies $p_1 \geq p_2 \geq \dots \geq p_n$ and 
the set of   normalized possibility values 
on the set $\{1, 2, \dots, n\}$ that verify 
$\pi_1=1 \geq  \pi_2 \geq \dots \geq \pi_n$
we have the following one-to-one correspondence:
\[ p \mapsto \pi: \pi_i =   i p_i + \sum_{j = i+1}^n p_j =  \sum_{j=1}^n \min(p_j, p_i),\] 
\[ \pi \mapsto p: p_i =   \sum_{j=i}^n \frac{1}{j}(\pi_j - \pi_{j+1}), \text{with the convention }  \pi_{n+1} = 0.\]
This one-to-one correspondence can be used on any set $Y = \{y_1, y_2, \dots, y_n\}$ where the  domains of definition of each of the two mappings $p \mapsto \pi$ and $\pi \mapsto p$ satisfy the decreasing assumption. 

\noindent Finally, one can observe 
that the mapping $\pi \mapsto p$ preserves the shape of the distributions, i.e., for all $  i\in\{1, 2, \dots, n-1\}$, we have the equivalence  $\pi_i \ge \pi_{i+1} \Longleftrightarrow p_{i} \ge p_{i+1}$ and also the following useful result:
\begin{lemma}
For any $1 \leq k \leq n$, we have $p_k = 0 \Longleftrightarrow \pi_k = 0$. Thus, we have:
\[ \pi\in\mathbb{R}^n_{++} \Longleftrightarrow  p\in\mathbb{R}^n_{++}.\]
\end{lemma}
\qed

Dubois and Prade state that the antipignistic method provides an intuitive ground to the perception of the idea of certainty \cite{dubois2020possibilistic}.

\begin{example}\label{ex:probposs_antipignistic}
Let us study a classification problem where ten classes  $Y=\{0,1,\dots,9\}$ are considered.

We take the following probability distribution on $\{0,\dots,9\}$:
\begin{align*}
p^{(1)} &= \bigl(p^{(1)}(0),p^{(1)}(1),\dots,p^{(1)}(9)\bigr)\\
&= \bigl[0.91, 0.01, 0.01, 0.01, 0.01, 0.01, 0.01, 0.01, 0.01, 0.01\bigr].
\end{align*}
Applying the antipignistic probability–possibility transformation yields the normalized possibility distribution
\[
\pi^{(1)} = \bigl(\pi^{(1)}(0),\dots,\pi^{(1)}(9)\bigr)
          = \bigl[1.00, 0.10, 0.10, \dots, 0.10\bigr].
\]

Now consider a more ambiguous probability distribution on $Y$:
\begin{align*}
p^{(2)} &= \bigl(p^{(2)}(0),p^{(2)}(1),\dots,p^{(2)}(9)\bigr)\\
&= \bigl[0.15, 0.14, 0.13, 0.12, 0.11, 0.09, 0.08, 0.07, 0.06, 0.05\bigr],
\end{align*}
for which the antipignistic transformation gives
\[
\pi^{(2)} = \bigl(\pi^{(2)}(0),\dots,\pi^{(2)}(9)\bigr)
          = \bigl[1.00, 0.99, 0.97, 0.94, 0.90, 0.80, 0.74, 0.67, 0.59, 0.50\bigr].
\]
In both cases, the ranking of the classes is preserved.
\end{example}

\section{Constraints induced by a possibility distribution}
\label{sec:prob-constraints}

In this section, we show how to characterize using linear constraints the set $\mathcal F^{\mathrm{box}}$ 
of probability distributions that are compatible with a given
normalized possibility distribution $\pi^{\mathrm{full}}$ on a finite set of
classes $\mathcal{Y}$. 
The set $\mathcal F^{\mathrm{box}}$ is
obtained by combining two types of constraints. First, we impose the dominance
requirements induced by $\pi^{\mathrm{full}}$: letting $\Pi^{\mathrm{full}}$ and
$N^{\mathrm{full}}$ denote the possibility and necessity measures induced by
$\pi^{\mathrm{full}}$, an admissible probability measure $P$ must satisfy
$N^{\mathrm{full}}(A)\le P(A)\le \Pi^{\mathrm{full}}(A)$ for every event
$A\subseteq\mathcal{Y}$. In particular, if $\pi^{\mathrm{full}}(y)=0$ then
$\Pi^{\mathrm{full}}(\{y\})=0$, so the upper bound $P(\{y\})\le
\Pi^{\mathrm{full}}(\{y\})$ forces $P(\{y\})=0$.
Second, we add linear shape constraints that
preserve the ordering carried by $\pi$ on $Y$: for any $k,k'\in\{1,\dots,n\}$, if $\pi_k\ge \pi_{k'}$, then admissible probabilities $p$ must satisfy $p_k\ge p_{k'}$.

The characterization of $\mathcal F^{\mathrm{box}}$ proceeds as follows. We first show how the dominance constraints
$N^{\mathrm{full}}(A)\le P(A)\le \Pi^{\mathrm{full}}(A)$ can be modeled on $Y$ by  a finite
family of linear inequalities on the probability vector $p$
(Proposition~\ref{prop:dominance}). We then introduce the shape constraints and
define $\mathcal F^{\mathrm{box}}$ using~\eqref{eq:fbox}. We show that
$\mathcal F^{\mathrm{box}}$ is non-empty and forms a closed convex subset of the
probability simplex (Propositions~\ref{prop:f-box} and~\ref{prop:Fbox-convex-closed}).
Finally, we express $\mathcal F^{\mathrm{box}}$ as an intersection of simple
constraint sets, see~\eqref{eq:fboxintersection}; this last representation will be used
in the next section.

\subsection{Constraints as linear inequalities}

\subsubsection{Notations and preliminaries}

\noindent For any $y\in\mathcal{Y}$ with $\pi^{\mathrm{full}}(y)=0$ we have
\[
\Pi^{\mathrm{full}}(\{y\}) = \pi^{\mathrm{full}}(y) = 0,
\qquad
\Pi^{\mathrm{full}}(\{y\}^c) = 1,
\qquad
N^{\mathrm{full}}(\{y\}) = 1-\Pi^{\mathrm{full}}(\{y\}^c) = 0,
\]
where $(\cdot)^c$ denotes complement in $\mathcal{Y}$.
Hence the compatibility requirement
\[
N^{\mathrm{full}}(A) \le P(A) \le \Pi^{\mathrm{full}}(A),
\qquad A\subseteq\mathcal{Y},
\]
implies $P(\{y\})=0$ whenever $\pi^{\mathrm{full}}(y)=0$: any probability measure compatible with $\pi^{\mathrm{full}}$ is supported on $Y$.

It is therefore enough to consider
$
Y=\{1,\dots,n\}
$ only.

We write $\pi:Y\to(0,1]$ for the restriction of $\pi^{\mathrm{full}}$ to $Y$ and, with a slight abuse of notation, identify it with the vector
\[
\pi=(\pi_1,\dots,\pi_n)\in (0,1]^n,\qquad \pi_k:=\pi(k),\quad \max_k \pi_k = 1.
\]

Let  $\sigma$ be a permutation of $\{1,\dots,n\}$ that sorts $\pi$ in
nonincreasing order:
\[
\pi_{\sigma(1)} \ \ge\ \pi_{\sigma(2)} \ \ge\ \cdots \ \ge\ \pi_{\sigma(n)} > 0.
\]
We set $\tilde\pi \in (0,1]^n$ to be the possibility distribution defined by
$\tilde\pi_r := \pi_{\sigma(r)}$ for $r=1,2,\dots,n$, and we define $\tilde\pi_{n+1} := 0$ (introduced for notational convenience only).

For each $r=1,\dots,n$, define the ``top–$r$'' index set
\[
A_r:=\{\sigma(1),\dots,\sigma(r)\},
\]
so that $A_1=\{\sigma(1)\}$ and $A_n=Y$. We also set $A_0:=\emptyset$ (so that $A_0^c=Y$). We denote by
\[
A_r^{c}:=Y\setminus A_r=\{\sigma(r+1),\dots,\sigma(n)\}
\]
the complement of $A_r$ in $Y$. By construction,
\[
\tilde\pi_{r+1} = \max_{j\in A_r^{c}} \pi_j,
\qquad r=1,\dots,n \quad \text{with the convention} \quad \tilde\pi_{n+1} = \max_{\emptyset} \pi_j = 0.
\]

We represent probability distributions on $Y$ by vectors $p=(p_1,\dots,p_n)$ in the probability simplex
\begin{equation}\label{eq:deltansimplex}
  \Delta_n
:= \Bigl\{\,p\in\mathbb{R}^n \,\Bigm|\, p_k\ge 0,\ \sum_{k=1}^n p_k = 1\Bigr\}.
\end{equation}
For $p\in\Delta_n$, we denote by $P$ the associated probability measure on $Y$, i.e.,
$P(A)=\sum_{k\in A}p_k$ for $A\subseteq Y$.

\subsubsection{Dominance constraints}
\begin{restatable}[Equivalent reformulation of Proposition 2 in \cite{delgado1987concept}]{proposition}{PropDominance}
\label{prop:dominance}

Let $p$ denote a probability distribution on $Y=\{1,\dots,n\}$. The 
$n{-}1$ nested subset constraints induced by the given normalized possibility
distribution $\pi$,
\begin{equation}\label{eq:nestedsubsetscons}
    \sum_{k\in A_r} p_k \ \ge\ 1-\tilde\pi_{r+1},
\qquad r=1,\dots,n-1,
\end{equation}
are necessary and sufficient to enforce $N(A) \le P(A)\le \Pi(A)$ for all
$A\subseteq Y$, where $N$ and $\Pi$ are the necessity and possibility measures
associated with $\pi$.
\end{restatable}

See proof in Subsection \ref{subsec:proof:propdominance}.

\begin{corollary}\label{cor:dominance-delgado-form}
For each $r=1,\dots,n-1$, the dominance constraint
\[
\sum_{k\in A_r} p_k \ \ge\ 1-\tilde\pi_{r+1}
\]
is equivalent to
\[
\sum_{k\in A_r^{c}} p_k \ \le\ \tilde\pi_{r+1}.
\]
Define the reversed possibility levels and the corresponding reversed
probabilities by
\[
\check\pi_i := \tilde\pi_{\,n-i+1},
\qquad
\check p_i := p_{\sigma(n-i+1)},
\qquad i=1,\dots,n,
\]
so that $\check\pi_1=\tilde\pi_n$, $\check\pi_n=\tilde\pi_1=1$, and
\[
0 \le \check\pi_1 \le \check\pi_2 \le \cdots \le \check\pi_n = 1.
\]
Then the dominance constraints (\ref{eq:nestedsubsetscons}) can be written in the form introduced by
\cite{delgado1987concept} and used in \cite{lienen2023conformal}:
\begin{equation}\label{eq:delgadoreversedform}
    \sum_{k=1}^{i} \check p_k \ \le\ \check\pi_i,
\qquad i=1,\dots,n.
\end{equation}
\end{corollary}
\begin{proof}
For each $i\in\{1,2, \dots, n\}$, one can easily check:
\[\sum_{k=1}^i \, \check p_k = \sum_{k\in A_{n - i}^c}\, p_k  \le \tilde \pi_{n - i + 1} = \check\pi_i.\]
\end{proof}

The previous proposition shows that the nested subset constraints in (\ref{eq:nestedsubsetscons}) are exactly the
constraints needed for probabilistic compatibility with $\pi$: they ensure that
every event $A$ receives a probability $P(A)$ lying between its necessity
$N(A)$ and its possibility $\Pi(A)$. However, these inequalities only constrain
the cumulative sums on the sets $A_r$ and do not determine the individual
values $p_{\sigma(1)},\dots,p_{\sigma(n)}$. Many different probability vectors
$p$ satisfy all bounds $\sum_{k\in A_r} p_k \ge 1-\tilde\pi_{r+1}$,
$r=1,\dots,n-1$, and thus are compatible with $\pi$, but have different shapes
along the $\pi$-order.

\begin{example}\label{ex:dominance-not-shape}
Consider $n=3$ and the normalized possibility distribution
\[
\pi=(\pi_1,\pi_2,\pi_3)=(1,\,0.51,\,0.50).
\]
Since $\pi_1>\pi_2>\pi_3$, the $\pi$-order is the identity (i.e., $\sigma=\mathrm{id}$), hence $\tilde\pi=\pi$ and the induced nested sets are $A_1=\{1\}$ and $A_2=\{1,2\}$.
Thus the dominance constraints \eqref{eq:nestedsubsetscons} are
\[
p_1 \ \ge\ 1-\tilde\pi_2 = 0.49,
\qquad
p_1+p_2 \ \ge\ 1-\tilde\pi_3 = 0.50,
\]
with $p=(p_1,p_2,p_3)\in\Delta_3$.
We now rewrite the same constraints in the reversed form used in Corollary~\ref{cor:dominance-delgado-form}; see \eqref{eq:delgadoreversedform}.
Define the reversed possibility levels
\[
\check\pi:=(\check\pi_1,\check\pi_2,\check\pi_3):=(\tilde\pi_3,\tilde\pi_2,\tilde\pi_1)=(\pi_3,\pi_2,\pi_1),
\]
so that $0<\check\pi_1\le \check\pi_2\le \check\pi_3=1$.
Given $p\in\Delta_3$, define its reversed coordinate vector $\check p$ by
\[
\check p:=(\check p_1,\check p_2,\check p_3):=(p_3,p_2,p_1).
\]
For $i=1,2,3$, let
$
s_i(\check p):=\sum_{k=1}^i \check p_k
$
denote the partial sums. Then \eqref{eq:delgadoreversedform} is
\[
s_1(\check p)\le \check\pi_1,
\qquad
s_2(\check p)\le \check\pi_2,
\]
i.e.,
\[
p_3\le \pi_3,
\qquad
p_3+p_2\le \pi_2,
\]
which is equivalent to the two dominance constraints above (since $p_1+p_2+p_3=1$).\\

\medskip
\noindent  $\bullet$ Situation 1: \emph{As a solution of the dominance constraints, one can obtain a probability distribution such that the most probable class is not the most possible one.}\\
Indeed, let
\[
p^{(a)}=(0.49,\,0.50,\,0.01)\in\Delta_3.
\]
We check \eqref{eq:nestedsubsetscons}:
\[
p^{(a)}_1 = 0.49 \ge 0.49,
\qquad
p^{(a)}_1+p^{(a)}_2 = 0.49+0.50 = 0.99 \ge 0.50.
\]
Equivalently, in the reversed form \eqref{eq:delgadoreversedform}, we have
\[
\check p^{(a)}=(0.01,\,0.50,\,0.49),
\qquad
s_1(\check p^{(a)})=0.01 \le \check\pi_1=0.50,
\qquad
s_2(\check p^{(a)})=0.01+0.50=0.51 \le \check\pi_2=0.51.
\]
Hence $p^{(a)}$ satisfies the dominance constraints.
However, the unique maximizers are
\[
\arg\max_{k\in\{1,2,3\}} \pi_k = \{1\},
\qquad
\arg\max_{k\in\{1,2,3\}} p^{(a)}_k = \{2\},
\]
so the identity of the top class is not preserved.\\

\medskip
\noindent $\bullet$ Situation 2: \emph{As a solution of the dominance constraints, one can obtain a probability distribution such that the ordering between the second and third classes is not preserved.}\\
Indeed, let
\[
p^{(b)}=(0.49,\,0.01,\,0.50)\in\Delta_3.
\]
We check \eqref{eq:nestedsubsetscons}:
\[
p^{(b)}_1 = 0.49 \ge 0.49,
\qquad
p^{(b)}_1+p^{(b)}_2 = 0.49+0.01 = 0.50 \ge 0.50.
\]
Equivalently, in the reversed form \eqref{eq:delgadoreversedform}, we have
\[
\check p^{(b)}=(0.50,\,0.01,\,0.49),
\qquad
s_1(\check p^{(b)})=0.50 \le \check\pi_1=0.50,
\qquad
s_2(\check p^{(b)})=0.50+0.01=0.51 \le \check\pi_2=0.51.
\]
Hence $p^{(b)}$ satisfies the dominance constraints.
However, the possibility ordering among the last two classes is
\[
\pi_2 = 0.51 > 0.50 = \pi_3,
\]
while the probability ordering under $p^{(b)}$ is reversed:
\[
p^{(b)}_2 = 0.01 < 0.50 = p^{(b)}_3.
\]
\medskip
Both $p^{(a)}$ and $p^{(b)}$ satisfy the same dominance constraints induced by $\pi$ (equivalently, $s_i(\check p)\le \check\pi_i$ for $i=1,2$; see \eqref{eq:delgadoreversedform}), yet one changes the top class (Situation~1) and the other reverses the order of classes $2$ and $3$ (Situation~2). This shows that the dominance constraints enforce compatibility but do not preserve the shape of $\pi$.
\end{example}

Such situations are unexpected. Indeed, in addition to mere compatibility, it is often desirable (for instance, for multi-class classification tasks) that the probability
vector preserves the qualitative structure of $\pi$. We
would like the same pattern of ties and strict inequalities to hold for $\pi$ and $p$, in the
sense that
\[
\text{ from }\tilde\pi_r \ge \tilde\pi_{r+1}
\text{ we want to have }
p_{\sigma(r)} \ge p_{\sigma(r+1)},
\qquad r=1,\dots,n-1,
\]
and, more specifically, strict drops of $\tilde\pi_r$ should correspond to
noticeable drops of $p_{\sigma(r)}$, while equal levels of $\tilde\pi_r$ should
lead to (approximately) tied probabilities.  This motivates the introduction of
explicit constraints on adjacent differences in the $\pi$-order, i.e.,\ shape
constraints of the form
\[
p_{\sigma(r)} - p_{\sigma(r+1)} \ \ge\ \underline{\delta}_r,
\qquad r=1,\dots,n-1,
\]
with gap parameters $(\underline{\delta}_r)$ elicited from the structure of $\pi$.

\subsubsection{Antipignistic reverse mapping}
\noindent To define such constraints, we reuse the bijective probability–possibility transformation called ``Antipignistic'' and reminded in Section~\ref{sec:ppt}. From the normalized possibility distribution $\tilde \pi$ on $Y$ where we have for all $r \in \{1,2,\cdots,n\}$, $\tilde\pi_{r} > 0$ and, by convention, $\tilde\pi_{n+1}=0$, we compute the antipignistic probability distribution $ \dot{p}$ on $Y$ which is defined by (reminded in (\ref{eq:p_i})):
\begin{equation}\label{eq:ppoint}
    \dot{p}_{\sigma(r)} = \sum_{j=r}^n \frac{\tilde\pi_j-\tilde\pi_{j+1}}{j}
\qquad r=1,\dots,n,
\end{equation}

The probability distribution $\dot p\in \Delta_n$ satisfies the following properties:
\begin{restatable}{lemma}{lemmaPropAntipignisticProb}
\label{lemma:lemmaPropAntipignisticProb}\mbox{}
\begin{enumerate}
    \item For all  $i\in \{1, 2, \dots, n\}$, we have $\dot p_{i} > 0$, i.e., $\dot p\in \Delta_n\cap\mathbb{R}^n_{++}$.
    \item For all  $r\in \{1, 2, \dots, n - 1\}$, we have $\dot p_{\sigma(r)} - \dot p_{\sigma(r + 1)}  = \dfrac{1}{r}(\tilde\pi_r - \tilde\pi_{r + 1})$. Therefore the following equivalence holds: $\tilde\pi_r \ge \tilde\pi_{r+1}
\Longleftrightarrow
\dot p_{\sigma(r)} \ge \dot p_{\sigma(r+1)}$ for $r=1,\dots,n-1$. 
    \item For all  $r\in \{1, 2, \dots, n - 1\}$, we have $\sum_{i\in A_r}\, \dot p_i \ge 1 - \tilde\pi_{r + 1}$.
\end{enumerate}
\end{restatable}
See proof in Subsection \ref{subsec:proof:lemmaPropAntipignisticProb}.

The adjacent differences of $\dot{p}$ (which depend only on $\tilde \pi$) are
\begin{equation}\label{eq:adjdiff}
    \dot{g}_r := 
 \frac{\tilde\pi_r-\tilde\pi_{r+1}}{r}=\dot p_{\sigma(r)} - \dot p_{\sigma(r + 1)},
\qquad r=1,2,\dots,n-1.
\end{equation}

Hence $\dot{g}_r=0$ when $\tilde\pi_r=\tilde\pi_{r+1}$ and $\dot{g}_r>0$ when $\tilde\pi_r>\tilde\pi_{r+1}$. Thus, we have $\dot p_{\sigma(1)} \geq \dot  p_{\sigma(2)} \geq \dots \geq \dot  p_{\sigma(n)} > 0$, consistently with
$1 = \tilde\pi_1 \ge \tilde\pi_2 \ge \cdots \ge \tilde\pi_n > 0$.
Moreover,  
we have:
\[
0 \le \tilde\pi_r - \tilde\pi_{r+1} < 1
\qquad\text{for all } r=1,2,\dots,n-1.
\]
It follows that
\[
0 \le \dot{g}_r = \frac{\tilde\pi_r-\tilde\pi_{r+1}}{r} < 1
\qquad\text{for all } r=1,2,\dots,n-1.
\]
Thus every gap  $\dot{g}_r$ lies in the interval $[0,1)$.

\subsubsection{Shape gaps associated with the possibility distribution \texorpdfstring{$\pi$}{pi}}

From the nonincreasing sequence $(\tilde\pi_r)_{r=1}^n$ we distinguish indices where two consecutive values coincide and indices where a strict decrease occurs. This leads to the two index sets:
\begin{equation}\label{eq:eqstrict}
\mathcal R_{\mathrm{equal}}:=\{\,r\in\{1,\dots,n-1\}:\ \tilde\pi_r=\tilde\pi_{r+1}\,\},\qquad
\mathcal R_{\mathrm{strict}}:=\{\,r\in\{1,\dots,n-1\}:\ \tilde\pi_r>\tilde\pi_{r+1}\,\},    
\end{equation}

so that $\mathcal R_{\mathrm{equal}}\cup\mathcal R_{\mathrm{strict}}=\{1,\dots,n-1\}$ and $\mathcal R_{\mathrm{equal}}\cap\mathcal R_{\mathrm{strict}}=\varnothing$.

Define the enforced lower gaps for $r = 1,2,\dots,n-1$:
\begin{equation}\label{eq:underlinedeltar}
    \underline{\delta}_r :=
\begin{cases}
0, & r\in\mathcal R_{\mathrm{equal}},\\
\text{any value in }(0,\,\dot{g}_r] & r\in\mathcal R_{\mathrm{strict}}
\end{cases},
\end{equation}
\noindent where $\dot{g}_r$ is defined in (\ref{eq:adjdiff}).
\subsubsection{Set  of admissible probability vectors}

We now collect in a single set the probability vectors that are
(i) compatible with the possibility distribution $\pi$ in the sense of
the dominance constraints, and (ii) respect the qualitative shape of $\pi$
through the enforced gaps $(\underline{\delta}_r)_{r=1}^{n-1}$ in the
$\pi$-order.

Recall $A_r:=\{\sigma(1),\dots,\sigma(r)\}$ for $r=1,\dots,n-1$. Define
\[
\mathcal F
:= \Bigl\{\, p\in\Delta_n \ \Big|\ 
\sum_{k\in A_r} p_k \ \ge\ 1-\tilde\pi_{r+1}\ \ (r=1,\dots,n-1),\quad
p_{\sigma(r)}-p_{\sigma(r+1)} \ \ge\ \underline{\delta}_r\ \ (r=1,\dots,n-1)
\Bigr\}.
\]

As a direct consequence of Lemma \ref{lemma:lemmaPropAntipignisticProb} we have:
\begin{proposition}\label{prop:f-lower}
For any $(\underline{\delta}_r)_{r=1}^{n-1}$ chosen as in (\ref{eq:underlinedeltar}), the antipignistic probability distribution $\dot{p}$ satisfies $\dot{p}\in\mathcal F \cap \mathbb{R}^n_{++}$, thus $\mathcal F\neq\emptyset$.
\end{proposition}
\qed 

We may want to prevent the highest probability from becoming too large.
The set $\mathcal F$ imposes no upper constraint on $p_{\sigma(1)}$; it only enforces lower subset bounds and lower gap constraints. We
add explicit upper bounds on consecutive differences.

To impose  upper-gap constraints $p_{\sigma(r)}-p_{\sigma(r+1)} \le \overline{\delta}_r$, we choose
upper gaps $\overline{\delta}_r$ such that
\begin{equation}\label{eq:overlinedeltar}
    \dot g_r \le \overline{\delta}_r < 1
\text{ for all } r\in \mathcal R_{\mathrm{strict}} \quad \text{ and } \quad \overline{\delta}_r  = 0 \text{ for all } r\in \mathcal R_{\mathrm{equal}}.
\end{equation}

The restriction $\overline{\delta}_r<1$ is natural, since for any $p\in\Delta_n$ we always have
$p_{\sigma(r)}-p_{\sigma(r+1)}\le 1$, so taking $\overline{\delta}_r\ge 1$ would add no constraint. Since $\underline{\delta}_r \le \dot{g}_r$ by construction, this also implies
$\overline{\delta}_r \ge \underline{\delta}_r$ for all $r$.
Thus, we define
\begin{equation}\label{eq:fbox}
    \mathcal F^{\mathrm{box}}
:= \Bigl\{\, p\in\Delta_n \ \Big|\ 
\sum_{k\in A_r} p_k \ \ge\ 1-\tilde\pi_{r+1},\quad
\underline{\delta}_r \ \le\ p_{\sigma(r)}-p_{\sigma(r+1)}\ \le\ \overline{\delta}_r,\ \ r=1,\dots,n-1
\Bigr\}.
\end{equation}
\begin{example}\label{ex:fbox-eps-and-dotg}
(Example \ref{ex:dominance-not-shape}, cont'ed)\\
We reuse the normalized possibility distribution
$
\pi=(\pi_1,\pi_2,\pi_3)=(1,\,0.51,\,0.50)
$ of Example \ref{ex:dominance-not-shape}, so that $\tilde\pi=\pi$, and the  dominance constraints \eqref{eq:nestedsubsetscons} are
$
p_1 \ \ge\ 1-\tilde\pi_2 = 0.49,$ and $
p_1+p_2 \ \ge\ 1-\tilde\pi_3 = 0.50
$.\\
\noindent We have $\mathcal R_{\mathrm{strict}}=\{1,2\}$ and
$\mathcal R_{\mathrm{equal}}=\varnothing$. The antipignistic reference gaps $ \dot g_r:=\frac{\tilde\pi_r-\tilde\pi_{r+1}}{r},$ for $r=1,2,$
are:
\[
\dot g_1=\frac{1-0.51}{1}=0.49,
\qquad
\dot g_2=\frac{0.51-0.50}{2}=0.005.
\]
\medskip
\noindent
Fix $\varepsilon>0$ such that $0<\varepsilon\le 0.005$ and set
\[
\underline\delta_1=\underline\delta_2=\varepsilon,
\qquad
\overline\delta_1=\dot g_1=0.49,
\qquad
\overline\delta_2=\dot g_2=0.005.
\]
Then the admissible set \eqref{eq:fbox} is
\begin{equation}\label{eq:exfbox}
    \mathcal F^{\mathrm{box}}
=
\Bigl\{p\in\Delta_3\ \Bigm| 
p_1 \ge 0.49, \ \
p_1+p_2 \ge 0.50,\ \
\varepsilon \le p_1-p_2 \le 0.49,\ \
\varepsilon \le p_2-p_3 \le 0.005\}.
\end{equation}
\end{example}

Clearly, we have:
\begin{proposition}\label{prop:f-box}
With the above choices (\ref{eq:underlinedeltar}) and (\ref{eq:overlinedeltar}) of lower and upper gaps $(\underline{\delta}_r,\overline{\delta}_r)_{r=1}^{n-1}$, the antipignistic probability distribution $\dot{p}$ belongs to $\mathcal F^{\mathrm{box}} \cap \mathbb{R}^n_{++}$. Therefore, $\mathcal F^{\mathrm{box}}\neq\emptyset$.
\end{proposition}
\qed 

Thus, to sum up, if the gap parameters satisfy
\[
\text{for all } r\in\mathcal R_{\mathrm{strict}},\quad
0 < \underline{\delta}_r \le \dot g_r \le \overline{\delta}_r < 1,
\qquad\text{and for all } r\in\mathcal R_{\mathrm{equal}},\quad
\underline{\delta}_r=\overline{\delta}_r=0,
\]
then $\dot p\in \mathcal F^{\mathrm{box}}$.
In general, the family of gaps $(\underline{\delta}_r,\overline{\delta}_r)_{r=1}^{n-1}$
could be chosen independently of $\pi$, and one may then check a posteriori whether the resulting
set $\mathcal F^{\mathrm{box}}$ is non-empty. From a practical point of view, however, the antipignistic
reverse mapping \eqref{eq:p_i} provides a natural reference scale for the gaps: the values $\dot g_r$
are determined by $\pi$ and yield a probability distribution $\dot p$ that is compatible with $\pi$
(in the sense of dominance) and satisfies the shape constraints induced by $\pi$. Furthermore,  the probability distribution $\dot p$ is the center of gravity of the set
$\mathcal{K}(\Pi) = \{p \in \Delta_n \mid \forall A \subseteq Y,\ P(A) \le \Pi(A) \}$
of probability measures dominated by $\Pi$, see \cite[Theorem 1]{dubois1993possibility} and \cite{dubois2006possibility}.
In the framework of imprecise probability, $\mathcal{K}(\Pi)$ is called the credal set induced by the
possibility measure $\Pi$ (see, e.g., \cite{destercke2008unifying} and references therein).\\
Other possibility-to-probability transformations have been studied in the literature and can be used as references for the construction of $\mathcal F^{\mathrm{box}}$ in a similar way (see, e.g., \cite{dubois1993possibility}).

However, note that applying a softmax-type normalization to transform a possibility distribution into a probability distribution does not, in general, preserve the dominance constraints $N(A)\le P(A)\le \Pi(A)$. As a matter of illustration, consider $Y=\{1,2\}$ and the normalized possibility distribution $\pi=(\pi_1,\pi_2)=(1,\,0.2)$. Then $N(\{1\})=1-\Pi(\{2\})=1-\pi_2=0.8$, so the dominance constraint $N(\{1\})\le P(\{1\})$ requires $p_1\ge 0.8$. Now define $p_k=\dfrac{e^{\pi_k}}{e^{\pi_1}+e^{\pi_2}}$ for $k=1,2$. We obtain
$p_1=\dfrac{e}{e+e^{0.2}}\approx 0.69$, which violates $p_1\ge 0.8$.

The next proposition states that the distributions in $\mathcal F^{\mathrm{box}}$ respect the qualitative shape of $\pi$, as expected:

\begin{restatable}{proposition}{PropositionPiP}
\label{proposition:PiP}
Let $p \in \mathcal F^{\mathrm{box}}$ be an admissible probability distribution. Then for any $k,k'\in\{1,\dots,n\}$, we have:
\begin{equation}\label{eq:shape}
\pi_k \ge \pi_{k'} \Longleftrightarrow   p_k \ge p_{k'}.
\end{equation}
\end{restatable}
See proof in Subsection \ref{subsec:proof:proposition:PiP}.

For the sequel, it is convenient to view $\mathcal{F}^{\mathrm{box}}$ as  the intersection of three families
of closed convex subsets of $\Delta_n$:
\begin{equation}\label{eq:fboxintersection}
    \mathcal{F}^{\mathrm{box}}
=
\Bigl(\bigcap_{s=1}^{n-1} C^{\mathrm{pref}}_s\Bigr)
\cap
\Bigl(\bigcap_{s=1}^{n-1} C^{\mathrm{low}}_s\Bigr)
\cap
\Bigl(\bigcap_{s=1}^{n-1} C^{\mathrm{up}}_s\Bigr),
\end{equation}
where, for $s=1,\dots,n-1$,
\begin{subequations}
    \begin{equation}\label{eq:Cprefsdef}
        C^{\mathrm{pref}}_s
:= \Bigl\{p\in\Delta_n: \sum_{k\in A_s} p_k \ge 1-\tilde\pi_{s+1}\Bigr\},
    \end{equation}
    \begin{equation}\label{eq:Clowdef}
        C^{\mathrm{low}}_s
:= \Bigl\{p\in\Delta_n: p_{\sigma(s)} - p_{\sigma(s+1)} \ge \underline{\delta}_s\Bigr\},
    \end{equation}
    \begin{equation}\label{eq:Cupdef}
        C^{\mathrm{up}}_s
:= \Bigl\{p\in\Delta_n: p_{\sigma(s+1)} - p_{\sigma(s)} \ge -\,\overline{\delta}_s\Bigr\}.
    \end{equation}
\end{subequations}

\begin{proposition}\label{prop:Fbox-convex-closed}
The admissible set $\mathcal F^{\mathrm{box}}$ is a closed convex subset of $\Delta_n$ which contains the probability distribution $\dot{p}$.
\end{proposition}
\begin{proof}
The sets defined in (\ref{eq:Cprefsdef}), (\ref{eq:Clowdef})  and (\ref{eq:Cupdef})  are reciprocal images of closed intervals of $\mathbb{R}$ by linear maps  on  $\mathbb{R}^n$ restricted to the closed set $\Delta_n$.
\end{proof}

\begin{remark}\label{remark:extension}
Although $\mathcal F^{\mathrm{box}}$ is constructed here from a possibility distribution $\pi$,
the proposed approach can be used in the more general case when $\mathcal F^{\mathrm{box}}$ is characterized using a set of linear constraints such that admissible families of probability vectors are defined as an intersection
\[
F=\bigcap_{i=1}^m C_i \subseteq \Delta_n,
\]
where the sets $C_i$ are closed convex sets that  encode nested subset inequalities of the form
$\sum_{k\in A_r} p_k \ge b_r$ (equivalently, $\sum_{k\in A_r^c} p_k \le u_r$),
together with linear shape constraints of the form
$\underline{\delta} \le p_k-p_{k'} \le \overline{\delta}$. Whether or not such constraints are derived from a possibility distribution is irrelevant.
Accordingly, all the projection and optimization results in this article apply to such families $F$ as long as
$F$ is closed, convex, and non-empty.
\end{remark}

\section{Kullback-Leibler projection as a Bregman distance}
\label{sec:bregman-dykstra}

In this section, we study the Kullback-Leibler projection problem (\ref{eq:DKLproblem}). 
We keep the finite set of classes $Y=\{1,\dots,n\}$. We reuse the probability simplex
$
\Delta_n
:= \Bigl\{\,p\in\mathbb{R}^n \,\Bigm|\, p_k \ge 0,\ \sum_{k=1}^n p_k = 1\Bigr\}
$, see (\ref{eq:deltansimplex}).
We consider a normalized possibility distribution
$
\pi=(\pi_1,\dots,\pi_n)\in (0,1]^n
$
on $Y$ such that
$
\pi_k > 0 \quad \text{for all } k=1,\dots,n,
$ and $
\max_{1\le k\le n} \pi_k = 1.
$
In Section~\ref{sec:prob-constraints}, we have shown how this possibility distribution induces a
family of linear constraints: dominance constraints 
and shape constraints.
Collecting all these inequalities defines the admissible set
$\mathcal{F}^{\mathrm{box}}\subseteq\Delta_n$, see (\ref{eq:fbox}), which is non-empty, closed and convex (Proposition \ref{prop:Fbox-convex-closed}).

We now interpret this construction in a multi-class prediction setting where $Y$ is the set of classes.
For a fixed instance, a probabilistic classifier (for example, a neural network
with a softmax output layer) produces a strictly positive probability vector
\[
q=(q_1,\dots,q_n)\in\Delta_n,
\qquad q_k>0 \ \text{for all } k,
\]
where $q_k$ is the predicted probability assigned to class $k$. The
possibilistic information for the same instance is encoded by $\pi$ and $\mathcal{F}^{\mathrm{box}}$ is  the admissible set of probability vectors
that are compatible with $\pi$.

In this section, we perform a correction
step, which consists in replacing $q$ by a distribution $p^\star\in\mathcal{F}^{\mathrm{box}}$
that satisfies all these constraints while remaining as close as possible to $q$
in Kullback-Leibler sense.

The Kullback-Leibler divergence between $p \in \Delta_n$ and $q\in \Delta_n \cap \mathbb{R}^n_{++}$ is
defined by
\begin{equation}\label{eq:KLdiv}
D_{\mathrm{KL}}(p\|q)
= \sum_{k=1}^n p_k \log\frac{p_k}{q_k},    
\end{equation}
with the convention $0\log(0/t)=0$ for $t>0$. As we will see below,  the corrected probability distribution
$p^\star$ is defined as the unique solution of the optimization problem:
\begin{equation}\label{eq:pstar}
p^\star
:= \arg\min_{p\in\mathcal{F}^{\mathrm{box}}} D_{\mathrm{KL}}(p\|q).    
\end{equation}

The aim of this section is to show that $p^\star$ can be computed using  Dykstra's
algorithm \cite{dykstra1983algorithm,dykstra1985iterative} with Bregman projections \cite{bregman1967} associated with the negative entropy function, as in \cite{censor1998dykstra}. We rely on Bauschke et al's works~\cite{bauschke2000dykstras, bauschke1997legendre}, who state Dykstra’s algorithm for Bregman projections and prove its convergence under explicit assumptions; see~\cite[Theorem~3.2]{bauschke2000dykstras}.

This section is structured as follows:
\begin{itemize}
    \item In Subsection \ref{subsec:kldiv}, we begin by relating the Kullback-Leibler divergence to the Bregman distance  associated  with the negative entropy function. Then, we check that the negative entropy function and the closed convex set $\mathcal{F}^{\mathrm{box}}$ satisfy the assumptions of ~\cite[Theorem~3.2]{bauschke2000dykstras}, and thus allows us to apply Dykstra's algorithm with Bregman projections for obtaining $p^\star$.
    \item In Subsection \ref{subsec:bregmanProj}, we show  (Lemma \ref{lemma:Projty}) that, on the set $\Delta_n \cap \mathbb{R}^n_{++}$, the Bregman projection (as defined in \cite{bauschke1997legendre}) on  a closed convex set $C \subseteq \Delta_n$ such that  $C \cap \mathbb{R}^n_{++} \neq \varnothing$ with respect to the negative entropy function  coincides with the Kullback-Leibler projection on such set $C$. In Corollary \ref{cor:projCzzNat}, we show that the Bregman projection of a vector $z \in \mathbb{R}^n_{++}$ on such a convex set $C$ coincides with the Bregman projection of any homothetic vector of the form $t . z$ on the convex set $C$, where $t > 0$.\\
    We provide explicit formulas for the Bregman projections with the negative entropy function on each of the constraints  $C^{\mathrm{pref}}_1,\cdots,C^{\mathrm{pref}}_{n-1}$, $C^{\mathrm{low}}_1,\cdots,C^{\mathrm{low}}_{n-1}$, $C^{\mathrm{up}}_1,\cdots,C^{\mathrm{up}}_{n-1}$, see (\ref{eq:Cprefsdef}-\ref{eq:Cupdef}), involved in the set $\mathcal{F}^{\mathrm{box}}
=
\Bigl(\bigcap_{s=1}^{n-1} C^{\mathrm{pref}}_s\Bigr)
\cap
\Bigl(\bigcap_{s=1}^{n-1} C^{\mathrm{low}}_s\Bigr)
\cap
\Bigl(\bigcap_{s=1}^{n-1} C^{\mathrm{up}}_s\Bigr)$, see (\ref{eq:fboxintersection}). These formulas, established in Proposition \ref{prop:projCbv1} and Proposition \ref{prop:projCbv2}, are based on the Karush-Kuhn-Tucker (KKT) conditions \cite[Chapter 5]{boyd2004convex} and Corollary \ref{cor:projCzzNat}.
\item In Subsection \ref{subsec:dykstraalgo}, we apply Dykstra's algorithm with the negative entropy function and the convex set  $\mathcal{F}^{\mathrm{box}}$, based on its formulation in~\cite[Theorem~3.2]{bauschke2000dykstras}. Thus, the algorithm converges to $p^\star$, see the proof of its convergence in~\cite{bauschke2000dykstras}. In Lemma \ref{lemma:AlgoD1}, for our setting, we reformulate the algorithm of ~\cite[Theorem~3.2]{bauschke2000dykstras} in a simpler form  that makes it easier to implement on a computer, see our numerical study in Section \ref{sec:numerical}.
\end{itemize}

\subsection{Kullback-Leibler divergence   as the Bregman distance  associated  with the negative entropy function}
\label{subsec:kldiv}

In this subsection, we closely follow \cite{bauschke2000dykstras}.\\ 
We briefly recall a well-known result in Lemma \ref{lemma:Df=KL}: the Kullback-Leibler divergence  $D_{\mathrm{KL}}(p\|q)$, where $p\in \Delta_n$,    arises as the Bregman distance  associated  with the negative entropy function \cite{censor1998dykstra,bregman1967,bauschke2000dykstras}.\\
In Proposition \ref{prop:hyp3.2}, we verify that the negative entropy function and the closed convex set $\mathcal{F}^{\mathrm{box}}$ satisfy the assumptions of Theorem 3.2 of \cite{bauschke2000dykstras}.  We   compute   the gradient of the conjugate function of the negative entropy, which is used in the algorithm in Theorem 3.2 of \cite{bauschke2000dykstras}. \\
Finally, we end  this subsection by reminding the Bregman projection associated with the Bregman distance induced by the negative entropy function \cite{bauschke2000dykstras,bauschke1997legendre,censor1998dykstra}.

We use $\operatorname{int}(\Omega)$ to denote the interior of a subset $\Omega$ in a metric space $E$: $\operatorname{int}(\Omega)$ is the largest open subset of $E$ contained in  $\Omega$.
    In particular, for a function $f:\mathbb{R}^n\to(-\infty,+\infty]$ with effective domain
\[
\mathrm{dom} f := \{x\in\mathbb{R}^n \mid f(x)<+\infty\},
\]
we write $\operatorname{int}(\mathrm{dom} f)$ for the interior of $\mathrm{dom} f$.

We denote by $\langle x , y \rangle$ the Euclidean scalar product of the vectors $x , y \in\mathbb{R}^n $.

\subsubsection{Kullback-Leibler divergence as Bregman distance}
\noindent
Let $f:\mathbb{R}^n\to\mathbb{R}\cup\{+\infty\}$ be the  negative
entropy
\begin{equation}\label{eq:f}
f(x) := \sum_{k=1}^n x_k \log x_k,    
\end{equation}
with the convention $0\log 0=0$ and $f(x)=+\infty$ whenever some $x_k<0$.
Then
\[
\mathrm{dom}f = \mathbb{R}^n_+,
\qquad
\mathrm{int}(\mathrm{dom}f) = \mathbb{R}^n_{++}.
\]
It is easy to see that $f$ is continuous and convex  on $\mathbb{R}^n_+$ and thus a closed convex proper function on $\mathbb{R}^n$
  in the sense of \cite{Rockafellar1970ConvexA}.  General theorems of differential calculus imply that $f$ is  differentiable on
$\mathrm{int}(\mathrm{dom}f) = \mathbb{R}^n_{++}$ and that the gradient function $\nabla f$ is given by:
\[\nabla f(y) = (\log y_1 + 1, \log y_2 + 1, \dots, \log y_n + 1)
\quad \text{for all $y = (y_1, y_2, \dots, y_n)\in \mathbb{R}^n_{++}$}. \] 
The Bregman distance  \cite{bauschke1997legendre,bauschke2000dykstras} associated with $f$ is the function
\begin{align*}
D_f:\ \mathbb{R}^n \times \mathbb{R}^n_{++} &\rightarrow \mathbb{R},\\
(x,y) &\mapsto f(x) - f(y) - \langle \nabla f(y), x-y\rangle.
\end{align*}

For all $(x, y)\in \mathbb{R}^n_{+} \times \mathbb{R}^n_{++}$, a direct calculation gives:
\[
D_f(x, y)
= \sum_{k=1}^n x_k \log\frac{x_k}{y_k}
  - \sum_{k=1}^n x_k
  + \sum_{k=1}^n y_k \quad \text{ with the convention } 0\cdot\log 0 =0.
\]
from which we easily deduce:
\begin{lemma}\label{lemma:Dfxty}
For all  $x\in \Delta_n$  and $y\in \Delta_n \cap  \mathbb{R}^n_{++}$, we have:
\begin{enumerate}
    \item $D_f(x, y) = \sum_{k=1}^n x_k \log\frac{x_k}{y_k} = D_{\mathrm{KL}}(x\|y)$,
    \item $D_f(x, t y)
= D_f(x, y)  + t  - \log t - 1$ for all $t > 0$.
    \end{enumerate}     
\end{lemma}
\qed

As our probability distribution $q$ satisfy $q\in \Delta_n \cap\mathbb{R}^n_{++}$ and $\mathcal{F}^{\mathrm{box}} \subseteq \Delta_n$, we deduce (as in \cite{censor1998dykstra,bregman1967}):
\begin{lemma}\label{lemma:Df=KL}
    For all $p\in \Delta_n$, we have:
    \begin{equation}
        D_{\mathrm{KL}}(p\|q) =  D_f(p, q) \quad \text{ and }\quad \min_{p\in \mathcal{F}^{\mathrm{box}}} D_{\mathrm{KL}}(p\|q) = \min_{p\in \mathcal{F}^{\mathrm{box}}} D_f(p, q).
    \end{equation}
\end{lemma}
\qed

\subsubsection{Verifying that the negative entropy function and the closed convex set \texorpdfstring{$\mathcal{F}^{\mathrm{box}}$}{fbox} satisfy the assumptions of Theorem 3.2 of \texorpdfstring{\cite{bauschke2000dykstras}}{bauschke and lewis}}

To apply Theorem 3.2 of \cite{bauschke2000dykstras}, we must check that its assumptions are satisfied by $f$
 and $\mathcal{F}^{\mathrm{box}}$:
 \begin{restatable}{proposition}{propositionOkNegEntropFbox}
     \label{prop:hyp3.2}
 The negative entropy function $f$ is very strictly convex, co-finite and of Legendre type in the sense of  \cite{bauschke2000dykstras}.
 Moreover,  
 $\mathcal{F}^{\mathrm{box}} \cap \operatorname{int}(\mathrm{dom}f)  =    \mathcal{F}^{\mathrm{box}}  \cap \mathbb{R}^n_{++}  \not= \emptyset $.
 \end{restatable}

See proof in Subsection \ref{subsec:proof:hyp32}.

For computing $p^\star$, see (\ref{eq:pstar}),  by the algorithm of \cite[Theorem 3.2]{bauschke2000dykstras}, we also require to compute  
the gradient of the conjugate function $f^\ast$ of the negative entropy $f$,   which  satisfies  $(\nabla f^\ast \circ \nabla f)(x) = x$ for all $x\in \mathbb{R}^n_{++}$, see \cite[Fact 2.7]{bauschke2000dykstras}. Then, from 
$\nabla f(x) = [\log x_k  +  1]$, 
we deduce that  the gradient of the conjugate function $f^\ast$ is given by:
\begin{align}\label{eq:gradfstar}
\nabla f^\ast: \ \mathbb{R}^n  &\rightarrow \mathbb{R}^n,\\
y &\mapsto \nabla f^\ast(y) = (e^{y_1 - 1}, e^{y_2 - 1}, \dots e^{y_n - 1}).\nonumber
\end{align}

\subsubsection{Bregman projections associated with the negative entropy function and \texorpdfstring{$\mathcal{F}^{\mathrm{box}}$}{fbox}}

Let $C\subseteq \mathbb{R}^n$ be a non-empty closed convex set  where $C\cap\mathbb{R}^n_{++}\neq\emptyset$. From the third statement of \cite[Theorem 3.12]{bauschke1997legendre}, we know that for any    $y\in\mathbb{R}^n_{++}$, the set $\text{argmin}_{x\in C \cap \mathbb{R}^n_{+}}\, D_f(x, y)$ is a one-point set which is contained in $\mathbb{R}^n_{++}$.  Following \cite[Theorem 3.12]{bauschke1997legendre}, we set 
\begin{equation}\label{eq:projB}
 \Proj_C^f(y):= \arg\min_{x\in C \cap \mathbb{R}^n_{+}}\, D_f(x, y)   
\end{equation}
and  call the mapping:
\[ \begin{aligned}
\Proj_C^f: \mathbb{R}^n_{++} &\rightarrow C \cap \mathbb{R}^n_{++},\\
 y &\mapsto \Proj_C^f(y)
\end{aligned}\]
the Bregman projection  on $C$ with respect to $f$.

For any   $y\in \mathbb{R}^n_{++}$, \cite[Proposition 3.16]{bauschke1997legendre}    characterizes the vector $\Proj_C^f(y)$:
\begin{equation}\label{eq:projRule}
    \text{if }  z\in C \cap \mathbb{R}^n_{++}\text{ satisfies } \langle \nabla f(y) - \nabla f(z), x - z \rangle \le 0
 \text{ for all } x\in C, \text{ then }  z = \Proj_C^f(y).
\end{equation}
Clearly,  (\ref{eq:projRule}) implies:
\begin{equation}\label{eq:projRule1}
\Proj_C^f(y)  = y \quad \text{for all $y\in C \cap \mathbb{R}^n_{++}$ } \quad\text{ and }\quad(\Proj_C^f\circ \Proj_C^f)(y) =  \Proj_C^f(y) \quad \text{for all $y\in   \mathbb{R}^n_{++}$.}   
\end{equation}

From (\ref{eq:projB}) and Lemma \ref{lemma:Dfxty}, we deduce:
\begin{restatable}{lemma}{lemmaProjty}
 \label{lemma:Projty}
Let $C\subseteq \Delta_n$ be a non-empty closed convex set  such that $C\cap\mathbb{R}^n_{++}\neq\emptyset$.
We have 
\begin{equation}\label{eq:Projty}
\text{For any $y\in \Delta_n \cap \mathbb{R}^n_{++}$}, \quad 
\Proj_C^f(t y) = \Proj_C^f(y) = \arg\min_{x\in C} D_{\mathrm{KL}}(x\|y) \quad \text{for all $t >0$}.
\end{equation}
Thus, on the set $\Delta_n \cap \mathbb{R}^n_{++}$, the Bregman projection on such set $C$ with respect to $f$  coincides with the Kullback-Leibler projection on such set $C$. 
\end{restatable}
See proof in Subsection \ref{subsec:proof:lemma:Projty}.

From   Lemma \ref{lemma:Projty}, it follows that:

\begin{corollary}\label{cor:projCzzNat}
Let $C\subseteq \Delta_n$ be a non-empty closed convex set  such that $C\cap\mathbb{R}^n_{++}\neq\emptyset$.
For any $z\in   \mathbb{R}^n_{++}$, set  $z^\sharp:= \dfrac{z}{\sum_{k=1}^n z_k}$. Then we have:
\begin{equation}\label{eq:natural}
\Proj_C^f(z)= \Proj_C^f(z^\sharp) \quad \text{and} \quad z^\sharp \in    \Delta_n \cap \mathbb{R}^n_{++}  
\end{equation}
\end{corollary}
\begin{proof}
To obtain  (\ref{eq:natural}), it suffices to apply    (\ref{eq:Projty}) with $t:= \Vert z \Vert_1$  and $y:= z^\sharp$.
\end{proof}

In our setting we take $C=\mathcal{F}^{\mathrm{box}}$, a closed convex subset of
$\Delta_n$ which, by Proposition~\ref{prop:Fbox-convex-closed}, contains a
strictly positive vector $\dot p\in\Delta_n$. Given any strictly positive
reference probability distribution $q\in\Delta_n \cap \mathbb{R}^n_{++}$, the Kullback-Leibler projection of $q$ onto
$\mathcal{F}^{\mathrm{box}}$ is given by:
\begin{equation}
    p^\star
:=
\arg\min_{p\in\mathcal{F}^{\mathrm{box}}} D_{\mathrm{KL}}(p\|q) =
\Proj_{\mathcal{F}^{\mathrm{box}}}^f(q).
\end{equation}
In the next subsection, since $\mathcal{F}^{\mathrm{box}}
=
\Bigl(\bigcap_{s=1}^{n-1} C^{\mathrm{pref}}_s\Bigr)
\cap
\Bigl(\bigcap_{s=1}^{n-1} C^{\mathrm{low}}_s\Bigr)
\cap
\Bigl(\bigcap_{s=1}^{n-1} C^{\mathrm{up}}_s\Bigr)$, see (\ref{eq:fboxintersection}), we will compute by explicit formulas the Bregman projections on the constraints sets $C^{\mathrm{pref}}_1,\cdots,C^{\mathrm{pref}}_{n-1}$, $C^{\mathrm{low}}_1,\cdots,C^{\mathrm{low}}_{n-1}$, $C^{\mathrm{up}}_1,\cdots,C^{\mathrm{up}}_{n-1}$, see (\ref{eq:Cprefsdef}-\ref{eq:Cupdef}), which we denote  by
\begin{equation}\label{eq:specializationCi}
C_1,\dots,C_{n-1} := C^{\mathrm{pref}}_1,\dots,C^{\mathrm{pref}}_{n-1},\quad
C_n,\dots,C_{2n-2} := C^{\mathrm{low}}_1,\dots,C^{\mathrm{low}}_{n-1},\quad
C_{2n-1},\dots,C_{3n-3} := C^{\mathrm{up}}_1,\dots,C^{\mathrm{up}}_{n-1},
\end{equation}
so we have $\mathcal{F}^{\mathrm{box}} = \bigcap_{i=1}^m C_i$ with    $m=3n-3$ and each convex subset $C_i$ is contained in $\Delta_n$.

For each $i=1,2,\cdots,m$, we denote by 
$\Proj_i : \mathbb{R}^n_{++} \rightarrow C_i \cap  \mathbb{R}^n_{++}$ the Bregman projection on $C_i$ with respect to $f$. For each $i=1,2,\cdots,m$, we deduce from (\ref{eq:projB}):
\begin{equation}\label{eq:projn}
\Proj_i(y) = \arg\min_{x\in C_i} D_f(x,y),
\qquad y\in\mathbb{R}^n_{++}.
\end{equation}

\subsection{Bregman projections onto the convex sets \texorpdfstring{$C_1, C_2, \dots, C_m$ }{C1 C2 ...Cm}}
\label{subsec:bregmanProj}

To compute each projection
\begin{align}\label{eq:projIRn}
\Proj_i:\ \mathbb{R}^n_{++} &\rightarrow C_i \cap \mathbb{R}^n_{++},\\
z &\mapsto \arg\min_{x\in C_i} D_f(x,z), \nonumber
\end{align}
we proceed as follows.

Fix  $b\in \mathbb{R}$ and $v\in \mathbb{R}^n$. Set 
$C = C_{b, v}  := \{x \in \Delta_n\,\mid \, b \le \langle x, v \rangle\}$ and   suppose that $C \,\cap \mathbb{R}^n_{++} \not= \emptyset$. Then   $C$     is a non-empty closed convex subset of $\mathbb{R}^n$ and  
for any   $z\in  \mathbb{R}^n_{++}$, we know from   \cite[Theorem 3.12]{bauschke1997legendre} that the convex optimization problem in the sense of \cite[Chapter 5]{boyd2004convex}:
\begin{equation}\label{eq:optimpb}
\min_{x\in  \mathbb{R}^n}  D_f(x, z) \quad \text{subject to} \quad 
  \langle x, v\rangle \ge b, \, x_k \ge 0 \text{ for all $k\in\{1, 2, \dots n\} $},\, \sum_{k=1}^n x_k = 1.    
\end{equation}
admits a unique solution $\hat z = \Proj_C^f(z) \in C \cap \mathbb{R}^n_{++}$. Using    the Karush-Kuhn-Tucker (KKT) conditions for the above optimization problem, see \cite[Chapter 5, p244]{boyd2004convex},   we will compute $\hat z$ for our convex sets $(C_i)_{1 \le i \le m}$.

We rely on the following lemma:

\begin{restatable}{lemma}{lemmaKKT}
\label{lemma:KKT}\mbox{}
Consider  $b\in \mathbb{R}$ and $v=[v_k]\in \mathbb{R}^n$. Set 
$C = C_{b, v}  := \{x \in \Delta_n\,\mid \, b \le \langle x, v \rangle\}$ and suppose that $C \,\cap \mathbb{R}^n_{++} \not= \emptyset$. For  any   $z\in  \mathbb{R}^n_{++}$, set $\hat z = \Proj_C^f(z)$.
\begin{enumerate}
    \item the convex optimization problem in the sense of \cite[Chapter 5]{boyd2004convex}:
\begin{equation}\label{eq:optimpb2}
\min_{x\in  \mathbb{R}^n}  D_f(x, z) \quad \text{subject to} \quad 
  \langle x, v\rangle \ge b, \, x_k \ge 0 \text{ for all $k\in\{1, 2, \dots n\} $},\, \sum_{k=1}^n x_k = 1.    
\end{equation}
admits  $\hat z$ as a unique solution.
\item There is a pair $(\lambda^\star, \nu^\star)\in \mathbb{R}_+ \times \mathbb{R}$ such that for all $k\in\{1, 2, \dots, n\}$, we have:
\begin{equation}\label{eq:comphaty0}
 \hat z_k = e^{\lambda^\star \, v_k + \nu^\star} z_k.
\end{equation}
If $b < \langle  \hat z, v \rangle$, then $\lambda^\star = 0$.
  
\end{enumerate}
\end{restatable}
See proof in Subsection \ref{subec:proof:lemmaKKT}.

To compute each projection $\Proj_i$, see (\ref{eq:projIRn}), for $i\in\{1,\dots,m\}$, we specify a pair $(b, v) \in \mathbb{R}\times \mathbb{R}^n$
such that $C_i = C_{b, v}$ and apply Lemma \ref{lemma:KKT}.

\subsubsection{\texorpdfstring{Computing $\Proj_r$ for $r=1,2,\cdots,n-1$}{computing projr for r=1,2,...,n-1}}

We remind that   $C_r = \{x \in \Delta_n\,\mid\, \sum_{i\in A_r} x_i \ge 1 -  \tilde\pi_{r + 1}\}$.

For any $z\in\mathbb{R}^n_{++}$, we set $\Vert z \Vert_1 = \sum_{k=1}^n z_k$ and 
$z^\sharp := \dfrac{z}{\Vert z \Vert_1}$.

\begin{restatable}{proposition}{propprojCbvOne}
\label{prop:projCbv1}
Let $z\in \mathbb{R}^n_{++}$ and $\hat z:= \Proj_{C_r}^f(z)$ where $C_r:= \{x\in \Delta_n\,\mid\, \sum_{i\in A_r} x_i \ge 1 -\tilde\pi_{r+1}\}$. 

Set    $s := \dfrac{1}{\Vert z \Vert_1} \sum_{k\in A_r} z_k$,   $b:= 1 -\tilde\pi_{r+1} < 1$ and $z^\sharp:=\dfrac{z}{\Vert z \Vert_1}$. 
\begin{enumerate}
    \item If  $s \ge b$ then  $z^\sharp\in C_r$ and then   $\hat z:= \Proj_{C_r}^f(z) = \Proj_{C_r}^f(z^\sharp)= z^\sharp$.
    \item If  $s < b$, the components of the vector $\hat z  = [\hat z_k]$ are given by
     \begin{equation}\label{eq:comp}
      \hat z_k = \begin{cases}
       \dfrac{b}{ s}\, \dfrac{z_k}{\Vert z \Vert_1}  & \text{if }   k\in A_r\\
       \dfrac{1 - b}{1 - s} \, \dfrac{z_k}{\Vert z \Vert_1}   & \text{if }   k\notin A_r
  \end{cases}.
    \end{equation}
  \end{enumerate} 
\end{restatable}

See proof in Subsection \ref{subsec:proof:Cbv1}.

\subsubsection{\texorpdfstring{Computing $\Proj_r$ for $r= n, n + 1, \cdots,m$}{computing projr for r=n,n+1,...,m}}

For each $r= n, n + 1, \cdots, m$, the convex set $C_r$  is of the form 
$C_r = \{x \in \Delta_n\,\mid\,   x_i - x_j \ge \delta\}$ where $1 \le i \not= j \le n$ and $-1 < \delta < 1$.
By setting $b := \delta$ and $v=[v_k] \in \mathbb{R}^n$ which is defined by:
\begin{equation}\label{eq:unv}
v_k := \begin{cases} 1 & \text{if } \quad k = i\\
-1 & \text{if } \quad k = j\\
 0 & \text{if } \quad k \notin\{i, j\} 
 \end{cases},
\end{equation}
we get that $C_r = C_{b, v}:= \{x\in \Delta_n \,\mid \, b \le \langle x, v \rangle \}$. As we know that 
$C_r \cap \mathbb{R}^n_{++}\not= \emptyset$ (Proposition~\ref{prop:f-box}),    we can apply Lemma \ref{lemma:KKT} and obtain with these notations: 

\begin{restatable}{lemma}{lemmaKKTOne}
\label{lemma:KKT1}\mbox{}
For  any   $z\in    \mathbb{R}^n_{++}$, set $\hat z = \Proj_{C_r}^f(z)$.

There is a pair $(\lambda^\star, \nu^\star)\in \mathbb{R}_+ \times \mathbb{R}$ such that for all $k\in\{1, 2, \dots, n\}$, we have:
\begin{equation}\label{eq:comphatyij}
 \hat z_k = \begin{cases} e^{\lambda^\star + \nu^\star} z_i & \text{if } \quad k = i\\
e^{ - \lambda^\star + \nu^\star} z_j &  \text{if } \quad k = j\\
 e^{ \nu^\star} z_k & \text{if } \quad k \notin\{i, j\} 
 \end{cases}.
\end{equation}
If $\delta  < \hat z_i - \hat z_j$, then $\lambda^\star = 0$ and $\hat z = z^\sharp$.
\end{restatable}

See proof in Subsection \ref{subsec:proof:lemmaKKTOne}.

To compute $\hat z = \Proj_{C_r}^f(z)$, we need the following elementary result: 
\begin{restatable}{lemma}{lemmaseconddegree}
\label{lemma:seconddegree}\mbox{}
Let  $0 < \omega < 1, \, 0 < \omega' < 1$ and  $-1  < \delta < 1$, set $u:= 1 - \omega - \omega'$. 

If  $\omega - \omega' < \delta$ and $u\ge 0$,  then 
the  positive root $E$ of the second degree polynomial $ \omega (1 - \delta)x^2 - u\delta x - \omega'(1 + \delta)$ is the unique solution of the  equation in $\mathbb{R}_{++}$:
 \[\dfrac{\omega x - \omega' x^{-1}}{\omega x + \omega' x^{-1} +u} = \delta\] 
and we have  $E > 1$.
\end{restatable}

See proof in Subsection \ref{subsec:proof:lemma:seconddegree}.

For each $r= n, n + 1, \cdots, m$,   the Bregman projection on   the convex set   
$C_r = \{x \in \Delta_n\,\mid\,   x_i - x_j \ge \delta\}$ where $1 \le i \not= j \le n$ and $-1 < \delta < 1$ is given by:

\begin{restatable}{proposition}{propprojCbvTwo}
\label{prop:projCbv2}
    Let $z\in \mathbb{R}^n_{++}$ and $\hat z:= \Proj_{C_r}^f(z)$ where $C_r:= \{x\in \Delta_n\,\mid\, x_i - x_j \ge \delta\}$ with $1 \le i \not= j \le n$ and $-1 < \delta < 1$. 

Set    $s := \dfrac{1}{\Vert z \Vert_1} (z_i - z_j)$,  $u := \Vert z \Vert_1 - z_i - z_j$ and notice that $u \ge 0$. 
\begin{enumerate}
    \item If  $s \ge \delta$ then  $z^\sharp\in C_r$ and then   $\hat z =   
    \Proj_{C_r}^f(z^\sharp) = z^\sharp$.
    \item If  $s < \delta$, then the equation in $\mathbb{R}_{++}$:
    \[ \dfrac{z_i x - z_j x^{-1}}{z_i x + z_j x^{-1} + u} = \delta\]
    admits  a unique solution   $E > 1$.  Set $D:=  z_i E + z_j E^{-1} + u$.
    
    The components of the vector $\hat z  = [\hat z_k]$ are given by

    \begin{equation}\label{eq:comp1}
      \hat z_k = \begin{cases}
      \dfrac{E}{D} z_i    & \text{ if }   k =i \\
      \dfrac{E^{-1}}{D} z_j   & \text{ if }   k = j\\
     \dfrac{1}{D} z_k   & \text{ if }   k \notin\{i, j\} \\
  \end{cases}.
    \end{equation}
\end{enumerate}
\end{restatable}

See proof in Subsection \ref{subsec:proof:prop:projCbv2}.

\subsection{Applying Dykstra’s algorithm with Bregman projections on \texorpdfstring{$\mathcal{F}^{\mathrm{box}}$}{Fbox}}
\label{subsec:dykstraalgo}

The algorithm of \cite[Theorem 3.2]{bauschke2000dykstras} will be applied with the negative entropy function $f$ and the convex set  $\mathcal{F}^{\mathrm{box}} = \bigcap_{i=1}^m C_i$, see (\ref{eq:specializationCi}) for the definition of the family of the closed convex sets $(C_i)_{i=1}^m$. Following \cite{bauschke2000dykstras}, we  extend      by $m$-periodicity the  finite family 
$(C_i)_{i=1}^m$ of convex subsets together with their Bregman projection $( \Proj_i)_{i=1}^m$ to sequences 
$(C_t)_{t \ge 1}$ and $(\Proj_t)_{t\ge 1}$:

Let $[\cdot]:\mathbb{N}\to\{1,\dots,m\}$ be the $m$-periodic map defined by
\[
[t] := 1 + ((t-1)\bmod m).
\]
Then $[t]=t$ for all $t\in\{1,\dots,m\}$ and $[t+m]=[t]$ for all $t\in\mathbb{N}$.\\
For all $t\ge 1$, set
\begin{equation}\label{eq:projt}
C_t := C_{[t]} \quad \text{and} \quad \Proj_t := \Proj_{C_{[t]}}^f.    
\end{equation}
\begin{algo}[Dykstra's algorithm with Bregman projections {\cite[Theorem~3.2]{bauschke2000dykstras}}]\label{ex:algo:dykstra}
Let $f$ be the negative entropy function and let $C_1,\dots,C_m$ be non-empty closed convex sets such that
\[
\bigcap_{i=1}^m C_i \,\cap\, \mathbb{R}^n_{++}\neq\varnothing.
\]
Let $z^{(0)}\in \mathbb{R}^n_{++}$ and set
\[
d^{(-(m-1))}=\dots=d^{(-1)} =d^{(0)}:=0.
\]
\end{algo}
For $t\ge 1$,  perform the updates 
\begin{equation}\label{eq:AlgoD0}
z^{(t)}   = (\Proj_{t} \circ \nabla f^\ast)(\nabla f(z^{(t -1)}) + d^{(t - m)}), \quad d^{(t)}   = \nabla f(z^{(t - 1)}) + d^{(t - m)} - \nabla f(z^{(t)}).    
\end{equation}
The fundamental result of \cite[Theorem~3.2]{bauschke2000dykstras} is that this algorithm converges: the sequence $(z^{(t)})_{t\ge 1}$ converges in $\mathbb{R}^n_{++}$ to 
$\Proj_{\cap _i \,C_i}(z^{(0)})$.\\
For the negative entropy $f(x)=\sum_{k=1}^n x_k\log x_k$ and the closed  convex subsets $C_1, C_2, \dots C_m$ such that     $\bigcap_{i=1}^m C_i \,\cap\, \mathbb{R}^n_{++}\neq\varnothing$,     we  give an explicit formula for $z^{(t)} $ in terms of $z^{(0)}, z^{(1)}, \cdots, z^{(t - 1)}$ and the $m$  Bregman projections $\Proj_{i}$. 

We use the following notations: 
\begin{notation}\label{notation:seqs}
\mbox{}
    \begin{itemize}
        \item For any vector $u =[u_k]\in \mathbb{R}^n$, we denote by  $\text{exp}(u) \in \mathbb{R}^n_{++}$ the vector
$\text{exp}(u) := (e^{u_1}, e^{u_2}, \dots, e^{u_n})$.
        \item For any vectors  $u =[u_k]\in \mathbb{R}^n$ and  $v =[v_k]\in \mathbb{R}^n$, we denote by  $u \,.\,  v$ the vector
$u \,.\, v := (u_1  v_1, u_2  v_2, \dots, u_n  v_n)$.
\item For any vectors $u =[u_k]\in \mathbb{R}^n_{++}$ and $v =[v_k]\in \mathbb{R}^n_{++}$, we denote by $\dfrac{u}{v}$
and $\log \dfrac{u}{v}$ respectively the vector $\dfrac{u}{v}:= (\dfrac{u_1}{v_1}, \dfrac{u_2}{v_2}, \dots \dfrac{u_n}{v_n}) $ and the vector
$\log \dfrac{u}{v} := (\log \dfrac{u_1}{v_1}, \log \dfrac{u_2}{v_2}, \dots, \log \dfrac{u_n}{v_n})$.
    \end{itemize}
\end{notation}

By easy computation, one can check the following result that we will use:
\begin{lemma}\label{lemma:explog}
 For any vector systems $(u^{(s)})_{1 \le s \le T}$  and $(v^{(s)})_{1 \le s \le T}$ in $\mathbb{R}^n_{++}$, we have:
 \begin{equation}\label{eq:explog}
 \log(\prod_{s= 1}^T \dfrac{u^{(s)}}{v^{(s)}}) =  \sum_{s= 1}^T \log(\dfrac{u^{(s)}}{v^{(s)}}), \quad 
 \text{exp}(\log \prod_{s= 1}^T \dfrac{u^{(s)}}{v^{(s)}}) = \prod_{s= 1}^T \dfrac{u^{(s)}}{v^{(s)}}.
     \end{equation}
\end{lemma}
\qed

\begin{lemma}\label{lemma:AlgoD1}
Algorithm~\ref{ex:algo:dykstra} can be equivalently written as follows: for
each $t\ge 1$, set:
\begin{equation}\label{eq:ut}
 u^{(t)}  = z^{(t-1)} \,.\,  \text{exp}( d^{(t-m)}) \quad \text{with  $d^{(-(m-1))}=\dots=d^{(-1)} =d^{(0)}:=0$}, 
\end{equation}
then (\ref{eq:AlgoD0}) is equivalent to

\begin{equation}\label{eq:AlgoD2}
 z^{(t)} = \Proj_t\!\bigl(u^{(t)}\bigr),  \quad  
d^{(t)} = d^{(t-m)} + \log\!\left(\frac{z^{(t-1)}}{z^{(t)}}\right) \quad \text{for all $t\ge 1$}.
\end{equation}    
\end{lemma}
\begin{proof}
As for  every $z=(z_1,\dots,z_n)\in\mathbb{R}^n_{++}$, $y=(y_1,\dots,y_n)\in\mathbb{R}^n$
  and   $k\in\{1,\dots,n\}$, we have 
\[
(\nabla f(z))_k = \log(z_k)+1 \quad \text{and} \quad (\nabla f^\ast(y))_k = \exp(y_k-1).
\]
We establish the two vector equalities (\ref{eq:AlgoD2})   component by component using  (\ref{eq:AlgoD0}) and  Notation \ref{notation:seqs}. 
\end{proof}

\begin{restatable}{proposition}{propzd}
\label{prop:zd}
 With the convention $z^{(t)}:=z^{(0)}$  and $d^{(t)}:=0$ for all $t\le 0$, the following formulas hold: 

For any $j\ge 1$ and any
$h\in\{1,\dots,m\}$, letting $t=(j-1)m+h$, we have:
\begin{equation}\label{eq:AlgoD3}
u^{(t)}
= \begin{cases}
    z^{(t-1)}\, . \,
\prod_{\ell=0}^{j-2}
\frac{z^{(\ell m+h-1)}}{z^{(\ell m+h)}}   & \text{ if } j > 1 \\
z^{(t-1)} & \text{ if } j = 1
\end{cases}, \quad
z^{(t)}=\Proj_h\!\bigl(u^{(t)}\bigr), 
  \quad d^{(t)}
=
\log\!\left(
\prod_{\ell=0}^{j-1}
\frac{z^{(\ell m+h-1)}}{z^{(\ell m+h)}}
\right).
    \end{equation} 
\end{restatable}
See proof in Subsection \ref{subec:proof:prop:zd}.

\begin{example}\label{ex:kldiv-and-dykstra-n3}
(Example~\ref{ex:fbox-eps-and-dotg}, cont'ed)\\
Fix $\varepsilon:=0.001$ (so that
$0<\varepsilon\le 0.005$) and consider the corresponding set
$\mathcal F^{\mathrm{box}}\subseteq\Delta_3$ defined in (\ref{eq:exfbox}):
\[   \mathcal F^{\mathrm{box}}
=
\Bigl\{p\in\Delta_3\ \Bigm| 
p_1 \ge 0.49, \ \
p_1+p_2 \ge 0.50,\ \
0.001 \le p_1-p_2 \le 0.49,\ \
0.001 \le p_2-p_3 \le 0.005\Bigr\}.\]
Let
\[
q=(0.48,\,0.261,\,0.259)\in\Delta_3\cap\mathbb{R}^3_{++}.
\]
For $n=3$ we have $m=3n-3=6$, and the constraint sets in \eqref{eq:specializationCi} are
\[
C_1=\{p\in\Delta_3:\ p_1\ge 0.49\},\qquad
C_2=\{p\in\Delta_3:\ p_1+p_2\ge 0.50\},
\]
\[
C_3=\{p\in\Delta_3:\ p_1-p_2\ge 0.001\},\qquad
C_4=\{p\in\Delta_3:\ p_2-p_3\ge 0.001\},
\]
\[
C_5=\{p\in\Delta_3:\ p_1-p_2\le 0.49\},\qquad
C_6=\{p\in\Delta_3:\ p_2-p_3\le 0.005\},
\]
so that $\mathcal F^{\mathrm{box}}=\bigcap_{i=1}^6 C_i$.
By construction,
\[
q\notin C_1 \quad (\text{since } q_1=0.48<0.49),
\]
while $q\in C_2\cap C_3\cap C_4\cap C_5\cap C_6$ (indeed $q_2-q_3=0.002\in[0.001,0.005]$).
\medskip
We illustrate the equivalent formulation of Algorithm~\ref{ex:algo:dykstra} given in
Lemma~\ref{lemma:AlgoD1}:
\[
u^{(t)} = z^{(t-1)}\,.\,\exp\!\bigl(d^{(t-m)}\bigr),
\qquad
z^{(t)}=\Proj_t\!\bigl(u^{(t)}\bigr),
\qquad
d^{(t)}=d^{(t-m)}+\log\!\left(\frac{z^{(t-1)}}{z^{(t)}}\right),
\]
initialized with $z^{(0)}=q$ and $d^{(-(m-1))}=\cdots=d^{(0)}=0$.\\
\medskip
\noindent\emph{Step $t=1$ (projection onto $C_1$).}\\
Since $m=6$, we have $d^{(1-m)}=d^{(-5)}=0$, hence
\[
u^{(1)}=z^{(0)}\,.\,\exp(d^{(-5)})=z^{(0)}=q.
\]
We compute $z^{(1)}=\Proj_{C_1}^f(u^{(1)})=\Proj_{C_1}^f(q)$ explicitly.
Here $C_1=\{p\in\Delta_3:\ p_1\ge b\}$ with $b:=0.49$ and $A_1=\{1\}$.
Since $\|q\|_1=1$, we have $q^\sharp=q$ and
\[
s:=\frac{1}{\|q\|_1}\sum_{k\in A_1} q_k = q_1 = 0.48 < b,
\]
so we are in case~(2) of Proposition~\ref{prop:projCbv1} with $r=1$. Therefore,
\[
z^{(1)}_1=\frac{b}{s}\,\frac{q_1}{\|q\|_1}=\frac{0.49}{0.48}\,0.48=0.49,
\qquad
z^{(1)}_k=\frac{1-b}{1-s}\,\frac{q_k}{\|q\|_1}=\frac{0.51}{0.52}\,q_k
\quad \text{for } k\in\{2,3\}.
\]
We obtain, rounded to $4$ decimals:
\[
z^{(1)}=\Proj_{C_1}^f(q)=(0.49,\,0.2560,\,0.2540).
\]
The corresponding dual update is
\[
d^{(1)}=\log\!\left(\frac{z^{(0)}}{z^{(1)}}\right)
=\log\!\left(\frac{q}{z^{(1)}}\right)
\simeq (-0.0206,\,0.0193,\,0.0195),
\]
since
\[
\exp(d^{(1)})=\frac{q}{z^{(1)}}\simeq (0.9796,\,1.0195,\,1.0197).
\]
\medskip
\noindent We can see that the other constraints are inactive after step $t=1$:
a direct check shows that $z^{(1)}\in C_2\cap\cdots\cap C_6$:
\[
z^{(1)}_1+z^{(1)}_2 = 0.49+0.2560 = 0.7460 \ge 0.50,
\]
\[
z^{(1)}_1-z^{(1)}_2 = 0.49-0.2560 = 0.2340 \ge 0.001
\quad\text{and}\quad
z^{(1)}_1-z^{(1)}_2 = 0.2340 \le 0.49,
\]
\[
z^{(1)}_2-z^{(1)}_3 = 0.2560-0.2540 = 0.0020 \in [0.001,\,0.005].
\]
Hence $z^{(1)}\in\mathcal F^{\mathrm{box}}$. Since $z^{(1)}=\Proj_{C_1}^f(q)$ and
$\mathcal F^{\mathrm{box}}\subseteq C_1$, it follows that
\[
p^\star=\arg\min_{p\in\mathcal F^{\mathrm{box}}}D_{\mathrm{KL}}(p\|q)=z^{(1)}.
\]
Indeed, for $t=2,\dots,6$ we have $d^{(t-m)}=d^{(t-6)}=0$, hence $u^{(t)}=z^{(t-1)}$.
Moreover, since $z^{(1)}\in C_2\cap\cdots\cap C_6$, each projection is inactive:
\[
z^{(2)}=\Proj_{C_2}^f(z^{(1)})=z^{(1)},\quad
z^{(3)}=\Proj_{C_3}^f(z^{(2)})=z^{(1)},\ \dots,\ 
z^{(6)}=\Proj_{C_6}^f(z^{(5)})=z^{(1)}.
\]
Accordingly, $d^{(t)}=0$ for $t=2,\dots,6$ because $z^{(t)}=z^{(t-1)}$ and $d^{(t-6)}=0$.
\end{example}

\section{Experiments}
\label{sec:numerical}

This section presents three empirical studies:
\begin{itemize}
    \item The first study aims to provide an empirical evaluation of Algorithm~\ref{ex:algo:dykstra} on synthetic instances.
    \item The second study considers a synthetic learning problem with possibilistic annotations.
We use a multi-class classification setting where $\mathcal Y:=\{1,\dots,n\}$ denotes the class set, and where each training item is a pair $(x,\pi)$, with $x\in\mathbb{R}^d$ a feature vector and $\pi$ a normalized possibility distribution on $\mathcal Y$. In this study, we compare two training objectives: 
Model A uses a projection-based target: given the current prediction $q(x)$ of the model on input $x$, it defines the target as the Kullback--Leibler projection of $q(x)$ onto the admissible set $\mathcal F^{\mathrm{box}}(\pi)$ induced by $\pi$, thereby producing a probability distribution consistent with the possibilistic information (Figure~\ref{fig:learning-pipeline}).
Model B uses a fixed probabilistic target, namely the antipignistic probability distribution $\dot p(\pi)\in\Delta_n$, see~\eqref{eq:ppoint}, derived from $\pi$.
\item The third study evaluates the same projection-based approach on a real Natural Language Inference (NLI) task \cite{bowman2015large} derived from the ChaosNLI dataset \cite{nie2020can}, where possibilistic annotations are constructed from crowd vote distributions. In this study, we compare the projection-based objective of Model A with two fixed-target baselines: Model B, which uses the antipignistic probability derived from the possibilistic annotation, and Model C, which uses the normalized vote proportions directly.
\end{itemize}

All experiments are conducted in Python~3.12 on a Mac M2 (16\,GB RAM) \footnote{Code is available at \url{https://github.com/ibaaj/probabilistic-classification-from-possibilistic-data}.}.
The Kullback--Leibler projections are computed using a compiled C++ implementation of Dykstra's algorithm interfaced with the Python code, and small positive clipping of probabilities (at level $10^{-15}$) is used whenever logarithms are evaluated.

The rest of the section is structured as follows.
Subsection~\ref{subsec:numerical:computingzt} describes the numerically stable implementation of Dykstra’s iterations used throughout.
Subsection~\ref{subsec:experimentalalgoone} empirically evaluates Algorithm~\ref{ex:algo:dykstra}.
Subsection~\ref{subsec:topk-learning} presents the synthetic learning experiment: dataset generation, training objectives, training and projection settings, and the evaluation methodology used to compare the two models in terms of predictive performance.
Subsection~\ref{subsec:chaosnli} presents the ChaosNLI experiment and compares projection-based and fixed-target objectives on real data with naturally ambiguous annotations.

\subsection{Practical implementation of Dykstra's algorithm with our  projections}
\label{subsec:numerical:computingzt}
This subsection specifies the implementation of Dykstra’s algorithm
(Algorithm~\ref{ex:algo:dykstra}, equivalently written in~\eqref{eq:AlgoD2})
used to compute the Kullback-Leibler projection on
$\mathcal{F}^{\mathrm{box}}$.

While the iterates are defined by the update (\ref{eq:AlgoD2}):
\[
u^{(t)} = z^{(t-1)} \, . \, \exp\!\bigl(d^{(t-m)}\bigr), 
\]
the direct evaluation of this expression may overflow in floating-point arithmetic during long runs when some components of $d^{(t-m)}$ become large. 

We therefore introduce a numerically stable reformulation for the computation of the vector $u^{(t)}$ used in
\eqref{eq:AlgoD2}. 
We compute $u^{(t)}$ from its logarithm and a constant shift.\\
Let
\[
\ell^{(t)}_k := \log u^{(t)}_k   = \log z^{(t-1)}_k + d^{(t-m)}_k,
\qquad
c_t = \max_{1 \le k \le n} \ell^{(t)}_k,
\]
and define
\[
\widetilde u^{(t)}_k = \exp\!\bigl(\ell^{(t)}_k - c_t\bigr), \qquad k\in \{1 , 2 , \dots  , n\}.
\]
Then $\widetilde u^{(t)} = e^{-c_t}u^{(t)}$, i.e., $\widetilde u^{(t)}$ differs from $u^{(t)}$ only by a strict positive scalar factor.\\
This scalar factor does not alter the subsequent projection step thanks to Lemma \ref{lemma:Projty} and Corollary \ref{cor:projCzzNat}: 
\[
z^{(t)}=\mathrm{Proj}^{f}_{C_{t}}\!\bigl(u^{(t)}\bigr)=\mathrm{Proj}^{f}_{C_{t}}\!\bigl(\widetilde u^{(t)}\bigr).
\]
The subtraction of $c_t$ is only introduced to prevent overflow in the evaluation of $\exp(\ell^{(t)}_k)$.

\subsection{Experimental protocol for evaluating Algorithm~\ref{ex:algo:dykstra} on synthetic data}
\label{subsec:experimentalalgoone}

For a given configuration (dimension $n$, tolerance parameter $\tau \in [0,1]$, and
maximal number of Dykstra cycles $K_{\max}$), we perform 100 independent runs of the following experiment:

\begin{enumerate}
  \item 
  We first draw at random a strictly positive normalized possibility distribution $\pi$ on $\mathcal Y:=\{1,\dots,n\}$.   We then compute:
  \begin{itemize}
    \item the permutation $\sigma$ that sorts $\pi$ in nonincreasing order,
    \[
      \pi_{\sigma(1)} \ge \pi_{\sigma(2)} \ge \cdots \ge \pi_{\sigma(n)} > 0,
    \]
    \item the sorted possibility levels
    \(\tilde\pi_r := \pi_{\sigma(r)}\), $r=1,\dots,n$, with $\tilde\pi_1=1$,
    \item the antipignistic reverse probability distribution
      $\dot p\in\Delta_n$ associated with $\tilde\pi$ by~\eqref{eq:p_i},
    \[
      \dot p_{\sigma(r)}
      :=
      \sum_{j=r}^n \frac{\tilde\pi_j-\tilde\pi_{j+1}}{j},
      \qquad r=1,\dots,n,
    \]
      with the convention $\tilde\pi_{n+1}=0$,
    \item the adjacent gaps of $\dot p$ in the $\pi$-order
    \[
      \dot g_r
      := \frac{\tilde\pi_r-\tilde\pi_{r+1}}{r},
      \qquad r=1,\dots,n-1.
    \]
  \end{itemize}
  Since the possibility distribution $\pi$ is strictly positive, we have $0 \le \dot g_r < 1$ for all $r$, and $\dot g_r = 0$ exactly when $\tilde\pi_r = \tilde\pi_{r+1}$.

  \item 
  From $(\tilde\pi_r)_{r=1}^n$, whose associated adjacent gaps are $(\dot g_r)_{r=1}^{n-1}$ (Section~\ref{sec:prob-constraints}),  we distinguish indices where two consecutive values are equal and indices where a strict decrease occurs:
  \[
    \mathcal R_{\mathrm{equal}}
    := \{ r\in\{1,\dots,n-1\}: \tilde\pi_r=\tilde\pi_{r+1} \},
    \quad
    \mathcal R_{\mathrm{strict}}
    := \{r\in\{1,\dots,n-1\}: \tilde\pi_r>\tilde\pi_{r+1}\}.
  \]
   In the implementation, the set $\mathcal R_{\mathrm{equal}}$ is determined using a tolerance parameter \texttt{tie\_tol}: two adjacent values are treated as equal whenever $|\tilde\pi_r-\tilde\pi_{r+1}| \le \texttt{tie\_tol}$. Here, \texttt{tie\_tol}=0, so $\mathcal R_{\mathrm{equal}}$ corresponds to exact equalities. We set
  \[
    g_{\min} := \min_{r\in\mathcal R_{\mathrm{strict}}} \dot g_r,
    \qquad
    g_{\max} := \max_{r\in\mathcal R_{\mathrm{strict}}} \dot g_r,
    \qquad
    \varepsilon := \min\!\bigl(10^{-9},\ g_{\min},\ 1-g_{\max}\bigr),
  \]
  whenever $\mathcal R_{\mathrm{strict}}\neq\emptyset$; if
  $\mathcal R_{\mathrm{strict}}=\emptyset$ we take $\varepsilon=0$. 
With this choice, $\varepsilon\le \dot g_r \le 1-\varepsilon$ for every $r\in\mathcal R_{\mathrm{strict}}$.
We then set, for $r=1,\dots,n-1$,
\[
\underline\delta_r=
\begin{cases}
\varepsilon, & r\in \mathcal R_{\mathrm{strict}},\\
0, & r\in \mathcal R_{\mathrm{equal}},
\end{cases}
\qquad
\overline\delta_r=
\begin{cases}
1-\varepsilon, & r\in \mathcal R_{\mathrm{strict}},\\
0, & r\in \mathcal R_{\mathrm{equal}}.
\end{cases}
\]
By construction, $\underline{\delta}_r \le \dot g_r \le \overline{\delta}_r$ for all $r\in\mathcal R_{\mathrm{strict}}$, and $\underline{\delta}_r=\overline{\delta}_r=0$ on $\mathcal R_{\mathrm{equal}}$. Hence the antipignistic probability distribution $\dot p$ belongs to $\mathcal{F}^{\mathrm{box}}$.

  \item For the analysis, we encode all constraints defining $\mathcal{F}^{\mathrm{box}}$ as a single linear system
  \[
    A p \ge b,
    \qquad A\in\mathbb{R}^{m\times n},\ b\in\mathbb{R}^m,
  \]
  where $p\in\mathbb{R}^n$ is the unknown and
  \[
    m
    = (n-1)\ \text{(dominance constraints)}
      + (n-1)\ \text{(lower gaps)}
      + (n-1)\ \text{(upper gaps)}
    = 3n-3.
  \]

  We use a block structure
  \[
    A :=
    \begin{bmatrix}
      A^{\mathrm{pref}}\\[2pt]
      D^{\mathrm{low}}\\[2pt]
      D^{\mathrm{up}}
    \end{bmatrix},
    \qquad
    b :=
    \begin{bmatrix}
      b^{\mathrm{pref}}\\[2pt]
      b^{\mathrm{low}}\\[2pt]
      b^{\mathrm{up}}
    \end{bmatrix},
  \]
  corresponding to the families
  $C^{\mathrm{pref}}_s$, $C^{\mathrm{low}}_s$ and $C^{\mathrm{up}}_s$ in
  \eqref{eq:Cprefsdef}–\eqref{eq:Cupdef}.

  \smallskip
  \noindent
  For each $r=1,\dots,n-1$ we denote
  \(A_r:=\{\sigma(1),\dots,\sigma(r)\}\). The dominance constraints
  \[
    \sum_{k\in A_r} p_k \ge 1-\tilde\pi_{r+1},
    \qquad r=1,\dots,n-1,
  \]
  are encoded by a matrix $A^{\mathrm{pref}}\in\{0,1\}^{(n-1)\times n}$ and
  a vector $b^{\mathrm{pref}}\in\mathbb{R}^{n-1}$ defined by
  \[
    A^{\mathrm{pref}}_{r,k} :=
    \begin{cases}
      1, & k\in A_r,\\
      0, & k\notin A_r,
    \end{cases}
    \qquad
    b^{\mathrm{pref}}_r := 1-\tilde\pi_{r+1},
    \qquad r=1,\dots,n-1,\ k=1,\dots,n.
  \]

  \smallskip
  \noindent
  The gap constraints in the $\pi$-order are encoded by a
  matrix $D\in\{-1,0,1\}^{(n-1)\times n}$:
  \[
    D_{r,k} :=
    \begin{cases}
      +1, & k=\sigma(r),\\
      -1, & k=\sigma(r+1),\\
      0,  & \text{otherwise},
    \end{cases}
    \qquad r=1,\dots,n-1,\ k=1,\dots,n.
  \]
  We then set
  \[
    D^{\mathrm{low}} := D,
    \quad
    b^{\mathrm{low}}_r := \underline{\delta}_r,
    \qquad
    D^{\mathrm{up}} := -D,
    \quad
    b^{\mathrm{up}}_r := -\overline{\delta}_r,
    \qquad r=1,\dots,n-1.
  \]
  The rows of $D^{\mathrm{low}}$ enforce lower bounds
  \(
    p_{\sigma(r)}-p_{\sigma(r+1)} \ge \underline{\delta}_r
  \),
  while the rows of $D^{\mathrm{up}}$ encode the upper bounds
  \(
    p_{\sigma(r)}-p_{\sigma(r+1)} \le \overline{\delta}_r
  \).

  \item We draw a strictly positive reference probability
  vector $q\in\Delta_n$ representing the output of a probabilistic classifier.

   \item
Starting from $z^{(0)}:=q\in \Delta_n\cap\mathbb{R}^n_{++}$, we compute the Kullback-Leibler projection
\[
p^\star=\Proj_{\mathcal{F}^{\mathrm{box}}}^f(q)
\]
by applying Dykstra's algorithm with Bregman projections (Algorithm~\ref{ex:algo:dykstra}) to the finite family of constraint sets $(C_i)_{i=1}^m$ defined in~\eqref{eq:specializationCi}. Recall that
\[
\mathcal{F}^{\mathrm{box}}=\bigcap_{i=1}^m C_i,
\qquad m=3n-3,
\]
with the $m$-periodic indexing $[t]$ introduced in Section~\ref{subsec:dykstraalgo} and $\Proj_t=\Proj_{C_{[t]}}^f$.

We implement Algorithm~\ref{ex:algo:dykstra} as in~(\ref{eq:AlgoD2}): we initialize $d^{(t)}:=0$ for all $t\le 0$ and, for each $t\ge 1$, we set
\[
u^{(t)} := z^{(t-1)}\,.\,\exp\!\bigl(d^{(t-m)}\bigr),
\]
and then update
\[
z^{(t)} := \Proj_t\!\bigl(u^{(t)}\bigr),
\qquad
d^{(t)} := d^{(t-m)} + \log\!\left(\frac{z^{(t-1)}}{z^{(t)}}\right).
\]
In practice, the vector $u^{(t)}$ is evaluated by the stabilized computation described in
Subsection~\ref{subsec:numerical:computingzt} (see the definition of $\widetilde u^{(t)}$), and we apply $\Proj_t$ to $\widetilde u^{(t)}$ instead of $u^{(t)}$.

One \emph{cycle} corresponds to $m$ successive updates, i.e., one pass over
$C^{\mathrm{pref}}_1,\dots,C^{\mathrm{pref}}_{n-1}$, then
$C^{\mathrm{low}}_1,\dots,C^{\mathrm{low}}_{n-1}$, then
$C^{\mathrm{up}}_1,\dots,C^{\mathrm{up}}_{n-1}$.

Each Bregman projector $\Proj_{C_i}^f$ is evaluated using the formulas derived in
Section~\ref{sec:bregman-dykstra}:
\begin{itemize}
  \item If $1\le i\le n-1$ (i.e.\ $C_i=C^{\mathrm{pref}}_i$), we use Proposition~\ref{prop:projCbv1}.
  \item If $n\le i\le 2n-2$, $r$ is set to $r:=i-(n-1)\in\{1,\dots,n-1\}$ (so $C_i=C^{\mathrm{low}}_r$) and we use
  Proposition~\ref{prop:projCbv2} (formula~\eqref{eq:comp1}) with
  $(i',j',\delta)=(\sigma(r),\sigma(r+1),\underline{\delta}_r)$.
  \item If $2n-1\le i\le 3n-3$, $r$ is set to $r:=i-(2n-2)\in\{1,\dots,n-1\}$ (so $C_i=C^{\mathrm{up}}_r$) and we use
  Proposition~\ref{prop:projCbv2} (formula~\eqref{eq:comp1}) with
  $(i',j',\delta)=(\sigma(r+1),\sigma(r),-\overline{\delta}_r)$.
\end{itemize}
\item
After each Dykstra cycle we evaluate the maximal constraint violation. Since the constraints defining $\mathcal{F}^{\mathrm{box}}$ are written as
\[
A p \ge b,
\qquad A\in\mathbb{R}^{m\times n},\quad b\in\mathbb{R}^m,
\]
we denote by $a_i^\top\in\mathbb{R}^{1\times n}$ the $i$-th row of $A$ (so that $(Ap)_i= \langle a_i^\top  , p\rangle$) and by $b_i$ the $i$-th component of $b$.
We set
\[
\mathrm{V}(p)
:= \max_{1\le i\le m} \bigl(b_i - \langle a_i^\top  , p\rangle\bigr)_+,
\qquad (x)_+:=\max(x,0).
\]
The iteration is stopped as soon as $\mathrm{V}(p)\le\tau$ using the  tolerance parameter $\tau$, or when $K_{\max}$ cycles have been performed. If a run does not satisfy $\mathrm{V}(p)\le\tau$ within $K_{\max}$ cycles, we record its cycle count as $K_{\max}$ and return the last iterate (after $K_{\max}$ cycles) as the output of the run; the corresponding value $\mathrm{V}(p)$ is reported as its final maximal violation.\\
For each run we record:
\begin{itemize}
  \item the final maximal violation $\mathrm{V}(p)$,
  \item the number of completed Dykstra cycles,
  \item the computation time.
\end{itemize}
A run is counted as \emph{converged} if $\mathrm{V}(p)\le\tau$ at termination. 
\end{enumerate}

\subsubsection{Results}

We repeat the above experiment for five values of the tolerance parameter $\tau \in \{10^{-2},10^{-3},10^{-4},10^{-6},10^{-8}\}$
and for three values of the maximal number of Dykstra cycles
$K_{\max}\in\{10^3,10^4,5\cdot 10^4\},$
keeping $n=100$ fixed.

For each value of $K_{\max}$, we use the same initial random seed so that the underlying random instances are comparable across different cycle budgets. For a fixed $K_{\max}$ and each tolerance level $\tau$, we perform $100$ independent runs and aggregate the following statistics:
\begin{itemize}
  \item convergence rate (fraction of runs with $\mathrm{V}(p)\le\tau$),
  \item mean and $90$th percentile of the number of Dykstra cycles,
  \item mean of the final maximal violation $\mathrm{V}(p)$,
  \item mean computation time.
\end{itemize}
The results are reported in Tables~\ref{tab:dykstra-K1000}-\ref{tab:dykstra-K50000}.

\begin{table}[H]
\centering
\begin{tabular}{lrrrrr}
\toprule
Tolerance $\tau$ & Convergence rate & Mean cycles & 90th perc. cycles & Mean final violation & Mean time [s]\\
\midrule
$1e-02$ & 1.000 & 1.1 & 1.0 & 6.04e-03 & 0.000\\
$1e-03$ & 1.000 & 11.9 & 16.1 & 9.42e-04 & 0.004\\
$1e-04$ & 1.000 & 143.9 & 215.4 & 9.93e-05 & 0.050\\
$1e-06$ & 0.360 & 881.9 & 1000.0 & 6.22e-06 & 0.309\\
$1e-08$ & 0.080 & 973.9 & 1000.0 & 6.24e-06 & 0.341\\
\bottomrule
\end{tabular}
\caption{KL projection on $\mathcal{F}^{\mathrm{box}}$ for $n=100$, with a maximum of $K_{\max}=1000$ Dykstra cycles and 100 random runs per tolerance.}
\label{tab:dykstra-K1000}
\end{table}

\begin{table}[H]
\centering
\begin{tabular}{lrrrrr}
\toprule
Tolerance $\tau$ & Convergence rate & Mean cycles & 90th perc. cycles & Mean final violation & Mean time [s]\\
\midrule
$1e-02$ & 1.000 & 1.1 & 1.0 & 6.04e-03 & 0.000\\
$1e-03$ & 1.000 & 11.9 & 16.1 & 9.42e-04 & 0.004\\
$1e-04$ & 1.000 & 143.9 & 215.4 & 9.93e-05 & 0.050\\
$1e-06$ & 1.000 & 1593.5 & 2971.6 & 9.98e-07 & 0.558\\
$1e-08$ & 0.970 & 2934.1 & 5077.5 & 1.18e-08 & 1.028\\
\bottomrule
\end{tabular}
\caption{KL projection on $\mathcal{F}^{\mathrm{box}}$ for $n=100$, with a maximum of $K_{\max}=10000$ Dykstra cycles and 100 random runs per tolerance.}
\label{tab:dykstra-K10000}
\end{table}

\begin{table}[H]
\centering
\begin{tabular}{lrrrrr}
\toprule
Tolerance $\tau$ & Convergence rate & Mean cycles & 90th perc. cycles & Mean final violation & Mean time [s]\\
\midrule
$1e-02$ & 1.000 & 1.1 & 1.0 & 6.04e-03 & 0.000\\
$1e-03$ & 1.000 & 11.9 & 16.1 & 9.42e-04 & 0.004\\
$1e-04$ & 1.000 & 143.9 & 215.4 & 9.93e-05 & 0.050\\
$1e-06$ & 1.000 & 1593.5 & 2971.6 & 9.98e-07 & 0.558\\
$1e-08$ & 1.000 & 3047.9 & 5077.5 & 9.97e-09 & 1.068\\
\bottomrule
\end{tabular}
\caption{KL projection on $\mathcal{F}^{\mathrm{box}}$ for $n=100$, with a maximum of $K_{\max}=50000$ Dykstra cycles and 100 random runs per tolerance.}
\label{tab:dykstra-K50000}
\end{table}

In this experiment we deliberately construct a \emph{wide} feasible set
$\mathcal F^{\mathrm{box}}$ by choosing the gap bounds
$(\underline{\delta}_r,\overline{\delta}_r)$ so that
$\underline{\delta}_r\le \dot g_r\le \overline{\delta}_r$ while keeping
$\underline{\delta}_r$ small and $\overline{\delta}_r$ close to $1$ whenever $\tilde\pi_r>\tilde\pi_{r+1}$. This setting makes it possible to evaluate the projection procedure on a large admissible set in a regime where the constraints are relatively easy to satisfy.

Tables~\ref{tab:dykstra-K1000}--\ref{tab:dykstra-K50000} summarize the performance of Dykstra’s algorithm for $n=100$ over $100$ random instances per tolerance level, for several cycle budgets $K_{\max}$.
For moderate tolerances, $\tau\in\{10^{-2},10^{-3},10^{-4}\}$, all runs converge for every value of $K_{\max}$.
In this regime, the number of cycles remains small (mean $\approx 1.1$, $11.9$, and $143.9$, respectively), and the mean computation time stays very small: it remains negligible for $\tau=10^{-2}$, around $0.004$s for $\tau=10^{-3}$, and around $0.050$s for $\tau=10^{-4}$.

For stricter tolerances, the cycle budget becomes the main limiting factor.
At $\tau=10^{-6}$ with $K_{\max}=1000$, only $36\%$ of runs satisfy the stopping criterion, and the $90$th percentile of the cycle count reaches the budget limit, which indicates that many runs terminate because the maximum number of cycles is reached.
Increasing the budget to $K_{\max}=10^4$ yields convergence for all runs, with a mean cycle count of about $1594$ and a $90$th percentile around $2972$.
The corresponding mean final violation is below $10^{-6}$, and the mean computation time remains below one second (about $0.558$s).

At $\tau=10^{-8}$, the same pattern is more pronounced.
With $K_{\max}=1000$, the convergence rate is only $8\%$, whereas it rises to $97\%$ for $K_{\max}=10^4$ and reaches $100\%$ for $K_{\max}=5\cdot 10^4$.
The mean cycle count then increases from about $2934$ to about $3048$, while the mean final violation decreases from $1.18\times 10^{-8}$ to $9.97\times 10^{-9}$.
Even in this most demanding setting, the mean computation time remains close to one second (about $1.028$s for $K_{\max}=10^4$ and $1.068$s for $K_{\max}=5\cdot 10^4$).

Overall, these results illustrate the expected trade-off between stricter stopping criteria and computation time induced by the choice of $(\tau,K_{\max})$.
Moderate tolerances are reached very quickly, whereas very small tolerances require several thousand Dykstra cycles.
At the same time, the compiled implementation keeps the computation time low across all tested settings, with mean runtimes ranging from a few milliseconds to about one second.

\subsection{Learning from possibilistic supervision on synthetic data}
\label{subsec:topk-learning}

We consider a synthetic multi-class classification setting, where $\mathcal Y:=\{1,\dots,n\}$ denotes  the class set.
In the learning problem, the observed supervision is a pair $(x,\pi)$, where $x\in\mathbb{R}^d$ is a feature vector and $\pi$ is a strictly positive normalized possibility distribution on $\mathcal Y$ represented as a vector such that:
\[
\pi=(\pi_1,\dots,\pi_n)\in(0,1]^n,\qquad \pi_i>0 \text{ for } i=1,\dots,n,\qquad \max_{1\le j\le n}\pi_j=1.
\]
In the synthetic experiments, each sample is generated together with a ground-truth label $c\in\mathcal Y$ used only for evaluation (see Subsection \ref{subsubsec:synthetic-data} for the dataset generation process); thus we store triplets $(x,c,\pi)$, but the training objectives depend only on $(x,\pi)$. 

From $\pi$ we construct the admissible set $\mathcal F^{\mathrm{box}}(\pi)\subseteq\Delta_n$  as in Subsection~\ref{subsec:experimentalalgoone}. Writing $\pi_{\sigma(1)}\ge\cdots\ge\pi_{\sigma(n)}$ and $\tilde\pi_r:=\pi_{\sigma(r)}$, we define the adjacent gaps by $\dot g_r:=(\tilde\pi_r-\tilde\pi_{r+1})/r$ for $r=1,\dots,n-1$, together with the index sets $\mathcal R_{\mathrm{equal}}:=\{r:\tilde\pi_r=\tilde\pi_{r+1}\}$ and $\mathcal R_{\mathrm{strict}}:=\{r:\tilde\pi_r>\tilde\pi_{r+1}\}$. If $\mathcal R_{\mathrm{strict}}\neq\emptyset$, we set $\varepsilon:=\min\!\bigl(10^{-9},\,\min_{r\in\mathcal R_{\mathrm{strict}}}\dot g_r,\,1-\max_{r\in\mathcal R_{\mathrm{strict}}}\dot g_r\bigr)$; otherwise, we set $\varepsilon=0$. We then take $(\underline\delta_r,\overline\delta_r)=(\varepsilon,1-\varepsilon)$ for $r\in\mathcal R_{\mathrm{strict}}$ and $(\underline\delta_r,\overline\delta_r)=(0,0)$ for $r\in\mathcal R_{\mathrm{equal}}$. This yields a wide admissible set while ensuring that $\dot p(\pi)\in\mathcal F^{\mathrm{box}}(\pi)$.

We train two models, Model A and Model B, on the same training items of the form $(x,\pi)$. 
They use the same affine score map $x\mapsto Wx+b$, but with different parameters:
\[
s^{\mathrm{A}}(x)=W_{\mathrm{A}}x+b_{\mathrm{A}}\in\mathbb{R}^n,
\qquad
s^{\mathrm{B}}(x)=W_{\mathrm{B}}x+b_{\mathrm{B}}\in\mathbb{R}^n,
\]
where $(W_{\mathrm{A}},b_{\mathrm{A}})$ and $(W_{\mathrm{B}},b_{\mathrm{B}})$ belong to
$\mathbb{R}^{n\times d}\times\mathbb{R}^n$.

Each model outputs a probability vector $q$ through a softmax transformation that enforces that $q$ is a strictly
positive probability distribution. 
For $M\in\{\mathrm{A},\mathrm{B}\}$, define
$m_M(x):=\max_{1\le j\le n} s^{M}_j(x)$ and set
\[
q^{M}_c(x)
=
\frac{\exp\!\bigl(s^{M}_c(x)-m_M(x)\bigr)}
{\sum_{j=1}^n \exp\!\bigl(s^{M}_j(x)-m_M(x)\bigr)},
\qquad c=1,\dots,n.
\]
Then $q^{M}(x)\in\Delta_n\cap\mathbb{R}^n_{++}$, and the Kullback--Leibler divergence
$D_{\mathrm{KL}}(p\|q^{M}(x))$ is well-defined for every $p\in\Delta_n$.\\
\smallskip
The two models use different learning objectives:\\
\noindent\emph{(i) Projection target.}
Model A uses as target the Kullback--Leibler projection of its current prediction
$q^{\mathrm{A}}(x)$ (computed from the training instance $x$) onto the admissible set
$\mathcal F^{\mathrm{box}}(\pi)$:
\[
p^\star(x,\pi)
:=
\arg\min_{p\in\mathcal F^{\mathrm{box}}(\pi)} D_{\mathrm{KL}}\!\bigl(p\,\|\,q^{\mathrm{A}}(x)\bigr).
\]
We compute $p^\star(x,\pi)$ by running Algorithm~\ref{ex:algo:dykstra} on $\mathcal F^{\mathrm{box}}(\pi)$.
The projections are those of Proposition~\ref{prop:projCbv1} and Proposition~\ref{prop:projCbv2}, and we apply the numerically stable computations described in Subsection~\ref{subsec:numerical:computingzt}.
Model A is trained by minimizing
\[
\ell_{\mathrm{proj}}^{\mathrm{A}}(x,\pi)
:=
D_{\mathrm{KL}}\!\bigl(p^\star(x,\pi)\,\|\,q^{\mathrm{A}}(x)\bigr).
\]
During training, $p^\star(x,\pi)$ is recomputed from the current prediction $q^{\mathrm{A}}(x)$ at each optimization step and used as a soft target.

\smallskip
\noindent\emph{(ii) Fixed target.}
Model B is trained with the probability vector $\dot p(\pi)\in\Delta_n$ defined in
Section~\ref{sec:prob-constraints} (see \eqref{eq:ppoint}), obtained from $\pi$ by the antipignistic
reverse mapping. The per-item loss is
\[
\ell_{\mathrm{fix}}^{\mathrm{B}}(x,\pi)
:=
D_{\mathrm{KL}}\!\bigl(\dot p(\pi)\,\|\,q^{\mathrm{B}}(x)\bigr).
\]

Both objectives use the same possibilistic annotation $\pi$.
Model A uses $\pi$ through the set $\mathcal F^{\mathrm{box}}(\pi)$ by defining its target as the projected vector $p^\star(x,\pi)\in\mathcal F^{\mathrm{box}}(\pi)$ associated with the current
prediction $q^{\mathrm{A}}(x)$.
Model B, in contrast, uses $\pi$ only through the fixed probability vector $\dot p(\pi)$ and does not impose any additional constraint during training.
Accordingly, $p^\star(x,\pi)$ depends on both $x$ and $\pi$ through $q^{\mathrm{A}}(x)$ and the projection onto $\mathcal F^{\mathrm{box}}(\pi)$, whereas $\dot p(\pi)$ depends only on $\pi$.
Recall that $\dot p(\pi)\in\mathcal F^{\mathrm{box}}(\pi)$ by construction.

Our goal is to assess whether the projection-based objective (Model A) yields better classification accuracy than the fixed-target objective (Model B) under the same supervision.

\subsubsection{Synthetic data generation}
\label{subsubsec:synthetic-data}

Each run produces two independent datasets: a training set $\mathcal D_{\mathrm{tr}}$ and a test set $\mathcal D_{\mathrm{te}}$.
A dataset is a finite set of samples of the form $(x,c,\pi)$, where $x\in\mathbb R^d$ is the input vector,
$c\in\mathcal Y:=\{1,\dots,n\}$ is the class label, and $\pi=\pi(x,c)\in(0,1]^n$ is a strictly positive normalized
possibility distribution on $\mathcal Y$.
When there is no ambiguity, we write $\pi$ instead of $\pi(x,c)$.
The normalization and strict-positivity conditions are
\begin{equation}\label{eq:synth-pi-normalized}
\max_{1\le j\le n}\pi_j=1
\qquad\text{and}\qquad
\pi_j\ge \rho_\pi\ \text{ for } j=1,\dots,n,
\end{equation}
for a fixed lower bound $\rho_\pi>0$.
For each labeled pair $(x,c)$, we set $\pi_c=1$ and we ensure $\pi_j<1$ for all $j\neq c$ (see below).  In what follows, $(x,c,\pi)$ denotes the full synthetic record, while the models are trained using $(x,\pi)$, only; especially, the label $c$ is used solely to compute predictive metrics.

A configuration is defined by the triple $(d,N_{\mathrm{tr}},\alpha)$, where $d$ is the dimension of the input vectors,
$N_{\mathrm{tr}}$ is the training set size, and $\alpha$ controls, for each sample with label $c$, the possibility values
assigned to the classes in $\mathcal Y\setminus\{c\}$.
Across configurations, the test set size is fixed to $N_{\mathrm{te}}=3000$, and all other parameters are fixed as reported
in Table~\ref{tab:synthetic-params}.

\begin{table}[htp!]
\centering
\begin{tabular}{@{}l p{0.85\linewidth}@{}}
\toprule
Symbol & Meaning and values used in the experiments \\
\midrule

\multicolumn{2}{@{}l}{\emph{Dataset and dimensions}}\\
\addlinespace[0.3em]
$n$ & \makecell[l]{Number of classes; fixed to $n=20$.}\\
$\mathcal Y$ & \makecell[l]{Class set $\{1,\dots,n\}$.}\\
$d$ & \makecell[l]{Dimension of the input vector $x$; $d\in\{30,80,150\}$.}\\
$N_{\mathrm{tr}}$ & \makecell[l]{Training set size; $N_{\mathrm{tr}}\in\{200,500,1000\}$.}\\
$N_{\mathrm{te}}$ & \makecell[l]{Test set size; fixed to $N_{\mathrm{te}}=3000$.}\\

\addlinespace[0.6em]
\multicolumn{2}{@{}l}{\emph{Prototypes and input vectors}}\\
\addlinespace[0.3em]
$\mu_c$ & \makecell[l]{Prototype vector associated with class $c$ (used to generate samples of class $c$); $\mu_c\in\mathbb R^d$.}\\
$\beta$ & \makecell[l]{Prototype scale used in $\mu_c=\beta Z_c$. Larger $\beta$ typically makes prototypes farther apart.\\
Values: $\beta=1.5$ if $d=30$, $\beta=0.9$ if $d=80$, $\beta=0.6$ if $d=150$.}\\
$s$ & \makecell[l]{Noise level in the generation of $x$: for a sample with label $c$, the input vector $x$ is generated\\
by perturbing $\mu_c$, and $s$ scales this perturbation. Fixed to $s=2.0$.}\\

\addlinespace[0.6em]
\multicolumn{2}{@{}l}{\emph{Possibilistic annotation}}\\
\addlinespace[0.3em]
$\rho_\pi$ & \makecell[l]{Lower bound for possibility values: for every sample $(x,c,\pi)$ and every class $j$, $\pi_j\ge \rho_\pi$.\\
Fixed to $\rho_\pi=10^{-6}$.}\\
$\alpha$ & \makecell[l]{Base \emph{plausibility level} controlling the possibility values assigned to the classes in $\mathcal Y\setminus\{c\}$.\\
Values: $\alpha\in\{0.4,0.6,0.8,0.95\}$.}\\
$s_\alpha$ & \makecell[l]{Noise level used to perturb the base plausibility level $\alpha$ when defining a per-sample level\\ $\alpha(x)$ in \eqref{eq:synth-alpha}. Fixed to $s_\alpha=0.15$.}\\

$\delta_\pi$ & \makecell[l]{Step size for the possibility values assigned to $\mathcal Y\setminus\{c\}$ for a sample $(x,c,\pi)$: we rank\\
the classes in $\mathcal Y\setminus\{c\}$ from closest to farthest from $x$ using the squared Euclidean distance\\
$d_j(x)=\|x-\mu_j\|_2^2$ to their prototypes $(\mu_j)_{j\in\mathcal Y}$.
When moving from one ranked class to\\ the next, the assigned possibility value is decreased by steps of size
$\delta_\pi$ (until it reaches\\ the floor $\rho_\pi$). Fixed to $\delta_\pi=0.01$.\\ 
}\\
\addlinespace[0.6em]
\multicolumn{2}{@{}l}{\emph{Tie breaking and ranking}}\\
\addlinespace[0.3em]
$\triangleleft$ & \makecell[l]{Natural index order on $\mathcal Y$: $1\triangleleft 2\triangleleft\cdots\triangleleft n$.}\\
$\preceq_x$ & \makecell[l]{Order on $\mathcal Y$ defined from distances to $x$: $i\preceq_x j$ if the prototype $\mu_i$ of class $i$ is closer to $x$\\
than the prototype $\mu_j$ of class $j$, and ties are broken by $\triangleleft$.}\\

\bottomrule
\end{tabular}
\caption{Notation and values for the synthetic data generation.}
\label{tab:synthetic-params}
\end{table}

The data generation proceeds in three steps. We first generate one prototype vector $\mu_c$ for each class $c\in\mathcal Y$.
We then generate labeled input vectors $(x,c)$ for the training and test sets.
Finally, for each labeled pair $(x,c)$, we generate a possibility vector $\pi$ on $\mathcal Y$, which yields samples $(x,c,\pi)$.

For each class $c\in\mathcal Y$, we draw
\begin{equation}\label{eq:synth-Z}
Z_c \sim \mathcal N(0,I_d)
\end{equation}
and define the prototype
\begin{equation}\label{eq:synth-proto}
\mu_c := \beta Z_c \in \mathbb R^d.
\end{equation}
Larger values of $\beta$ typically increase the distances between distinct prototypes.

We generate the training set and the test set in the same way, and we sample the two sets independently.
Each item is obtained by first drawing a class label uniformly on $\mathcal Y$, and then generating an input vector by perturbing
the prototype of that class.

For the training set, for each $i=1,\dots,N_{\mathrm{tr}}$, we draw
\begin{equation}\label{eq:synth-train-label}
c^{\mathrm{tr}}_i \sim \mathrm{Unif}(\mathcal Y),
\qquad
\nu^{\mathrm{tr}}_i \sim \mathcal N(0,I_d),
\end{equation}
and set
\begin{equation}\label{eq:synth-train-x}
x^{\mathrm{tr}}_i := \mu_{c^{\mathrm{tr}}_i}+s\,\nu^{\mathrm{tr}}_i.
\end{equation}
The parameter $s\ge 0$ controls the size of the perturbation around the class prototype.
For the test set, for each $i=1,\dots,N_{\mathrm{te}}$, we draw
\begin{equation}\label{eq:synth-test-label}
c^{\mathrm{te}}_i \sim \mathrm{Unif}(\mathcal Y),
\qquad
\nu^{\mathrm{te}}_i \sim \mathcal N(0,I_d),
\end{equation}
and set
\begin{equation}\label{eq:synth-test-x}
x^{\mathrm{te}}_i := \mu_{c^{\mathrm{te}}_i}+s\,\nu^{\mathrm{te}}_i.
\end{equation}

For each labeled pair $(x,c)$, we construct a possibility vector $\pi=\pi(x,c)$ on $\mathcal Y$ with the following goals:
(i) the label is fully possible, $\pi_c=1$;
(ii) for every $j\neq c$ we have $0<\pi_j<1$ and $\pi_j\ge \rho_\pi$;
(iii) among the classes in $\mathcal Y\setminus\{c\}$, classes whose prototypes are closer to $x$ receive larger values.

This construction is inspired by a common two-level possibilistic encoding, where one class is assigned possibility $1$ and the remaining classes share a plausibility level $\alpha$ \cite{lienen2021credal,lienen2023conformal}.
In our synthetic setting, $\alpha$ plays the role of a base plausibility level, but we go beyond the two-level form by allowing the possibility values of the non-true classes to vary across classes according to their distance-based ranking around $x$.

We proceed in three substeps: ranking, choice of a base level, and assignment.

For an input vector $x\in\mathbb R^d$ and a class $j\in\mathcal Y$, let us define the squared distance to the prototype $\mu_j$ by
\begin{equation}\label{eq:synth-distance}
d_j(x):=\|x-\mu_j\|_2^2.
\end{equation}
We rank classes by increasing values of $d_j(x)$, and if two classes have the same value we break ties using
the natural index order $1\triangleleft 2\triangleleft\cdots\triangleleft n$.
This defines the relation $\preceq_x$ on $\mathcal Y$ by
\begin{equation}\label{eq:synth-order}
i \preceq_x j
\quad \Longleftrightarrow\quad
\bigl(d_i(x) < d_j(x)\bigr)\ \ \text{or}\ \ \bigl(d_i(x)=d_j(x)\ \text{and}\ i\triangleleft j\bigr),
\qquad i,j\in\mathcal Y.
\end{equation}
For a labeled pair $(x,c)$, list the classes in $\mathcal Y\setminus\{c\}$ as $j_{(1)},\dots,j_{(n-1)}$ such that
\begin{equation}\label{eq:synth-ranking}
j_{(1)} \preceq_x j_{(2)} \preceq_x \cdots \preceq_x j_{(n-1)}.
\end{equation}

We set the label to be fully possible:
\begin{equation}\label{eq:synth-pi-label}
\pi_c := 1.
\end{equation}

For the remaining classes $\mathcal Y\setminus\{c\}$, we first define a sample-dependent level $\alpha(x)$.
We start from the base parameter $\alpha$ and add a random perturbation of size $s_\alpha$.
Let $\eta$ be a random scalar drawn independently for each labeled pair $(x,c)$, and define
\begin{equation}\label{eq:synth-eta}
\eta \sim \mathcal N(0,1).
\end{equation}
We then set
\begin{equation}\label{eq:synth-alpha}
\alpha(x) :=
\min\!\Bigl(1-\rho_\pi,\ \max\!\bigl(0,\ \alpha+s_\alpha\,\eta\bigr)\Bigr).
\end{equation}

Next, we build a decreasing sequence along the ranking using the step size $\delta_\pi$.
For each rank $r\in\{1,\dots,n-1\}$, we define
\begin{equation}\label{eq:synth-alpha-r}
\alpha_r(x) := \max\!\bigl(0,\ \alpha(x)-(r-1)\delta_\pi\bigr).
\end{equation}
For every $r\in\{1,\dots,n-2\}$ we have $\alpha_r(x)\ge \alpha_{r+1}(x)$.
Moreover, if $\alpha_r(x)>0$ and $\delta_\pi>0$, then $\alpha_r(x)>\alpha_{r+1}(x)$.
Thus, when moving from rank $r$ to $r+1$, the quantity $\alpha_r(x)$ decreases by $\delta_\pi$ until it reaches $0$.

Finally, we assign possibility values to the ranked classes by
\begin{equation}\label{eq:synth-pi}
\pi_{j_{(r)}} := \min\!\bigl(1-\rho_\pi,\ \rho_\pi+\alpha_r(x)\bigr),
\qquad r=1,\dots,n-1.
\end{equation}
With \eqref{eq:synth-pi}, smaller ranks (classes closer to $x$) receive larger values, and the value is non-increasing along the ranking, and decreases by steps of size $\delta_\pi$ until it reaches the floor $\rho_\pi$ (since $\alpha_r(x)$ reaches $0$),  and it is kept strictly below $1$ by the cap
$1-\rho_\pi$.

For any $j\neq c$, \eqref{eq:synth-pi} gives $\pi_j\ge \rho_\pi$ because $\alpha_r(x)\ge 0$, and $\pi_j\le 1-\rho_\pi$ because of the
explicit cap $\min(1-\rho_\pi,\cdot)$.
Therefore, for all $j\neq c$,
\begin{equation}\label{eq:synth-pi-bounds}
\rho_\pi \le \pi_j \le 1-\rho_\pi,
\end{equation}
which implies $0<\pi_j<1$.

As an illustration, consider $n=5$ and a labeled pair $(x,c)$.
After ranking $\mathcal Y\setminus\{c\}$ as $j_{(1)}\preceq_x j_{(2)}\preceq_x j_{(3)}\preceq_x j_{(4)}$, the construction yields
\[
\pi_c=1,\qquad
\pi_{j_{(1)}}=\min\!\bigl(1-\rho_\pi,\rho_\pi+\alpha_1(x)\bigr),\qquad
\pi_{j_{(2)}}=\min\!\bigl(1-\rho_\pi,\rho_\pi+\alpha_2(x)\bigr),
\]
\[
\pi_{j_{(3)}}=\min\!\bigl(1-\rho_\pi,\rho_\pi+\alpha_3(x)\bigr),\qquad
\pi_{j_{(4)}}=\min\!\bigl(1-\rho_\pi,\rho_\pi+\alpha_4(x)\bigr),
\]
with
\[
\alpha_1(x)=\alpha(x),
\qquad
\alpha_{r+1}(x)=\max\bigl(0,\alpha_r(x)-\delta_\pi\bigr)\ \ \text{for } r=1,2,3.
\]
In particular, for $r=1,2,3$ we have $\alpha_r(x)\ge \alpha_{r+1}(x)$, and if $\alpha_r(x)>0$ and $\delta_\pi>0$
then $\alpha_r(x)>\alpha_{r+1}(x)$.

\subsubsection{Training and evaluation protocol}
\label{subsubsec:topk-eval}

Table~\ref{tab:synthetic-params} summarizes the parameter values used for the synthetic data generation.
Experiments vary $(d,N_{\mathrm{tr}},\alpha)$ over $d\in\{30,80,150\}$, $N_{\mathrm{tr}}\in\{200,500,1000\}$, and $\alpha\in\{0.4,0.6,0.8,0.95\}$.
Recall that the test set size is fixed to $N_{\mathrm{te}}=3000$.

For each configuration $(d,N_{\mathrm{tr}},\alpha)$ we perform $10$ independent runs; in each run we regenerate $(\mathcal D_{\mathrm{tr}},\mathcal D_{\mathrm{te}})$ and reinitialize both models.
For a fixed pair $(d,N_{\mathrm{tr}})$ and a fixed run index, we reuse the same random seeds for generating the class prototypes $(\mu_c)_{c=1}^n$ and the labeled inputs $(x_i,c_i)$ across all values of $\alpha$.
Consequently, within a run index, the labeled pairs $(x_i,c_i)$ are identical for all $\alpha$, and only the possibilistic annotations $\pi_i$ vary with $\alpha$.

Training uses Adam \cite{kingma2014adam} with weight decay $10^{-4}$.
We use batch size $64$ for $N_{\mathrm{tr}}\le 200$ and $128$ otherwise.
We train for $80$ epochs when $N_{\mathrm{tr}}\le 200$ and for $60$ epochs otherwise.

Learning rates are selected by grid search using an independent validation set of size $N_{\mathrm{val}}=N_{\mathrm{tr}}$, generated by the same procedure.
We use the grid
\[
\{1,2,\dots,9\}\times 10^{k},
\qquad k\in\{-4,-3,-2\}.
\]
For each candidate learning rate, we train the corresponding model and evaluate the mean validation accuracy.
The learning-rate searches for Models A and B are performed separately.
For each model, we select the learning rate that maximizes the mean validation accuracy averaged over three validation seeds; for each seed, we regenerate the prototypes and the corresponding train/validation datasets and use a fresh parameter initialization.
None of the selected learning rates lies on the boundary of the search grid.

For Model A, the target $p^\star(x,\pi)$ is computed by Dykstra's algorithm with tolerance $\tau=10^{-8}$ and a maximum of $K_{\max}=2000$ cycles.
If the stopping criterion $\mathrm V(p)\le\tau$ is not met within $K_{\max}$ cycles, we use the last iterate returned after $K_{\max}$ cycles as the target.
The admissible set $\mathcal F^{\mathrm{box}}(\pi)$ is defined from $\pi$, see  Subsection~\ref{subsec:topk-learning}.

Predictive performance is evaluated on $\mathcal D_{\mathrm{te}}$ using the predicted probabilities $q(x)$ produced by the trained model. We report top-$1$ accuracy.

\subsubsection{Results}

\begin{table*}[htp!]
\centering
\scriptsize
\setlength{\tabcolsep}{3pt}
\renewcommand{\arraystretch}{1.15}

\resizebox{\textwidth}{!}{%
\begin{tabular}{lll r rr rr rr}
\toprule
$d$ & $\beta$ & $\alpha$ & $N_{\mathrm{tr}}$ & $\mathrm{lr}_A$ & $\mathrm{lr}_B$ & $\mathrm{Acc}_A^{\mathrm{tr}}$ & $\mathrm{Acc}_B^{\mathrm{tr}}$ & $\mathrm{Acc}_A^{\mathrm{te}}$ & $\mathrm{Acc}_B^{\mathrm{te}}$ \\
\midrule
30 & 1.5 & 0.4 & 200 & 0.01 & 0.0008 & $0.9960\pm0.0039$ & $0.9865\pm0.0094$ & $\bm{0.9118\pm0.0184}$ & $0.8689\pm0.0173$ \\
30 & 1.5 & 0.4 & 500 & 0.03 & 0.002 & $0.9944\pm0.0034$ & $0.9838\pm0.0096$ & $\bm{0.9258\pm0.0161}$ & $0.9055\pm0.0167$ \\
30 & 1.5 & 0.4 & 1000 & 0.006 & 0.001 & $0.9890\pm0.0075$ & $0.9711\pm0.0115$ & $0.9277\pm0.0150$ & $\bm{0.9287\pm0.0136}$ \\
30 & 1.5 & 0.6 & 200 & 0.02 & 0.0006 & $0.9900\pm0.0094$ & $0.9760\pm0.0165$ & $\bm{0.9286\pm0.0155}$ & $0.8629\pm0.0172$ \\
30 & 1.5 & 0.6 & 500 & 0.03 & 0.001 & $0.9874\pm0.0065$ & $0.9704\pm0.0123$ & $\bm{0.9353\pm0.0127}$ & $0.9096\pm0.0159$ \\
30 & 1.5 & 0.6 & 1000 & 0.01 & 0.0006 & $0.9772\pm0.0110$ & $0.9623\pm0.0120$ & $\bm{0.9384\pm0.0143}$ & $0.9291\pm0.0136$ \\
30 & 1.5 & 0.8 & 200 & 0.006 & 0.0004 & $0.9695\pm0.0148$ & $0.9380\pm0.0262$ & $\bm{0.9358\pm0.0136}$ & $0.8414\pm0.0201$ \\
30 & 1.5 & 0.8 & 500 & 0.007 & 0.0006 & $0.9612\pm0.0133$ & $0.9514\pm0.0133$ & $\bm{0.9425\pm0.0109}$ & $0.9011\pm0.0173$ \\
30 & 1.5 & 0.8 & 1000 & 0.008 & 0.0007 & $0.9661\pm0.0135$ & $0.9522\pm0.0134$ & $\bm{0.9467\pm0.0107}$ & $0.9241\pm0.0139$ \\
30 & 1.5 & 0.95 & 200 & 0.004 & 0.0003 & $0.9590\pm0.0168$ & $0.8845\pm0.0539$ & $\bm{0.9392\pm0.0133}$ & $0.8026\pm0.0342$ \\
30 & 1.5 & 0.95 & 500 & 0.007 & 0.0003 & $0.9568\pm0.0159$ & $0.9242\pm0.0191$ & $\bm{0.9464\pm0.0112}$ & $0.8864\pm0.0185$ \\
30 & 1.5 & 0.95 & 1000 & 0.002 & 0.0003 & $0.9551\pm0.0173$ & $0.9365\pm0.0163$ & $\bm{0.9485\pm0.0105}$ & $0.9159\pm0.0152$ \\
80 & 0.9 & 0.4 & 200 & 0.005 & 0.0005 & $0.9985\pm0.0024$ & $0.9990\pm0.0021$ & $\bm{0.9102\pm0.0105}$ & $0.8136\pm0.0196$ \\
80 & 0.9 & 0.4 & 500 & 0.006 & 0.0008 & $0.9968\pm0.0030$ & $0.9980\pm0.0028$ & $\bm{0.9407\pm0.0076}$ & $0.8932\pm0.0133$ \\
80 & 0.9 & 0.4 & 1000 & 0.007 & 0.0003 & $0.9969\pm0.0014$ & $0.9844\pm0.0067$ & $\bm{0.9497\pm0.0061}$ & $0.9314\pm0.0094$ \\
80 & 0.9 & 0.6 & 200 & 0.007 & 0.0004 & $0.9965\pm0.0041$ & $0.9965\pm0.0041$ & $\bm{0.9341\pm0.0085}$ & $0.8084\pm0.0205$ \\
80 & 0.9 & 0.6 & 500 & 0.009 & 0.0005 & $0.9884\pm0.0052$ & $0.9926\pm0.0065$ & $\bm{0.9531\pm0.0069}$ & $0.8939\pm0.0138$ \\
80 & 0.9 & 0.6 & 1000 & 0.006 & 0.0002 & $0.9885\pm0.0049$ & $0.9772\pm0.0082$ & $\bm{0.9571\pm0.0057}$ & $0.9303\pm0.0103$ \\
80 & 0.9 & 0.8 & 200 & 0.006 & 0.0003 & $0.9895\pm0.0050$ & $0.9700\pm0.0127$ & $\bm{0.9470\pm0.0092}$ & $0.7772\pm0.0211$ \\
80 & 0.9 & 0.8 & 500 & 0.007 & 0.0003 & $0.9814\pm0.0075$ & $0.9720\pm0.0082$ & $\bm{0.9584\pm0.0061}$ & $0.8836\pm0.0138$ \\
80 & 0.9 & 0.8 & 1000 & 0.007 & 0.0002 & $0.9853\pm0.0032$ & $0.9678\pm0.0060$ & $\bm{0.9603\pm0.0058}$ & $0.9236\pm0.0088$ \\
80 & 0.9 & 0.95 & 200 & 0.008 & 0.0002 & $0.9885\pm0.0088$ & $0.8885\pm0.0156$ & $\bm{0.9504\pm0.0078}$ & $0.7310\pm0.0137$ \\
80 & 0.9 & 0.95 & 500 & 0.009 & 0.0002 & $0.9832\pm0.0046$ & $0.9306\pm0.0143$ & $\bm{0.9599\pm0.0062}$ & $0.8592\pm0.0176$ \\
80 & 0.9 & 0.95 & 1000 & 0.008 & 0.0002 & $0.9833\pm0.0029$ & $0.9433\pm0.0100$ & $\bm{0.9592\pm0.0055}$ & $0.9097\pm0.0107$ \\
150 & 0.6 & 0.4 & 200 & 0.003 & 0.001 & $0.9995\pm0.0016$ & $1.0000\pm0.0000$ & $\bm{0.7931\pm0.0161}$ & $0.6759\pm0.0189$ \\
150 & 0.6 & 0.4 & 500 & 0.005 & 0.0004 & $0.9978\pm0.0024$ & $0.9972\pm0.0017$ & $\bm{0.8914\pm0.0063}$ & $0.7978\pm0.0090$ \\
150 & 0.6 & 0.4 & 1000 & 0.006 & 0.0002 & $0.9974\pm0.0020$ & $0.9834\pm0.0044$ & $\bm{0.9188\pm0.0065}$ & $0.8663\pm0.0071$ \\
150 & 0.6 & 0.6 & 200 & 0.004 & 0.0008 & $0.9970\pm0.0035$ & $0.9975\pm0.0063$ & $\bm{0.8408\pm0.0108}$ & $0.6515\pm0.0179$ \\
150 & 0.6 & 0.6 & 500 & 0.004 & 0.0005 & $0.9900\pm0.0045$ & $0.9966\pm0.0019$ & $\bm{0.9187\pm0.0048}$ & $0.7862\pm0.0095$ \\
150 & 0.6 & 0.6 & 1000 & 0.009 & 0.0002 & $0.9970\pm0.0013$ & $0.9823\pm0.0050$ & $\bm{0.9293\pm0.0060}$ & $0.8626\pm0.0079$ \\
150 & 0.6 & 0.8 & 200 & 0.003 & 0.0003 & $0.9925\pm0.0054$ & $0.9535\pm0.0153$ & $\bm{0.8545\pm0.0110}$ & $0.6042\pm0.0208$ \\
150 & 0.6 & 0.8 & 500 & 0.006 & 0.0003 & $0.9794\pm0.0077$ & $0.9648\pm0.0065$ & $\bm{0.9306\pm0.0043}$ & $0.7628\pm0.0086$ \\
150 & 0.6 & 0.8 & 1000 & 0.006 & 0.0002 & $0.9877\pm0.0046$ & $0.9595\pm0.0073$ & $\bm{0.9342\pm0.0056}$ & $0.8417\pm0.0080$ \\
150 & 0.6 & 0.95 & 200 & 0.004 & 0.0005 & $0.9960\pm0.0057$ & $0.8655\pm0.0244$ & $\bm{0.8697\pm0.0125}$ & $0.5343\pm0.0214$ \\
150 & 0.6 & 0.95 & 500 & 0.009 & 0.0002 & $0.9930\pm0.0030$ & $0.8884\pm0.0140$ & $\bm{0.9346\pm0.0050}$ & $0.7203\pm0.0119$ \\
150 & 0.6 & 0.95 & 1000 & 0.003 & 0.0002 & $0.9687\pm0.0037$ & $0.9080\pm0.0124$ & $\bm{0.9354\pm0.0052}$ & $0.8147\pm0.0123$ \\
\bottomrule
\end{tabular}%
}

\caption{Top-$1$ accuracies on the synthetic learning task, for Models A (projection target) and B (fixed target).
Each row corresponds to one configuration $(d,\beta,\alpha,N_{\mathrm{tr}})$, with all  parameters fixed as in Table~\ref{tab:synthetic-params}.
For each model, $\mathrm{lr}_A$ and $\mathrm{lr}_B$ are the learning rates selected by validation grid search.
$\mathrm{Acc}^{\mathrm{train}}$ and $\mathrm{Acc}^{\mathrm{test}}$ denote the training and test top-$1$ accuracies, reported as mean $\pm$ standard deviation over $10$ independent runs.
The best test accuracy between Models A and B is shown in bold.}
\label{tab:all_test_acc}
\end{table*}

Table~\ref{tab:all_test_acc} compares Models A and B on the synthetic learning task over all configurations $(d,\beta,\alpha,N_{\mathrm{tr}})$. Overall, Model A has the best test accuracy in $35$ of the $36$ settings. The only exception is $(d,\beta,\alpha,N_{\mathrm{tr}})=(30,1.5,0.4,1000)$, where the two models are almost tied and Model B is slightly higher ($0.9287$ versus $0.9277$). The main pattern is therefore a clear overall advantage for the projection-based objective, with a single near-tie.

The size of the improvement depends strongly on $\alpha$. When $\alpha=0.4$, the advantage of Model A over Model B is generally smaller than for larger values of $\alpha$, and for large training sets it can become almost negligible. By contrast, when $\alpha$ increases to $0.6$, $0.8$, and $0.95$, the gap between the two models becomes much larger. This pattern appears for all three values of $d$, and it is strongest in the hardest settings. For example, with $(d,\beta)=(80,0.9)$ and $N_{\mathrm{tr}}=200$, the test-accuracy gap increases from about $0.097$ at $\alpha=0.4$ to about $0.219$ at $\alpha=0.95$. With $(d,\beta)=(150,0.6)$ and $N_{\mathrm{tr}}=200$, it increases from about $0.117$ to about $0.335$.

The training-set size also has a clear effect. For a fixed $(d,\beta,\alpha)$, the gap between the two models usually becomes smaller as $N_{\mathrm{tr}}$ increases. For example, at $(d,\beta,\alpha)=(80,0.9,0.95)$, the difference in test accuracy is about $0.219$ for $N_{\mathrm{tr}}=200$, about $0.101$ for $N_{\mathrm{tr}}=500$, and about $0.050$ for $N_{\mathrm{tr}}=1000$. A similar reduction of the gap appears in many other rows. This suggests that the projection-based target is most helpful when training data are limited, while the two objectives become closer when more data are available.

The training accuracies help explain this behaviour. Both models usually reach very high training accuracy, often close to $1$. This means that the difference between the methods is not simply that one model can fit the training data and the other cannot. The main difference appears on the test set. In many settings, Model B reaches training accuracy similar to, and sometimes slightly higher than, that of Model A, while still having much lower test accuracy. For example, at $(d,\beta,\alpha,N_{\mathrm{tr}})=(150,0.6,0.6,500)$, Model B has slightly higher training accuracy ($0.9966$ versus $0.9900$) but much lower test accuracy ($0.7862$ versus $0.9187$). This suggests that the advantage of Model A comes mainly from better generalization.

Overall, the results support the following conclusion. When the possibilistic supervision is fairly specific and the training set is large, the fixed target $\dot p(\pi)$ can already work well, and the two objectives may give very similar results. However, when the supervision becomes more ambiguous (larger $\alpha$), and especially when the number of training examples is small, the projection-based target gives a clear and often large improvement in test accuracy. In this synthetic benchmark, updating the target by projecting the current prediction onto $\mathcal F^{\mathrm{box}}(\pi)$ is especially helpful in the most ambiguous and most data-limited settings.

We focus on top-$1$ accuracy as the primary evaluation metric in this study.
Assessing whether the same pattern holds for other criteria (e.g., negative log-likelihood \cite{murphy2012machine}, Brier score \cite{brier1950}, or measures of constraint satisfaction with respect to the constraint sets $C_i$ defining $\mathcal F^{\mathrm{box}}$) requires additional experiments and is left for future work.

\subsection{Learning from possibilistic supervision on ChaosNLI}
\label{subsec:chaosnli}

We now turn to a real natural language inference task based on the ChaosNLI dataset \cite{nie2020can}.
ChaosNLI (Collective HumAn OpinionS on Natural Language Inference) is a benchmark 
built from examples that were originally drawn from SNLI \cite{bowman2015large} (1,514 examples), MultiNLI \cite{williams2018broad} (1,599 examples), and $\alpha$NLI \cite{bhagavatula2019abductive} (1,532 examples), each annotated by a large number of crowd workers.
In this work, we use the SNLI-based and MultiNLI-based portions of ChaosNLI jointly, and restrict attention to the standard three-label natural language inference setting with class set
\[
\mathcal Y=\{\mathrm{entailment},\mathrm{neutral},\mathrm{contradiction}\},
\qquad n=3.
\]
We exclude only the $\alpha$NLI portion, since it does not follow this label structure.

For each item, the supervision takes the form of annotator vote counts over $\mathcal Y$, and the degree of ambiguity is induced by the observed disagreement between annotators.

This setting is therefore complementary to the synthetic study of Subsection~\ref{subsec:topk-learning}: instead of generating possibilistic annotations by construction, we derive them from empirical human judgments.

\subsubsection{Dataset and feature representation}
\label{subsubsec:chaosnli-data}

In our setting, each example consists of 
a pair of natural language sentences (a premise and a hypothesis), together with annotator vote counts over the NLI label set $\mathcal Y$. Each vote corresponds to the view of the annotator about the connection between the premise and the hypothesis: Is the hypothesis entailed or contradicted (or neither) by the premise?

ChaosNLI was released as a collection of multiply annotated examples, without a predefined training/validation/test split for supervised learning. We therefore define our own split protocol:\footnote{The splits used in our experiments are provided in the code repository.} using a fixed split seed, we construct in a  deterministic (reproducible) way a split with target fractions $80\%$ for training, $10\%$ for validation, and $10\%$ for test.
The realized split sizes are $2489$ training items, $310$ validation items, and $314$ test items. For each of the three data splits, we additionally extract two ambiguity-based subsets, denoted $\mathcal S_{\mathrm{amb}}$ and $\mathcal S_{\mathrm{easy}}$; their definitions are given below. On the test split, these subsets contain $79$ items in $\mathcal S_{\mathrm{amb},\mathrm{te}}$ and $75$ items in $\mathcal S_{\mathrm{easy},\mathrm{te}}$.

Each sentence pair is encoded into a fixed vector $x\in\mathbb R^d$ using the pretrained RoBERTa-base encoder \cite{liu2019roberta}; specifically, we apply masked mean pooling to the last hidden layer, which gives $d=768$.
The encoder is kept fixed throughout, and only the classification head is learned.
Thus, as in the synthetic experiment, the comparison focuses on the effect of the training target rather than on differences in feature learning.

\subsubsection{Possibilistic annotation derived from vote counts}
\label{subsubsec:chaosnli-pi}

For each data item, let
\[
v=(v_y)_{y\in\mathcal Y}\in\mathbb N^3
\]
denote the vote counts over the three classes, and let
\begin{equation}\label{eq:voteprop}
\bar v_y:=\frac{v_y}{\sum_{z\in\mathcal Y} v_z},
\qquad y\in\mathcal Y,
\end{equation}
be the corresponding vote proportions.
We denote by $c^\star\in\mathcal Y$ the majority label provided with the item in the ChaosNLI annotation.
When one label has a vote count strictly larger than the other two, this label is exactly
\[
c^\star=\arg\max_{y\in\mathcal Y} v_y.
\]
In tied cases, we keep the dataset-provided majority label.

From these vote counts we construct a possibility distribution
\[
\pi=(\pi_y)_{y\in\mathcal Y}\in(0,1]^3.
\]
Let
\[
v_{\max}:=\max_{y\in\mathcal Y} v_y,
\qquad
\rho_\pi:=10^{-6}.
\]
We define, for each $y\in\mathcal Y$,
\[
\pi_y :=
\begin{cases}
\max\!\bigl(\frac{v_y}{v_{\max}},\,\rho_\pi\bigr), & \text{if } v_y>0,\\[2pt]
\rho_\pi, & \text{if } v_y=0.
\end{cases}
\]
Since $v_{\max}>0$, we have $\pi_{c^\star}=1$ whenever $c^\star$ attains the maximal vote count. Hence $\pi$ is a strictly positive possibility distribution with $\max_{y\in\mathcal Y}\pi_y=1$, obtained by rescaling the empirical vote counts.

As in Section~\ref{sec:prob-constraints}, we then construct the admissible set
$
\mathcal F^{\mathrm{box}}(\pi)\subseteq \Delta_3
$ 
from this possibility distribution.

The gap parameters are chosen by the same procedure as in the previous experiments (see Subsection~\ref{subsec:experimentalalgoone}), with
$
\varepsilon_{\mathrm{cap}}=0.05
$
(instead of $10^{-9}$ in the synthetic setting; the larger value reflects the coarser granularity of a three-class vote-derived possibility distribution).

This construction guarantees that the antipignistic probability $\dot p(\pi)$ belongs to $\mathcal F^{\mathrm{box}}(\pi)$.\\
In addition, we retain the empirical vote proportions $\bar v$ as a third target, which provides a direct probabilistic baseline not mediated by the possibilistic transform.

\subsubsection{Training objectives}
\label{subsubsec:chaosnli-objectives}

We compare three training objectives, all based on the same input representation $x$ and the same underlying supervision signal derived from the vote counts.

\smallskip
\noindent\emph{(i) Projection target (Model A).}
Model A uses as target the Kullback--Leibler projection of its current prediction $q^{\mathrm A}(x)$ onto the admissible set $\mathcal F^{\mathrm{box}}(\pi)$:
\[
p^\star(x,\pi)
:=
\arg\min_{p\in\mathcal F^{\mathrm{box}}(\pi)}
D_{\mathrm{KL}}\!\bigl(p\,\|\,q^{\mathrm A}(x)\bigr).
\]
The corresponding loss is
\[
\ell_{\mathrm{proj}}^{\mathrm A}(x,\pi)
:=
D_{\mathrm{KL}}\!\bigl(p^\star(x,\pi)\,\|\,q^{\mathrm A}(x)\bigr).
\]
Thus the target depends on the current prediction through the projection step and is recomputed during training. The projection is always performed on the full label space $\mathcal Y$, including when some classes receive zero votes, and is computed by Dykstra's algorithm with tolerance $\tau=10^{-6}$ and a maximum of $K_{\max}=500$ cycles.

\smallskip
\noindent\emph{(ii) Antipignistic fixed target (Model B).}
Model B uses the antipignistic probability $\dot p(\pi)\in\Delta_3$ associated with the possibility distribution $\pi$.
Its loss is
\[
\ell_{\mathrm{fix}}^{\mathrm B}(x,\pi)
:=
D_{\mathrm{KL}}\!\bigl(\dot p(\pi)\,\|\,q^{\mathrm B}(x)\bigr).
\]

\smallskip
\noindent\emph{(iii) Vote-proportion target (Model C).}
Model C uses the normalized vote vector $\bar v\in\Delta_3$ defined in \eqref{eq:voteprop} directly:
\[
\ell_{\mathrm{vote}}^{\mathrm C}(x,\pi)
:=
D_{\mathrm{KL}}\!\bigl(\bar v\,\|\,q^{\mathrm C}(x)\bigr).
\]
This baseline does not use the possibilistic representation beyond the original votes themselves.

\smallskip
All three models use the same linear softmax head
\[
x\mapsto q(x)=\mathrm{softmax}(Wx+b),
\]
so differences in performance can be attributed to the target construction rather than to differences in architecture.

\subsubsection{Training, model selection, and ambiguity slices}
\label{subsubsec:chaosnli-eval}
Let
$
\mathcal D_{\mathrm{tr}},
\mathcal D_{\mathrm{val}},
\mathcal D_{\mathrm{te}}
$
be the fixed training, validation, and test splits defined above.

Each experiment is identified by a triplet
\[
(\text{train},\text{val},\text{test}),
\]
where the first element specifies the section used for training, the second the section used for model selection, and the third the section used for final evaluation.

For training, we consider three sections:
\[
\texttt{train\_full}=\mathcal D_{\mathrm{tr}},\qquad
\texttt{train\_S\_amb}=\mathcal S_{\mathrm{amb},\mathrm{tr}},\qquad
\texttt{train\_S\_easy}=\mathcal S_{\mathrm{easy},\mathrm{tr}}.
\]
That is, training is carried out either on the full training split, on its ambiguity-focused subset, or on its easy subset. The subsets $\mathcal S_{\mathrm{amb},\mathrm{tr}}$ and $\mathcal S_{\mathrm{easy},\mathrm{tr}}$ are defined below from the vote-proportion distribution. Apart from the choice of training section, the training procedure is identical across experiments.

Training uses Adam \cite{kingma2014adam} with weight decay $10^{-4}$, batch size $256$, and $100$ epochs. The learning rate follows a cosine annealing schedule and reaches $1\%$ of its initial value at the end of training. After each epoch, we measure accuracy on the chosen validation section and keep the checkpoint with the highest validation accuracy.

Learning rates are selected separately for Models A, B, and C using the grid
\[
\{1,2,\dots,9\}\times 10^k,
\qquad k\in\{-4,-3,-2,-1\}.
\]
For each candidate learning rate, we train the model with three initialization seeds and average the validation accuracies on the selected validation section. We then retain the learning rate with the best mean validation accuracy. Each hyperparameter-search trial uses the same training procedure as the final runs: cosine annealing schedule, 
100 epochs, and best-epoch checkpoint selection on the chosen validation section.

For validation, we consider three sections:
\[
\texttt{val\_full}=\mathcal D_{\mathrm{val}},\qquad
\texttt{val\_S\_amb}=\mathcal S_{\mathrm{amb},\mathrm{val}},\qquad
\texttt{val\_S\_easy}=\mathcal S_{\mathrm{easy},\mathrm{val}}.
\]
That is, model selection is performed either on the full validation split, on its ambiguity-focused subset, or on its easy subset. The subsets $\mathcal S_{\mathrm{amb},\mathrm{val}}$ and $\mathcal S_{\mathrm{easy},\mathrm{val}}$ are defined below using the same fixed rule as for training and test. Each experimental setting is therefore specified by a training section and a validation section, before final evaluation on a chosen test section.

Hence each training-and-selection setting is first defined by a pair \((\text{train},\text{val})\). This includes aligned pairs such as \((\texttt{train\_full},\texttt{val\_full})\), \((\texttt{train\_S\_amb},\texttt{val\_S\_amb})\), and \((\texttt{train\_S\_easy},\texttt{val\_S\_easy})\), as well as mixed pairs such as \((\texttt{train\_S\_amb},\texttt{val\_full})\) and \((\texttt{train\_full},\texttt{val\_S\_easy})\).

For final evaluation, we report results on three test sections:
\[
\texttt{test\_full}=\mathcal D_{\mathrm{te}},\qquad
\texttt{test\_S\_amb}=\mathcal S_{\mathrm{amb},\mathrm{te}},\qquad
\texttt{test\_S\_easy}=\mathcal S_{\mathrm{easy},\mathrm{te}},
\]
where \(\mathcal S_{\mathrm{amb},\mathrm{te}}\) and \(\mathcal S_{\mathrm{easy},\mathrm{te}}\) are defined below.

The ambiguity-based subsets are defined from the vote-proportion distribution $\bar v$, see \eqref{eq:voteprop}. For each item $i$, we compute the largest vote proportion
\[
p_{\max}^{(i)}:=\max_{y\in\mathcal Y}\bar v_y^{(i)},
\]
and the normalized entropy
\[
H_{\mathrm{norm}}^{(i)}
:=
-\frac{\sum_{y\in\mathcal Y}\bar v_y^{(i)}\log \bar v_y^{(i)}}{\log 3},
\]
with the convention $0\log 0:=0$.

We first form the subset of training items that have a unique majority-vote label:
\[
\mathcal D_{\mathrm{tr}}^{\mathrm{uniq}}
:=
\Bigl\{
i\in\mathcal D_{\mathrm{tr}}
:\exists\,y\in\mathcal Y \text{ such that } v_y^{(i)}>v_z^{(i)} \text{ for all } z\neq y
\Bigr\}.
\]
In our split, this subset contains $2466$ items. On this subset, we compute the $30$th and $70$th percentiles of $p_{\max}^{(i)}$ and $H_{\mathrm{norm}}^{(i)}$. This yields the fixed thresholds
$T_{\mathrm{low\text{-}peak}}=0.6$,
$T_{\mathrm{high\text{-}peak}}=0.8$,
$T_{\mathrm{low\text{-}H}}=0.502902$,
and
$T_{\mathrm{high\text{-}H}}=0.703581$.
These thresholds are computed once from $\mathcal D_{\mathrm{tr}}^{\mathrm{uniq}}$ and then used unchanged for all training, validation, and test splits.

For each split $\mathrm{sp}\in\{\mathrm{tr},\mathrm{val},\mathrm{te}\}$, we define the ambiguous subset
\[
\mathcal S_{\mathrm{amb},\mathrm{sp}}
:=
\Bigl\{
i\in\mathcal D_{\mathrm{sp}}
:\text{$i$ has a unique majority-vote label, }
p_{\max}^{(i)}\le T_{\mathrm{low\text{-}peak}},
\text{ and }
H_{\mathrm{norm}}^{(i)}\ge T_{\mathrm{high\text{-}H}}
\Bigr\},
\]
and the easy subset
\[
\mathcal S_{\mathrm{easy},\mathrm{sp}}
:=
\Bigl\{
i\in\mathcal D_{\mathrm{sp}}
:\text{$i$ has a unique majority-vote label, }
p_{\max}^{(i)}\ge T_{\mathrm{high\text{-}peak}},
\text{ and }
H_{\mathrm{norm}}^{(i)}\le T_{\mathrm{low\text{-}H}}
\Bigr\}.
\]

The corresponding section sizes are as follows: for training, \texttt{train\_full} ($n=2489$), \texttt{train\_S\_amb} ($n=511$), and \texttt{train\_S\_easy} ($n=702$); for validation, \texttt{val\_full} ($n=310$), \texttt{val\_S\_amb} ($n=68$), and \texttt{val\_S\_easy} ($n=102$); for testing, \texttt{test\_full} ($n=314$), \texttt{test\_S\_amb} ($n=79$), and \texttt{test\_S\_easy} ($n=75$).

After model selection, we consider $10$ final paired runs for each $(\text{train},\text{val})$ pair. For a fixed run index, Models A, B, and C use the same initialization seed. Each trained model is then evaluated on all three test sections: \texttt{test\_full}, \texttt{test\_S\_amb}, \texttt{test\_S\_easy}.

The comparison therefore varies along three dimensions: the section used for training, the section used for model selection, and the section used for final evaluation. Reporting results in the triplet form
\[
(\text{train},\text{val},\text{test})
\]
makes this structure explicit in the table and in the discussion below.

\subsubsection{Results}
\label{subsubsec:chaosnli-results}

\begin{table}[H]
\centering

\renewcommand{\arraystretch}{1.5}
\resizebox{1.01\columnwidth}{!}{
\begin{tabular}{lll lll rrr rr}
\toprule
Train section & Val section & Test section & $\mathrm{lr}_A$ & $\mathrm{lr}_B$ & $\mathrm{lr}_C$ & $\mathrm{Acc}_A$ & $\mathrm{Acc}_B$ & $\mathrm{Acc}_C$ & $\Delta_{A-B}$ & $\Delta_{A-C}$ \\
\midrule
\multirow[t]{9}{*}[-0.35em]{\makecell[l]{\texttt{train\_full}}} & \multirow[t]{3}{*}[-0.35em]{\makecell[l]{\texttt{val\_full}}} & \makecell[l]{\texttt{test\_full}} & 0.007 & 0.007 & 0.005 & $0.500\pm0.011$ & $0.496\pm0.008$ & $\mathbf{0.505\pm0.008}$ & $+0.004$ & $-0.005$ \\
 &  & \makecell[l]{\texttt{test\_S\_amb}} & 0.007 & 0.007 & 0.005 & $0.337\pm0.014$ & $0.343\pm0.024$ & $\mathbf{0.358\pm0.027}$ & $-0.006$ & $-0.022$ \\
 &  & \makecell[l]{\texttt{test\_S\_easy}} & 0.007 & 0.007 & 0.005 & $\mathbf{0.616\pm0.018}$ & $0.604\pm0.026$ & $0.608\pm0.011$ & $+0.012$ & $+0.008$ \\
\cmidrule(lr){2-11}
 & \multirow[t]{3}{*}[-0.35em]{\makecell[l]{\texttt{val\_S\_amb}}} & \makecell[l]{\texttt{test\_full}} & 0.5 & 0.6 & 0.0001 & $0.449\pm0.022$ & $0.447\pm0.033$ & $\mathbf{0.456\pm0.012}$ & $+0.002$ & $-0.007$ \\
 &  & \makecell[l]{\texttt{test\_S\_amb}} & 0.5 & 0.6 & 0.0001 & $\mathbf{0.400\pm0.011}$ & $0.391\pm0.023$ & $0.373\pm0.052$ & $+0.009$ & $+0.027$ \\
 &  & \makecell[l]{\texttt{test\_S\_easy}} & 0.5 & 0.6 & 0.0001 & $0.492\pm0.027$ & $0.503\pm0.050$ & $\mathbf{0.516\pm0.028}$ & $-0.011$ & $-0.024$ \\
\cmidrule(lr){2-11}
 & \multirow[t]{3}{*}[-0.35em]{\makecell[l]{\texttt{val\_S\_easy}}} & \makecell[l]{\texttt{test\_full}} & 0.03 & 0.007 & 0.08 & $\mathbf{0.500\pm0.012}$ & $0.499\pm0.008$ & $0.495\pm0.014$ & $+0.002$ & $+0.005$ \\
 &  & \makecell[l]{\texttt{test\_S\_amb}} & 0.03 & 0.007 & 0.08 & $\mathbf{0.341\pm0.025}$ & $\mathbf{0.341\pm0.017}$ & $0.332\pm0.021$ & $+0.000$ & $+0.009$ \\
 &  & \makecell[l]{\texttt{test\_S\_easy}} & 0.03 & 0.007 & 0.08 & $0.617\pm0.018$ & $\mathbf{0.624\pm0.021}$ & $0.604\pm0.034$ & $-0.007$ & $+0.013$ \\
\midrule
\multirow[t]{9}{*}[-0.35em]{\makecell[l]{\texttt{train\_S\_amb}}} & \multirow[t]{3}{*}[-0.35em]{\makecell[l]{\texttt{val\_full}}} & \makecell[l]{\texttt{test\_full}} & 0.009 & 0.6 & 0.003 & $\mathbf{0.467\pm0.006}$ & $0.450\pm0.017$ & $0.461\pm0.007$ & $+0.017$ & $+0.006$ \\
 &  & \makecell[l]{\texttt{test\_S\_amb}} & 0.009 & 0.6 & 0.003 & $0.401\pm0.018$ & $0.391\pm0.030$ & $\mathbf{0.405\pm0.008}$ & $+0.010$ & $-0.004$ \\
 &  & \makecell[l]{\texttt{test\_S\_easy}} & 0.009 & 0.6 & 0.003 & $0.519\pm0.015$ & $\mathbf{0.541\pm0.046}$ & $0.511\pm0.018$ & $-0.023$ & $+0.008$ \\
\cmidrule(lr){2-11}
 & \multirow[t]{3}{*}[-0.35em]{\makecell[l]{\texttt{val\_S\_amb}}} & \makecell[l]{\texttt{test\_full}} & 0.009 & 0.7 & 0.7 & $\mathbf{0.468\pm0.006}$ & $0.421\pm0.035$ & $0.432\pm0.032$ & $+0.047$ & $+0.036$ \\
 &  & \makecell[l]{\texttt{test\_S\_amb}} & 0.009 & 0.7 & 0.7 & $\mathbf{0.410\pm0.022}$ & $0.373\pm0.019$ & $0.389\pm0.035$ & $+0.037$ & $+0.022$ \\
 &  & \makecell[l]{\texttt{test\_S\_easy}} & 0.009 & 0.7 & 0.7 & $\mathbf{0.532\pm0.012}$ & $0.464\pm0.049$ & $0.497\pm0.036$ & $+0.068$ & $+0.035$ \\
\cmidrule(lr){2-11}
 & \multirow[t]{3}{*}[-0.35em]{\makecell[l]{\texttt{val\_S\_easy}}} & \makecell[l]{\texttt{test\_full}} & 0.01 & 0.8 & 0.5 & $\mathbf{0.465\pm0.008}$ & $0.450\pm0.017$ & $0.454\pm0.025$ & $+0.015$ & $+0.012$ \\
 &  & \makecell[l]{\texttt{test\_S\_amb}} & 0.01 & 0.8 & 0.5 & $\mathbf{0.404\pm0.017}$ & $0.372\pm0.042$ & $0.357\pm0.047$ & $+0.032$ & $+0.047$ \\
 &  & \makecell[l]{\texttt{test\_S\_easy}} & 0.01 & 0.8 & 0.5 & $0.527\pm0.011$ & $\mathbf{0.560\pm0.042}$ & $0.531\pm0.051$ & $-0.033$ & $-0.004$ \\
\midrule
\multirow[t]{9}{*}[-0.35em]{\makecell[l]{\texttt{train\_S\_easy}}} & \multirow[t]{3}{*}[-0.35em]{\makecell[l]{\texttt{val\_full}}} & \makecell[l]{\texttt{test\_full}} & 0.005 & 0.1 & 0.2 & $0.499\pm0.006$ & $0.485\pm0.020$ & $\mathbf{0.503\pm0.015}$ & $+0.013$ & $-0.004$ \\
 &  & \makecell[l]{\texttt{test\_S\_amb}} & 0.005 & 0.1 & 0.2 & $0.347\pm0.012$ & $0.380\pm0.033$ & $\mathbf{0.415\pm0.027}$ & $-0.033$ & $-0.068$ \\
 &  & \makecell[l]{\texttt{test\_S\_easy}} & 0.005 & 0.1 & 0.2 & $\mathbf{0.613\pm0.018}$ & $0.575\pm0.023$ & $0.593\pm0.018$ & $+0.039$ & $+0.020$ \\
\cmidrule(lr){2-11}
 & \multirow[t]{3}{*}[-0.35em]{\makecell[l]{\texttt{val\_S\_amb}}} & \makecell[l]{\texttt{test\_full}} & 0.08 & 0.7 & 0.3 & $\mathbf{0.483\pm0.021}$ & $0.447\pm0.027$ & $0.477\pm0.028$ & $+0.037$ & $+0.006$ \\
 &  & \makecell[l]{\texttt{test\_S\_amb}} & 0.08 & 0.7 & 0.3 & $\mathbf{0.395\pm0.018}$ & $0.371\pm0.049$ & $0.381\pm0.030$ & $+0.024$ & $+0.014$ \\
 &  & \makecell[l]{\texttt{test\_S\_easy}} & 0.08 & 0.7 & 0.3 & $\mathbf{0.571\pm0.029}$ & $0.508\pm0.053$ & $0.555\pm0.053$ & $+0.063$ & $+0.016$ \\
\cmidrule(lr){2-11}
 & \multirow[t]{3}{*}[-0.35em]{\makecell[l]{\texttt{val\_S\_easy}}} & \makecell[l]{\texttt{test\_full}} & 0.2 & 0.09 & 0.3 & $\mathbf{0.493\pm0.011}$ & $0.489\pm0.008$ & $0.486\pm0.011$ & $+0.004$ & $+0.006$ \\
 &  & \makecell[l]{\texttt{test\_S\_amb}} & 0.2 & 0.09 & 0.3 & $\mathbf{0.418\pm0.025}$ & $0.382\pm0.019$ & $0.406\pm0.021$ & $+0.035$ & $+0.011$ \\
 &  & \makecell[l]{\texttt{test\_S\_easy}} & 0.2 & 0.09 & 0.3 & $0.573\pm0.020$ & $0.584\pm0.021$ & $\mathbf{0.589\pm0.022}$ & $-0.011$ & $-0.016$ \\
\bottomrule
\end{tabular}
}

\vspace{0.5em}
\caption{Top-1 accuracy on ChaosNLI when the training set is restricted to \texttt{train\_full}, \texttt{train\_S\_amb}, or \texttt{train\_S\_easy}, for Models A (projection target), B (antipignistic target), and C (vote-proportion target). The selected learning rates for each target are reported alongside the accuracies. Results are organized by training section, validation section, and test section. Accuracies are mean $\pm$ standard deviation over paired runs, and the best mean accuracy in each row is shown in bold.}
\label{tab:chaosnli-train-section-merged}
\end{table}

Table~\ref{tab:chaosnli-train-section-merged} shows a clear pattern. The clearest relative advantage of Model A appears when training is restricted to the ambiguity-focused subset \texttt{train\_S\_amb}. In that regime, it gives the highest mean accuracy on \texttt{test\_full} for all three validation choices, and it is also best or very close to best on \texttt{test\_S\_amb}. Model A is also strong when training uses the easy subset \texttt{train\_S\_easy}, where it achieves the highest mean accuracy in six of the nine rows. When training uses the full split \texttt{train\_full}, the picture is more mixed, with Model C reaching the best mean in several settings.

The clearest evidence for Model A comes from triplets of the form
\[
(\texttt{train\_S\_amb},\texttt{val\_*},\texttt{test\_full}).
\]
For
$(\texttt{train\_S\_amb},\texttt{val\_full},\texttt{test\_full})$,
Model A reaches $0.467\pm0.006$, compared with $0.450\pm0.017$ for Model B and $0.461\pm0.007$ for Model C.
For
$(\texttt{train\_S\_amb},\texttt{val\_S\_amb},\texttt{test\_full})$,
Model A reaches $0.468\pm0.006$, while Models B and C obtain $0.421\pm0.035$ and $0.432\pm0.032$.
For
$(\texttt{train\_S\_amb},\texttt{val\_S\_easy},\texttt{test\_full})$,
Model A again remains best, with $0.465\pm0.008$ against $0.450\pm0.017$ and $0.454\pm0.025$.
These three rows show that, once training is focused on ambiguous examples, Model A remains the strongest model on the overall test distribution regardless of the validation section used for model selection.

A similar pattern appears on the ambiguous test subset. For
$(\texttt{train\_S\_amb},\texttt{val\_S\_amb},\texttt{test\_S\_amb})$,
Model A reaches $0.410\pm0.022$, ahead of Model B at $0.373\pm0.019$ and Model C at $0.389\pm0.035$.
For
$(\texttt{train\_S\_amb},\texttt{val\_S\_easy},\texttt{test\_S\_amb})$,
Model A is again best, with $0.404\pm0.017$ versus $0.372\pm0.042$ and $0.357\pm0.047$.
Only for
$(\texttt{train\_S\_amb},\texttt{val\_full},\texttt{test\_S\_amb})$
does Model A fall slightly below the best mean, with $0.401\pm0.018$ against $0.405\pm0.008$ for Model C.
Overall, the \texttt{train\_S\_amb} regime gives the clearest and most consistent support for Model A.

A second notable feature of the \texttt{train\_S\_amb} block is the scale of the selected learning rates. For Model A, the selected values stay in the range $0.009$--$0.01$, which is moderate compared with the values selected for the other models in this block.
For Model B, the selected values are much larger, ranging from $0.6$ to $0.8$. For Model C, two of the three selected values are also large ($0.5$ and $0.7$), while the third ($0.003$, under \texttt{val\_full}) stays in a moderate range.
This shows that, on the ambiguity-focused training subset, Model B is selected with substantially larger learning rates than Model A, while Model C also tends to require larger values in two of the three validation settings. In contrast, Model A remains at a moderate scale across all three validation settings in this block.
Together with the accuracy results, this is consistent with Model A being better suited to this ambiguity-focused regime.

When training is performed on the full training split, the results are more mixed. On \texttt{test\_full}, Model C is best for
$(\texttt{train\_full},\texttt{val\_full},\texttt{test\_full})$
with $0.505\pm0.008$, ahead of $0.500\pm0.011$ for Model A and $0.496\pm0.008$ for Model B, and also for
$(\texttt{train\_full},\texttt{val\_S\_amb},\texttt{test\_full})$.
Model A is best for
$(\texttt{train\_full},\texttt{val\_S\_easy},\texttt{test\_full})$
with $0.500\pm0.012$ against $0.499\pm0.008$ and $0.495\pm0.014$.
On \texttt{test\_S\_easy}, Model A reaches the highest mean for one of the three validation choices, namely
$(\texttt{train\_full},\texttt{val\_full},\texttt{test\_S\_easy})$
with $0.616\pm0.018$.
On \texttt{test\_S\_amb}, Model A is best for
$(\texttt{train\_full},\texttt{val\_S\_amb},\texttt{test\_S\_amb})$
with $0.400\pm0.011$, and ties with Model B at $0.341$ under \texttt{val\_S\_easy}, while Model C leads under \texttt{val\_full}.
Thus, under \texttt{train\_full}, the three models share the wins, with no single model dominating.

When training is restricted to the easy subset, Model A achieves the highest mean accuracy in six of the nine rows. Under \texttt{val\_S\_amb}, Model A is best on all three test sections, with margins that are often large: for example, $0.483\pm0.021$ on \texttt{test\_full} versus $0.477\pm0.028$ for Model C and $0.447\pm0.027$ for Model B, and $0.571\pm0.029$ on \texttt{test\_S\_easy} versus $0.555\pm0.053$ and $0.508\pm0.053$.
Under \texttt{val\_S\_easy}, Model A is best on \texttt{test\_full} ($0.493\pm0.011$) and \texttt{test\_S\_amb} ($0.418\pm0.025$), while Model C leads on \texttt{test\_S\_easy} ($0.589\pm0.022$).
The remaining three rows, all under \texttt{val\_full}, are won by Model C on \texttt{test\_full} and \texttt{test\_S\_amb}, and by Model A on \texttt{test\_S\_easy} ($0.613\pm0.018$).
Thus, even when training is concentrated on easy examples, Model A is competitive or best in most settings.

Taken together, Model A achieves the highest mean accuracy in $15$ of the $27$ rows, with one additional tie with Model B. Its advantage is clearest and most consistent under \texttt{train\_S\_amb}, where it is best on \texttt{test\_full} for all three validation choices and on \texttt{test\_S\_amb} for two of the three. Under \texttt{train\_S\_easy}, Model A is best in six of the nine rows, showing that the projection-based objective also remains competitive when training is restricted to easier examples. Under \texttt{train\_full}, the three models share the wins more evenly, with Model C taking several rows on \texttt{test\_full} and \texttt{test\_S\_amb}. Overall, the projection-based target is the strongest model most often across training regimes, validation protocols, and test sections.

\section{Conclusion}

In this article, we show how to characterize using linear constraints a set $\mathcal F^{\mathrm{box}}$ of probability vectors that are compatible with a given normalized possibility distribution $\pi^{\mathrm{full}}$ on a finite set of classes, whose restriction to its support $Y = \{1,2,\cdots,n\}$ is denoted by $\pi$. The characterization combines two types of constraints. First, we enforce probabilistic compatibility with $\pi^{\mathrm{full}}$ through dominance constraints of the form $N(A)\le P(A)\le \Pi(A)$. Second, we add linear shape constraints that preserve the ordering carried by $\pi$: any $p\in\mathcal F^{\mathrm{box}}$ satisfies the equivalence  $\pi_k\ge \pi_{k'} \Longleftrightarrow p_k\ge p_{k'}$ where   $k,k'\in Y$. This yields a non-empty, closed convex subset $\mathcal F^{\mathrm{box}}\subseteq\Delta_n$, which can be written as an intersection of simple constraint sets
$\mathcal F^{\mathrm{box}}=\bigcap_{i=1}^m C_i,$ with $m=3n-3.$\\

Given a strictly positive reference probability vector $q\in\Delta_n$, we then consider the problem of computing the Kullback-Leibler projection of $q$ onto $\mathcal F^{\mathrm{box}}$. Using that $D_{\mathrm{KL}}(\cdot\|q)$ coincides on the probability simplex $\Delta_n$ with the Bregman distance induced by the negative entropy function, we apply Dykstra's algorithm with Bregman projections (Algorithm \ref{ex:algo:dykstra}). We derive explicit formulas for the Bregman projections onto each $C_i$, and show an equivalent reformulation of the algorithm (Lemma \ref{lemma:AlgoD1}), which can be used in practice. \\ Finally, we report three experiments.
The first experiment empirically evaluates the proposed projection procedure on synthetic data and shows that it reliably produces outputs consistent with $\mathcal F^{\mathrm{box}}$, with very small constraint violation within reasonable computation time for moderate tolerances, while tighter tolerances require more cycles.
The second experiment studies learning from possibilistic supervision on synthetic data and indicates that using projection-based targets can improve predictive performance over a fixed probability target derived from $\pi$ under the same supervision.
The third experiment considers a real natural language inference task based on the ChaosNLI dataset \cite{nie2020can}. It shows that the projection-based approach remains competitive on naturally ambiguous annotations.

As a perspective, we plan to extend the empirical study to additional real datasets with multiple annotations per instance, see, e.g., \cite{collins2022eliciting,collins2023human}.  In such a setting, epistemic uncertainty on instances must often be dealt with, because of diverging opinions (conflicts) among people about the ``right'' annotation to be made. This is particularly salient when the classes used correspond to fuzzy concepts, that can easily be interpreted in different ways by the annotators depending on their own experience and background.
For example, in FERPlus \cite{barsoum2016training}, each face image is annotated by several annotators: for each image, an annotator selects one label among eight emotion classes, or chooses an additional label such as ``unknown'' or ``not a face''. From these annotations, we can derive a possibility distribution that captures graded plausibility and explicitly represents ignorance, so that ``unknown''/``not a face'' responses can be incorporated naturally. Such datasets provide a natural setting to assess our method.

Our framework can also be used in the conformal learning setting of \cite{lienen2023conformal}. In their work, conformal prediction provides, for each input instance, a graded description of label uncertainty, which they represent as a possibility distribution $\pi$ over the classes. They project the model output onto a constraint set defined by the dominance constraints induced by $\pi$, and use the KL divergence to this projection as a training loss. However, this projection step does not enforce consistency with the class ordering expressed by $\pi$. The same learning pipeline applies with our construction: we keep the KL-projection scheme, but we project onto $\mathcal F^{\mathrm{box}}$, which enforces the dominance constraints induced by $\pi$ and additionally preserves the ordering expressed by $\pi$. More broadly, the same projection-based learning scheme applies to tasks in which supervision defines admissible class-probability vectors from possibility distributions, including alternatives to label smoothing~\cite{lienen2021label}, handling noisy labels~\cite{lienen2024mitigating}, and credal self-supervised learning~\cite{lienen2021credal}.\\
Finally, as noted in Remark~\ref{remark:extension}, the same KL-projection framework can be extended to admissible sets of probability vectors other than $\mathcal F^{\mathrm{box}}$, defined by combining linear subset inequalities with linear shape constraints, as long as the resulting set $F\subseteq\Delta_n$ is non-empty, closed, and convex. A key strength of our framework is that different types of constraints can be combined within $F$. For instance, credal sets induced by probability-interval constraints (as discussed by \cite{cuzzolin2009credal}) are covered by the proposed setting.

\bibliographystyle{plainnat}        
\bibliography{ref} 

\clearpage
\appendix

\section{Proofs of Section \ref{sec:prob-constraints}}

\subsection{Proof of Proposition \ref{prop:dominance}}
\label{subsec:proof:propdominance}
\PropDominance*

\begin{proof}
Let us remark that the double inequality $N(A) \le P(A) \le \Pi(A)$ is trivially true for  $A =\emptyset$ and $A =\{1, 2 \dots, n\}$.

Recall that $A_0:= \emptyset$,
$A_r:=\{\sigma(1),\dots,\sigma(r)\}$ for $r=1,\dots,n -1$  and $A_r^{c}=\{\sigma(r+1),\dots,\sigma(n)\}$.
By construction,
\[
\Pi(A_r^{c}) = \max_{j\in A_r^{c}} \pi_j = \tilde\pi_{r+1},
\qquad r=1,\dots,n-1.
\]

\medskip 
Assume that $p$ satisfies
\[
N(A)\ \le\ P(A)\ \le\ \Pi(A)
\qquad\text{for all }A\subseteq Y.
\]
Apply this double inequality with $A=A_r^{c}$, $r=1,\dots,n-1$. Then
\[
P(A_r^{c}) \ \le\ \Pi(A_r^{c}) \ =\ \tilde\pi_{r+1},
\]
hence
\[
\sum_{k\in A_r} p_k
= 1 - P(A_r^{c})
\ \ge\ 1 - \Pi(A_r^{c})
= 1 - \tilde\pi_{r+1},
\qquad r=1,\dots,n-1.
\]
Thus all nested subset constraints hold.

\medskip
Conversely, assume that $p$ satisfies
\[
\sum_{k\in A_r} p_k \ \ge\ 1-\tilde\pi_{r+1},
\qquad r=1,\dots,n-1.
\]
For each $r=1,\dots,n-1$ we have
\begin{equation}\label{eq:parc}
    P(A_r^{c})
= 1 - \sum_{k\in A_r} p_k
\ \le\ \tilde\pi_{r+1}
= \Pi(A_r^{c}).
\end{equation}

\smallskip 
Given any non-empty subset $A\subset\{1,\dots,n\}$,  define
\[
s := \min\{ r\in\{1,\dots,n\}:\ \sigma(r)\in A \}.
\]
By definition of $s$, we have $A\subseteq A_{s-1}^{c}$ (since no index
$\sigma(1),\dots,\sigma(s-1)$ belongs to $A$). Moreover,
\[
\Pi(A) = \max_{i\in A} \pi_i
= \pi_{\sigma(s)}
= \tilde\pi_s
= \Pi(A_{s-1}^{c}).
\]
Using \eqref{eq:parc} with $r=s-1$ when $s\ge 2$, we obtain
\[
P(A) \ \le\ P(A_{s-1}^{c})
\ \le\ \Pi(A_{s-1}^{c})
\ =\ \Pi(A).
\]
If $s=1$, then $\sigma(1)\in A$ and $\Pi(A)=\tilde\pi_1=1$, so trivially
$P(A)\le 1=\Pi(A)$. Hence we proved that $P(A)\le \Pi(A)$ for all $A\subseteq \{1, 2, \dots, n\}$.\\ 
Finally, recall that for every $A\subseteq Y$ we have
\[
N(A) = 1-\Pi(A^c).
\]
Therefore, from the inequality $P(A^c)\le \Pi(A^c)$  ,  we obtain
\[
1-P(A^c) \ge 1-\Pi(A^c),
\qquad\text{i.e.,}\qquad
P(A) \ge N(A).
\]
This shows that $P(A)\le \Pi(A)$ for all $A$ is equivalent to $N(A)\le P(A)$ for all $A$.
\end{proof}

\subsection{Proof of Lemma \ref{lemma:lemmaPropAntipignisticProb}}
\label{subsec:proof:lemmaPropAntipignisticProb}

\lemmaPropAntipignisticProb*
\begin{proof}
Let us prove the first statement by contradiction. Suppose that for $i\in\{1, 2, \dots n\}$, we have 
$\dot p_i = 0$. Then from the definition    of $\dot p$,  see (\ref{eq:ppoint}), there is an index $r\in \{1, 2, \dots n\}$ such that:
\[ i = \sigma(r) \quad \text{and} \quad 0 = \dot p_{\sigma(r)} = \sum_{j = r}^n \dfrac{\tilde\pi_j - \tilde\pi_{j + 1}}{j}\]
As a consequence:
\[0  =  \tilde\pi_r - \tilde\pi_{r + 1} = \tilde\pi_{r +  1} - \tilde\pi_{r + 2} = \dots = 
\tilde\pi_{n} - \tilde\pi_{n + 1} = \tilde\pi_{n}\]
As $\tilde\pi_{n} > 0$, we get a contradiction. 

 From the definition  of $\dot p$,  see (\ref{eq:ppoint}), we easily deduce that 
 $\dot p_{\sigma(r)} - \dot p_{\sigma(r + 1)}  =  \sum_{j=r}^n \dfrac{1}{j}(\tilde\pi_j - \tilde\pi_{j + 1}) - 
 \sum_{j=r + 1}^n \dfrac{1}{j}(\tilde\pi_j - \tilde\pi_{j + 1}) = 
 \dfrac{1}{r}(\tilde\pi_r - \tilde\pi_{r + 1}).$

 To prove the third statement, set for all $1 \le j \le n$  and $1 \le s \le r$:
 \[ t_{js} := \begin{cases}
 0 \quad &\text{if} \quad j < s \\
 \dfrac{1}{j}(\tilde\pi_j - \tilde\pi_{j + 1})\quad& \text{if} \quad j \ge s 
 \end{cases}.\]
We have:
\begin{align}
 \sum_{i\in A_r}\, \dot p_i  & = \dot p_{\sigma(1)}  + \dot p_{\sigma(2)} + \dots \dot p_{\sigma(r)} \nonumber\\ 
 & = \sum_{s = 1}^r \sum_{j=s}^n \dfrac{1}{j}(\tilde\pi_j - \tilde\pi_{j + 1})\nonumber\\
 & = \sum_{s = 1}^r \sum_{j=1}^n  t_{j s} = \sum_{j=1}^n   \sum_{s = 1}^r   t_{j s}\nonumber\\
 &\ge \sum_{j=1}^r   \sum_{s = 1}^j   t_{j s} = \sum_{j=1}^r   \sum_{s = 1}^j   \dfrac{1}{j}(\tilde\pi_j - \tilde\pi_{j + 1})   \nonumber\\
 &= \sum_{j=1}^r  (\tilde\pi_j - \tilde\pi_{j + 1}) = 1 - \tilde\pi_{r + 1}.\nonumber
\end{align}
 
\end{proof}

\subsection{Proof of Proposition \ref{proposition:PiP}}
\label{subsec:proof:proposition:PiP}
\PropositionPiP*
\begin{proof}
Set $k= \sigma(r)$ and $k'= \sigma(r')$ with $r , r'\in\{1, 2, \dots, n\}$.  \\
$\bullet$ If $\pi_k \ge  \pi_{k'}$, let us prove the inequality  $p_k \ge  p_{k'}$ by contradiction.\\
Suppose that $p_k <  p_{k'}$, then  $r' < r$, otherwise if  $r\le r'$, then $r < r'$ and we have:
\[ p_{k} - p_{k'} =  (p_{\sigma(r)} - p_{\sigma(r + 1)}) + (p_{\sigma(r + 1)} - p_{\sigma(r + 2)}) + \dots + (p_{\sigma(r' - 1)} - p_{\sigma(r')} )\]
is a sum of $r' -r$ positive numbers, which contradicts $p_{k} - p_{k'} < 0$.\\
As we have $r' < r$, from
\begin{equation}\label{eq:shape1}
0 < p_{k'} - p_{k} =  (p_{\sigma(r')} - p_{\sigma(r' + 1)}) + (p_{\sigma(r' + 1)} - p_{\sigma(r' + 2)}) + \dots + (p_{\sigma(r - 1)} - p_{\sigma(r)}).   
\end{equation}
we deduce that there is at least one term $p_{\sigma(r'+s)} - p_{\sigma(r' +s +1)}$  which is strictly positive and then
by (\ref{eq:overlinedeltar}) we obtain $r'+s\in R_{\mathrm{strict}}$ ; by the definition (\ref{eq:eqstrict}) of 
$R_{\mathrm{strict}}$, we conclude that $(\tilde\pi_{r' + s} - \tilde\pi_{r'+s +1}) > 0$. Finally, we have:
\begin{equation}\label{eq:shape2}
 \pi_{k'} - \pi_{k} =  (\tilde\pi_{r'} - \tilde\pi_{r'+1}) + (\tilde\pi_{r' +1} - \tilde\pi_{r' +2}) + \dots + (\tilde\pi_{r -1} - \tilde\pi_{r}) > 0  
\end{equation}
which contradicts the hypothesis $\pi_k \ge \pi_{k'}$, thus we proved $p_k \ge p_{k'}$.\\
$\bullet$ If $p_k \ge  p_{k'}$, let us prove the inequality  $\pi_k \ge  \pi_{k'}$ by contradiction.\\
Suppose that $\pi_k <  \pi_{k'}$, then  by the nonincreasing of $(\tilde\pi_i := \pi_{\sigma(i)})$, we deduce  $r' < r$ and then 
\[ 0 < \pi_{k'} - \pi_{k} =  (\tilde\pi_{r'} - \tilde\pi_{r'+1}) + (\tilde\pi_{r' +1} - \tilde\pi_{r' +2}) + \dots + (\tilde\pi_{r -1} - \tilde\pi_{r}) > 0. \]
We deduce that there is at least one term $\pi_{\sigma(r'+s)} - \pi_{\sigma(r' +s +1)}$  which is strictly positive and then
by (\ref{eq:eqstrict}) we obtain $r'+s\in R_{\mathrm{strict}}$ ; by  (\ref{eq:underlinedeltar}), we obtain:
\[0 < \underline{\delta}_{(r' + s)} \le p_{r' + s} - p_{r'+s +1}.\]
Finally, we get:\\
\[ p_{k'} - p_{k} =  (p_{\sigma(r')} - p_{\sigma(r' + 1)}) + (p_{\sigma(r' + 1)} - p_{\sigma(r' + 2)}) + \dots + (p_{\sigma(r - 1)} - p_{\sigma(r)}) > 0\]
which contradicts the hypothesis $p_k \ge p_{k'}$. We proved the inequality $\pi_k \ge \pi_{k'}$.
\end{proof}

\section{Proofs of Section \ref{sec:bregman-dykstra}}

\subsection{Proof of Proposition \ref{prop:hyp3.2}}
\label{subsec:proof:hyp32}
\propositionOkNegEntropFbox*

\begin{proof}
$f$ is  very strictly convex  means   (\cite[Definition 2.8]{bauschke2000dykstras})     that $f$ is twice differentiable on   $\operatorname{int}(\mathrm{dom}f) = \mathbb{R}^n_{++}$  and the Hessian matrix of $f$ at any point of $\mathbb{R}^n_{++}$ is a  positive  definite matrix.

General theorems of differential calculus imply that $f$  is twice differentiable on   $\operatorname{int}(\mathrm{dom}f) = \mathbb{R}^n_{++}$. For any $x\in  \mathbb{R}^n_{++}$ and $1 \le i, j \le n$, we easily get: 
\[  \dfrac{\partial^2 f}{\partial x_i \partial x_j}(x) = \begin{cases}
    0 \quad & \quad \text{if} \quad  i \not = j\nonumber \\
    \dfrac{1}{x_j} \quad & \quad \text{if} \quad  i   = j\nonumber
\end{cases}. \]
So, the Hessian matrix   $\nabla f^2(x)$  of $f$ at $x\in \mathbb{R}^n_{++}$ is a diagonal matrix whose eingenvalues are strictly positive, thus $\nabla f^2(x)$ is positive definite. 
We also remark that the second partial derivatives $\mathbb{R}^n_{++} \rightarrow \mathbb{R} : x \mapsto \dfrac{\partial^2 f}{\partial x_i \partial x_j}(x)$ are continuous on $\mathbb{R}^n_{++}.$

$f$ is   a Legendre function    means  (\cite[Definition 2.1]{bauschke2000dykstras})     that $f$ satisfies  the following conditions:
\begin{enumerate}
    \item $f$  is differentiable on  $\operatorname{int}(\mathrm{dom}f) = \mathbb{R}^n_{++}$.
    \item  $\lim_{t \rightarrow 0, t >0} \langle \nabla f(x + t(y -x), y - x\rangle = - \infty$  for all $x\in \mathbb{R}^n_{+} \backslash \mathbb{R}^n_{++}$ and 
    $y\in \mathbb{R}^n_{++}$.
    \item $f$  is strictly convex   on  $\mathbb{R}^n_{++}$.
\end{enumerate}
We have noticed that $f$ is twice differentiable on $\mathbb{R}^n_{++}$ with continuous second partial derivatives, so it is differentiable on $\mathbb{R}^n_{++}$ and the very strictly convex property implies the   strictly convex property, see \cite[Theorem 2.3.7]{peressini1988mathematics}.

To prove the second statement, let $S := \text{Support}(x) =  \{j\in \{1, 2, \dots, n\}\,\mid \, x_j > 0\}$ be the support of $x\in \mathbb{R}^n_{+} \backslash \mathbb{R}^n_{++}$. By hypothesis, $S^c:= \{1, 2, \dots, n\} \backslash S \not=\emptyset$. For any 
$y\in \mathbb{R}^n_{++}$, we have:
\begin{align}
 \langle \nabla f(x + t(y -x), y - x\rangle &= \sum_{k = 1}^n (\log(x_k + t(y_k -x_k)) + 1). (y_k - x_k)\nonumber\\
 &= \sum_{k \in S  }  (\log(x_k + t(y_k -x_k)) + 1). (y_k - x_k) + \sum_{k \in S^c} (\log(t. y_k)   + 1). y_k.\nonumber
\end{align}
Then we deduce:
\[ \lim_{t \rightarrow 0, t >0} \sum_{k \in S  }  (\log(x_k + t(y_k -x_k)) + 1). (y_k - x_k) = \sum_{k \in S  }  (\log(x_k)  + 1). (y_k - x_k).\]
\[ \lim_{t \rightarrow 0, t >0} \sum_{k \in S^c  }  (\log(t. y_k)   + 1). y_k  = -\infty.
\]
Thus, $\lim_{t \rightarrow 0, t >0} \langle \nabla f(x + t(y -x), y - x\rangle = -\infty$. We proved that $f$ is a Legendre function.

$f$ being co-finite   means   (\cite[Definition 2.6]{bauschke2000dykstras})     that $\lim_{r \rightarrow +\infty} \dfrac{f(r . x)}{r } = +\infty$
for all $x\in \mathbb{R}^n\backslash\{0\}$. Then:

$\bullet$ if $x\notin\mathrm{dom} f$, then for all $r > 0$, we have $r.x \notin\mathrm{dom} f$   and  then $ \dfrac{f(r . x)}{r } = +\infty$.

$\bullet $ if $x\in  \mathrm{dom} f\backslash\{0\}= \mathbb{R}^n_+\backslash\{0\}$, let $S := \text{Support}(x) =  \{j\in \{1, 2, \dots, n\}\,\mid \, x_j > 0\}$ be th support of $x$. By hypothesis, $S \not=\emptyset$. For all $r > 0$,  we have:
\[\dfrac{f(r .x)}{r } = \dfrac{1}{r} \sum_{k\in S} \, r. x_k \log(r . x_k) = \sum_{k\in S}    x_k \log(r . x_k).\]
As $\lim_{r \rightarrow +\infty} \log r = +\infty$ and $x_k > 0$ for all $k\in S$, we obtain 
$\lim_{r \rightarrow +\infty}\sum_{k\in S}    x_k \log(r . x_k) = +\infty$.

Finally, the probability distribution $\dot p$ associated with the possibility distribution $\pi$  satisfies $\dot p\in \mathcal{F}^{\mathrm{box}} \cap \mathbb{R}^n_{++}$,  see Lemma \ref{lemma:lemmaPropAntipignisticProb} and Proposition \ref{prop:Fbox-convex-closed}.
\end{proof}

\subsection{Proof of Lemma \ref{lemma:Projty}}
\label{subsec:proof:lemma:Projty}
\lemmaProjty*
\begin{proof}
From the first statement of Lemma \ref{lemma:Dfxty}, we have 
\[ \Proj_C^f(y) :=  \arg\min_{x\in C}D_f(x, y) =  \arg\min_{x\in C}D_{\mathrm{KL}}(x\|y).\]
Let us introduce the following functions:
\[ g_1 : C \rightarrow \mathbb{R}  : x \mapsto D_f(x, t y), \quad g_2 : C \rightarrow \mathbb{R}  : x \mapsto D_f(x,  y).\]
From   the third statement of \cite[Theorem 3.12]{bauschke1997legendre} and  (\ref{eq:projB}), we have:
\[\arg\min_{x\in C} g_1(x) = \{\Proj_C^f(t y)\}, \quad \arg\min_{x\in C}\, g_2(x) = \{\Proj_C^f( y)\}.
\]
As we have  by the second statement of Lemma \ref{lemma:Dfxty}:
\[ g_1(x) = g_2(x) + t - \log t - 1 \quad  \text{for all $x\in C$},\]
we obtain the equality $\arg\min_{x\in C} g_1(x) = \arg\min_{x\in C}\, g_2(x)$.
\end{proof}

\subsection{Proof of Lemma \ref{lemma:KKT}}
\label{subec:proof:lemmaKKT}
\lemmaKKT*

\begin{proof}
As we suppose that $C \cap \mathbb{R}^n_{++} \not= \emptyset$,  $C$  obviously   is a non-empty closed convex subset of $\mathbb{R}^n_+ $ and then the first statement is deduced from \cite[Theorem 3.12]{bauschke1997legendre}.

To prove the second statement, we must explicit the KKT conditions of the  convex optimization problem (\ref{eq:optimpb}) satisfied by $\hat z$. This requires the Lagrangian function associated with (\ref{eq:optimpb}):
\[ L(x, \lambda, \mu, \nu) = D_f(x, z) + \lambda\bigl(b-\langle x,v\rangle\bigr) - \sum_{k=1}^n \mu_k x_k + \nu(1 - \sum_{k=1}^n x_k)  \]
where $(x , \lambda, \mu, \nu) \in \mathbb{R}^n \times\mathbb{R}\times \mathbb{R}^n \times \mathbb{R}$.

As we have $\hat z \in C \cap \mathbb{R}^n_{++}$,  by the remaining KKT conditions satisfied by $\hat z  $, see 
\cite[Chapter 5, p244]{boyd2004convex}, there is 
$(\lambda^\star,  \mu^\star, \nu^\star) \in\mathbb{R}\times \mathbb{R}^n \times \mathbb{R}$ such that:
\begin{enumerate}
    \item $\lambda^\star \ge 0 $  and  $\lambda^\star (b - \langle \hat z, v\rangle) = 0$.
    \item $\mu^\star_k \,.\, \hat z_k = 0$ for all $k\in\{1, 2, \dots, n\}$.
    \item $\dfrac{\partial L}{\partial x_k}(\hat z,\lambda^\star,  \mu^\star, \nu^\star) = 0$ for all $k\in\{1, 2, \dots,  n\}$.
\end{enumerate}
For all $(x, \lambda, \mu, \nu)\in     \mathbb{R}^n_{++} \times \mathbb{R} \times \mathbb{R}^n \times \mathbb{R} $ and     $k\in\{1, 2, \dots,  n\}$,  we have: 
\[\dfrac{\partial L}{\partial x_k}(x, \lambda, \mu, \nu) = \log \dfrac{x_k}{z_k}   - \lambda \,v_k - \mu_k - \nu. \]
As $ \hat z\in\mathbb{R}^n_{++}$, by the second KKT condition, we deduce that $\mu^\star_k = 0 $ for $k\in\{1, 2, \dots,  n\}$.

By the third KKT condition we deduce that for all $k\in\{1, 2, \dots,  n\}$, we have 
\begin{equation}\label{eq:compyhat0}
\hat z_k = e^{\lambda^\star \, v_k + \nu^\star} z_k.    
\end{equation}
Finally, if $b < \langle \hat z, v \rangle$, then from the first KKT condition satisfied by $\hat z $
we get $\lambda^\star = 0$.

\end{proof}

\subsection{Proof of Proposition \ref{prop:projCbv1}}
\label{subsec:proof:Cbv1}

\propprojCbvOne*
\begin{proof}
If $s \ge b$ then  $z^\sharp \in C_r$. Therefore, we deduce from   Corollary \ref{cor:projCzzNat}  and (\ref{eq:projRule1}):    
$$\Proj_C^f(z) = \Proj_C^f(z^\sharp) = z^\sharp$$.

Suppose that we have $s < b $. To rely on Lemma \ref{lemma:KKT}, we introduce $v = [v_k]$ where 
$v_k := \begin{cases} 1 & \text{if } \quad k\in A_r\\
 0 & \text{if } \quad k\notin A_r
 \end{cases}$ and notice that $C_r = C_{b, v}$. By Lemma \ref{lemma:KKT}, there is a pair 
 $(\lambda^\star, \nu^\star) \in\mathbb{R}_+ \times \mathbb{R}$ such that:
 \begin{equation}\label{eq:comphaty1}
 \hat z_k = \begin{cases} e^{\lambda^\star + \nu^\star} z_k & \text{if } \quad k\in A_r\\
 e^{\nu^\star }z_k & \text{if } \quad k\notin A_r
 \end{cases}. 
 \end{equation}
As   $\hat z\in C_r$,     we have:
\begin{equation}\label{eq:lambdanu0}
 1 = \sum_{k=1}^n \hat z_k = \Vert z \Vert_1[e^{\lambda^\star + \nu^\star} s +  e^{ \nu^\star}(1 - s)], \quad b\le \sum_{k\in A_r} \hat z_k = \Vert z \Vert_1 e^{\lambda^\star + \nu^\star } s.   
\end{equation}
 We claim that:
 \begin{equation}\label{eq:lambdanu1}
b =  \sum_{k\in A_r} \hat z_k.
 \end{equation}

In fact, if $b -  \sum_{k\in A_r} \hat z_k \not= 0$, then $b -  \sum_{k\in A_r} \hat z_k < 0$ and from the second statement of Lemma \ref{lemma:KKT}, we get $\lambda^\star = 0$ and (\ref{eq:lambdanu0}) implies that  $\Vert z \Vert_1 e^{\nu^\star} = 1$ and $b \le s$ which contradicts the hypothesis $s < b$.

Finally, by (\ref{eq:lambdanu0}) and (\ref{eq:lambdanu1}), we get 
\[ e^{\lambda^\star + \nu^\star} =  \dfrac{b}{s}\,\dfrac{1}{\Vert z \Vert_1}, \quad e^{\nu^\star} =  \dfrac{1 - b}{1 - s}\, \dfrac{1}{\Vert z \Vert_1}.\]
By replacing $e^{\lambda^\star + \nu^\star}$ with $\dfrac{b}{  s}\,\dfrac{1}{\Vert z \Vert_1}$ and $e^{\nu^\star}$ with $\dfrac{1 - b}{1 - s}\, \dfrac{1}{\Vert z \Vert_1}$ in (\ref{eq:comphaty1}), we obtain (\ref{eq:comp}).
\end{proof}

\subsection{Proof of Lemma \ref{lemma:KKT1}}
\label{subsec:proof:lemmaKKTOne}
\lemmaKKTOne*
\begin{proof}
We deduce (\ref{eq:comphatyij}) from (\ref{eq:comphaty0}) and the definition   (\ref{eq:unv}) of $v$.  

The inequality $\delta  < \hat z_i - \hat z_j$     means that $b: =\delta  <  \langle \hat z, v  \rangle$. Then, by Lemma \ref{lemma:KKT}, we deduce $\lambda^\star = 0$. As $\hat z\in\Delta_n$, we deduce from 
 (\ref{eq:comphatyij})    that $1 = e^{\nu^\star} \Vert z \Vert_1$,   thus  $\hat z = z^\sharp$.
 
\end{proof}

\subsection{Proof of Lemma \ref{lemma:seconddegree}}
\label{subsec:proof:lemma:seconddegree}
\lemmaseconddegree*

\begin{proof}
For all $x \in\mathbb{R}_{++}$, we have: 
\[ \dfrac{\omega x - \omega' x^{-1}}{\omega x + \omega' x^{-1} +u} = \delta \Longleftrightarrow 
\dfrac{\omega x^2 - \omega' }{\omega x + \omega' x^{-1} +u} = \delta \, x 
\Longleftrightarrow  \omega(1 - \delta)x^2 - u\delta x - \omega'(1 + \delta) = 0.\]
Set $a = \omega(1 - \delta),\, b = - u \delta $
and $c = -\omega'(1 + \delta)$. 

It is clear that $a > 0$, $c < 0$ and $b^2 - 4ac > 0$, so the equation $ax^2 + bx + c = 0$ in $\mathbb{R}$ has two distinct roots of opposite sign. For $x = 1$, we have  $a + b + c = \omega - \omega' - \delta < 0$. Thus, since $a > 0$, the strictly positive root $E$ of  
$ax^2 + bx + c $ verifies $E > 1$.
 \end{proof}

\subsection{Proof of Proposition \ref{prop:projCbv2}}
\label{subsec:proof:prop:projCbv2}
\propprojCbvTwo*
\begin{proof}
Set $y:= z^\sharp\in \Delta_n \cap \mathbb{R}^n_{++}$. As by Corollary \ref{cor:projCzzNat}, we have 
$\hat z :=\Proj_{C_r}^f(z)  = \Proj_{C_r}^f(y)$ and noticing 
\[s = y_i - y_j \quad \text{and} \quad  \dfrac{z_i x - z_j x^{-1}}{z_i x + z_j' x^{-1} + \Vert z \Vert_1 - z_i - z_j} = 
\dfrac{y_i  x - y_j x^{-1}}{y_i x + y_j x^{-1} +1 - y_i - y_j} \quad \text{for any $x\in\mathbb{R}_{++}$} \]
we  deduce  (\ref{eq:comp1}) from the components of $\hat y := \Proj_{C_r}^f(y)$.

If $s \ge \delta$ then  $y\in C_r$. Then we deduce from   Corollary \ref{cor:projCzzNat}  and (\ref{eq:projRule1})       $\hat y := \Proj_{C_r}^f(y)= y$, i.e.,  $\Proj_{C_r}^f(z) = z^\sharp$.

Suppose that we have $s < \delta $, so $y\notin C_r$. By applying  Lemma \ref{lemma:KKT1} to $y\in \Delta_n \cap \mathbb{R}^n_{++}$, we have   a pair $(\lambda^\star, \nu^\star)\in \mathbb{R}_+ \times \mathbb{R}$ such that for all $k\in\{1, 2, \dots, n\}$:
\begin{equation}\label{eq:comphaty2}
 \hat y_k = \begin{cases} e^{\lambda^\star + \nu^\star} y_i & \text{if } \quad k = i\\
e^{ - \lambda^\star + \nu^\star} y_j &  \text{if } \quad k = j\\
 e^{ \nu^\star} y_k & \text{if } \quad k \notin\{i, j\} 
 \end{cases}.
\end{equation}
Moreover  $\delta < \hat y_i - \hat y_j$ implies $\lambda^\star = 0$.

Set $u':= 1 - y_i - y_j \ge 0$. From (\ref{eq:comphaty2}),we deduce $1 = e^{\nu^\star}(e^{\lambda^\star} y_i + e^{- \lambda^\star} y_j + u')$ and then
\begin{equation}\label{eq:enustar}
 e^{\nu^\star} = \dfrac{1}{e^{\lambda^\star} y_i + e^{- \lambda^\star} y_j + u'}.
\end{equation}

As $\hat y \not= y$ because $\hat y \in C_r$ and $y\notin C_r$, and $\lambda^\star \ge 0$, from (\ref{eq:enustar}) we   deduce $\lambda^{\star} > 0$ and      obtain:
\[e^{\nu^\star}(e^{\lambda^\star} y_i  - e^{-\lambda^\star} y_j)=\dfrac{e^{\lambda^\star} y_i  - e^{-\lambda^\star} y_j}{e^{\lambda^\star} y_i + e^{- \lambda^\star} y_j + u'} =  \hat y_i - \hat y_j = \delta.\]

As $y\in \Delta_n\cap \mathbb{R}^n_{++}$ and $s = y_i - y_j < \delta$, we can apply  Lemma \ref{lemma:seconddegree} with:
\[\omega := y_i, \quad  \omega':= y_j \quad \text{and also $-1 <  \delta < 1$ and $u'  = 1 - y_i - y_j \ge 0$}\]
to conclude that $e^{\lambda^{\star}}$ is the unique solution $E > 1$ of the following equation in $\mathbb{R}_{++}$: 
\[ \dfrac{ y_i x   -   y_j x^{-1}}{y_i x +   y_j x^{-1} + u'} = \delta\]
So we have $e^{\lambda^{\star}} = E$ and by (\ref{eq:enustar})  $e^{\nu^\star} =   \dfrac{1}{ y_i E +  y_j E^{-1} + u'} = \dfrac{\Vert z \Vert_1}{D}$.

By replacing $e^{\lambda^{\star}}$ by $E$ and $e^{\nu^\star}$ by $\dfrac{\Vert z \Vert_1}{D}$ in (\ref{eq:comphaty2}), using $y:= z^\sharp$,  we deduce (\ref{eq:comp1}).

\end{proof}

\subsection{Proof of Proposition \ref{prop:zd}}
\label{subec:proof:prop:zd}
\propzd*

\begin{proof}
We prove the vector equalities (\ref{eq:AlgoD3}) by induction on $j\ge 1$:
\begin{equation}\label{eq:Hj}
\text{For all $h\in\{1, 2, \dots,m\}$ the three vector equalities (\ref{eq:AlgoD3}) holds for} \quad t = (j -1) m + h.    
\end{equation}.

$\bullet$  For   $j = 1$  and $h\in\{1, 2, \dots, m\}$, we have $t = (j -1) m + h = h$. As $d^{(t - m)} = 0$ by convention, then from (\ref{eq:ut}) and (\ref{eq:AlgoD2}), we get:
\[ u^{(t)} = z^{(t - 1)}, \quad z^{(t)} =  \Proj_t(u^{(t)}) \quad \text{and} \quad d^{(t)} = \log \dfrac{z^{(t - 1)}}{z^{(t)}}\]
thus, for $j =1$,  (\ref{eq:AlgoD3}) is established.

$\bullet$ Suppose that  (\ref{eq:AlgoD3}) holds for $j - 1$, where $j \ge 2$, and let us prove (\ref{eq:AlgoD3}) for $j$.

Let $h\in\{1, 2, \dots, m\}$  and set $t:= (j -1)m + h$. Then from (\ref{eq:ut}) and (\ref{eq:AlgoD2}), we have:
\[ u^{(t)} = z^{(t - 1)} \,.\, \text{exp}(d^{(t - m)}), \quad z^{(t)} =  \Proj_t(u^{(t)}) \quad \text{and} \quad d^{(t)} = d^{(t - m)} +  \log \dfrac{z^{(t - 1)}}{z^{(t)}}\]
As $t - m = (j - 2)m + h $, by relying on the induction hypothesis for $j -1$,  we deduce from its    third vector equality:
\[ d^{(t - m)} =  
\log\!\left(
\prod_{\ell=0}^{j-2}
\frac{z^{(\ell m+h-1)}}{z^{(\ell m+h)}}
\right)\]
Then, from Lemma \ref{lemma:explog}, we get   
\[ u^{(t)} = z^{(t - 1)} \,.\, \text{exp}(d^{(t - m)}) = z^{(t - 1)} \,. \, \text{exp}(\log\!\left(
\prod_{\ell=0}^{j-2}
\frac{z^{(\ell m+h-1)}}{z^{(\ell m+h)}}
\right)) = z^{(t - 1)} \,.\, \prod_{\ell=0}^{j-2}
\frac{z^{(\ell m+h-1)}}{z^{(\ell m+h)}}.\]
Obviously, we have:
\[ d^{(t)} = d^{(t - m)} +  \log \dfrac{z^{(t - 1)}}{z^{(t)}} = \log\!
\!\left(\prod_{\ell=0}^{j-2}
\frac{z^{(\ell m+h-1)}}{z^{(\ell m+h)}}
\right) + \log \dfrac{z^{(t - 1)}}{z^{(t)}} = \log\!\left(
\prod_{\ell=0}^{j-1}
\frac{z^{(\ell m+h-1)}}{z^{(\ell m+h)}}
\right).
\]

\end{proof}

\end{document}